\documentclass[12pt]{article}
\usepackage{enumerate}
\usepackage{graphicx}
\RequirePackage{natbib}
\usepackage{mathrsfs}
\usepackage{smile}
\usepackage{amsmath}
\usepackage{amsthm}
\usepackage{amssymb}
\usepackage{algorithm}
\usepackage{setspace}
\usepackage{soul}
\usepackage{color}
\usepackage[left=2.54cm,top=2.54cm,right=2.54cm,bottom=2.54cm]{geometry}

\usepackage{fullpage}

\usepackage[colorlinks,
linkcolor=red,
anchorcolor=blue,
citecolor=blue
]{hyperref}

\begin{document}

\title{Estimating and Inferring the Maximum Degree of Stimulus-Locked Time-Varying Brain Connectivity Networks}
\author{Kean Ming Tan, Junwei Lu, Tong Zhang, Han Liu}
\maketitle

\begin{abstract}
Neuroscientists have enjoyed much success in understanding brain functions by constructing brain connectivity networks using data collected under highly controlled experimental settings.   
However, these experimental settings bear little resemblance to our real-life experience in day-to-day interactions with the surroundings.  
To address this issue, neuroscientists have been measuring brain activity under natural viewing experiments in which the subjects are given continuous stimuli, such as watching a movie or listening to a story.
The main challenge with this approach is that the measured signal consists of both the stimulus-induced signal, as well as intrinsic-neural and non-neuronal signals.
By exploiting the experimental design, we propose to estimate stimulus-locked brain network by treating non-stimulus-induced signals as nuisance parameters.
In many neuroscience applications, it is often important to identify brain regions that are connected to many other brain regions during cognitive process.  
We propose an inferential method to test whether the maximum degree of the estimated network is larger than a pre-specific number. 
We prove that the type I error can be controlled and that the  power increases to one asymptotically. 
 Simulation studies are conducted to assess the performance of our method.  
Finally, we analyze a functional magnetic resonance imaging dataset obtained under the Sherlock Holmes  movie stimuli.
\end{abstract}


\section{Introduction} 
\label{section:introduction}
In the past few decades, much effort has been put into understanding task-based brain connectivity networks.
For instance, in a typical visual mapping experiment, subjects are presented with a simple static visual stimulus and are asked to maintain fixation at the visual stimulus, while their brain activities are measured.
Under such highly controlled experimental settings, numerous studies have shown that there are substantial similarities across  brain connectivity networks constructed for different subjects \citep{press2001visual,hasson2003large}. 
However, such experimental settings bear little resemblance to our real-life experience in several aspects: natural viewing consists of a continuous stream of perceptual stimuli; subjects can freely move their eyes; there are interactions among viewing, context, and emotion \citep{hasson2004intersubject}. 
To address this issue, neuroscientists have started measuring brain activity under continuous natural stimuli, such as watching a movie or listening to a story \citep{hasson2004intersubject,simony2016dynamic,chen2017shared}.
The main scientific question is to understand the dynamics of the brain connectivity network that are specific to the continuous natural stimuli.  

In the neuroscience literature, a typical approach for constructing a brain connectivity network is to calculate a sample covariance matrix for each subject: the covariance matrix encodes marginal relationships for each pair of brain regions within each subject.    
More recently, graphical models have been used in modeling brain connectivity networks:  graphical models encode conditional dependence relationships between each pair of brain regions, given the others \citep{rubinov2010complex}.
A graph consists of $d$ nodes, each representing a random variable, as well as a set of edges joining pairs of nodes corresponding to conditionally dependent variables.  
There is a vast literature on learning the structure of static undirected graphical models, and we refer the reader to \citet{drton2017structure} for a detailed review.  


Under natural continuous stimuli, it is often of interest to estimate a dynamic brain connectivity network, i.e., a graph that changes over time.
A natural candidate for this purpose is the time-varying Gaussian graphical model (\citealp{zhou2010time,kolar2010estimating}). 
The time-varying Gaussian graphical model assumes
\begin{equation}             
\label{Eq:timevarying GGM1}
\bX (z) \mid Z = z \sim N_d\{\mathbf{0},\bSigma_X(z)\},           
\end{equation}
where $\bSigma_X (z)$ is the covariance matrix of $\bX(z)$ given $Z=z$, and $Z\in [0,1]$ has a continuous density.  The inverse covariance matrix $\{\bSigma_X (z)\}^{-1}$ encodes conditional dependence relationships between pairs of random variables at time $Z=z$: $\{\bSigma_X (z)\}^{-1}_{jk} = 0 $ if and only if the $j$th and $k$th variables are conditionally independent given the other variables at time $Z=z$.

In natural viewing experiments, the main goal is to construct a brain connectivity network that is locked to the processing of external stimuli, referred to as \emph{stimulus-locked network} \citep{simony2016dynamic,chen2017shared,regev2018propagation}. 
Constructing a stimulus-locked network can better characterize the dynamic changes of brain patterns across the continuous stimulus \citep{simony2016dynamic}.
The main challenge in constructing stimulus-locked network is the lack of highly controlled experiments that remove spontaneous and individual variations.
The measured blood-oxygen-level dependent (BOLD) signal consists of not only signal that is specific to the stimulus, but also intrinsic neural signal (random fluctuations) and non-neuronal signal (physiological noise) that are specific to each subject.
The intrinsic neural signal and non-neuronal signal can be interpreted as measurement error or latent variables that confound the stimuli-specific signal. We refer to non-stimulus-induced signals as subject specific effects throughout the manuscript. 
Thus, directly fitting  \eqref{Eq:timevarying GGM1} using the measured data will yield a time-varying graph that primarily reflects intrinsic BOLD fluctuations within each brain rather than BOLD fluctuations due to the natural continuous stimulus.

In this paper, we exploit the experimental design aspect of natural viewing experiments and propose to estimate a  dynamic stimulus-locked brain connectivity network by treating the intrinsic neural signal and non-neuronal signal as nuisance parameters.
Our proposal exploits the fact that the same stimulus will be given to multiple independent subjects, and that the intrinsic neural signal and non-neuronal signal for different subjects are independent. 
Thus, this motivates us to estimate a brain connectivity network across two brains rather than within each brain.  In fact, this approach has been considered in \citet{simony2016dynamic} and \citet{chen2017shared} where they estimated brain connectivity networks by calculating covariance for brain regions between two brains.
 
After estimating the stimulus-locked brain connectivity network, the next important question is to infer whether there are any regions of interest that are connected to many other regions of interest during cognitive process \citep{hagmann2008mapping}.  
These highly connected brain regions are referred to as \emph{hub nodes}, and the number of connections for each brain region is referred to as \emph{degree}.
 Identifying hub brain regions that are specific to the given natural continuous stimulus will lead to a better understanding of the cognitive processes in the brain, and may shed light on various cognitive disorders. 
In the existing literature, several authors have proposed statistical methodologies to estimate networks with hubs (see, for instance, \citealp{tan2014learning}).  
In this paper, we instead focus on developing a novel inferential framework to test the hypothesis whether there exists at least one time point such that the maximum degree of the graph is greater than $k$.

Our proposed inferential framework is motivated  by two major components: (1) the Gaussian multiplier bootstrap for approximating the distribution of supreme of empirical processes \citep{chernozhukov2013gaussian,chernozhukov2014gaussian}, and (2) the step-down method for multiple hypothesis testing problems \citep{romano2005exact}.  
In a concurrent work, \citet{neykov2016combinatorial} proposed a framework for testing general graph structure on a static graph.
In Appendix A, we will show that our proposed method can be extended to testing a large family of graph structures similar to that of \citet{neykov2016combinatorial}.


\section{Stimulus-Locked Time-Varying Brain Connectivity Networks} 
\label{section:method}
\subsection{A Statistical Model}
\label{section:stat model}
Let $\bX(z)$, $\bS(z)$, $\bE(z)$ be the observed data, stimulus-induced signal, and subject specific effects at time $Z=z$, respectively.  
Assume that $Z$ is a continuous random variable with a continuous density. 
For a given $Z=z$, we model the observed data as the summation of stimulus-induced signal and the subject specific effects:
\begin{equation}
\label{model}
\bX(z) = \bS(z)+ \bE(z), \quad \bS(z) \mid Z= z\sim N_d\{\mathbf{0},\bSigma(z)\},\quad \bE(z) \mid Z= z\sim N_d\{\mathbf{0},\Lb_X(z)\},
\end{equation}
where $\bSigma(z)$ is the covariance matrix of the stimulus-induced signal, and $\Lb_X(z)$ is the covariance matrix of the subject specific effects.
We assume that $\bS(z)$ and $\bE(z)$ are independent for all $z$.
Thus, estimating the stimulus-locked brain connectivity network amounts to estimating $\{\bSigma(z)\}^{-1}$.
Fitting the  model in \eqref{Eq:timevarying GGM1}  using  the observed data will yield an estimate of $\{\bSigma(z)+\Lb_X(z)\}^{-1}$, and thus,  \eqref{Eq:timevarying GGM1} fails to estimate the  stimulus-locked brain connectivity network $\{\bSigma(z)\}^{-1}$.



To address this issue, we exploit the experimental design aspect of natural viewing experiments.  In many studies, neuroscientists often measure brain activity for multiple subjects under the same continuous natural stimulus \citep{chen2017shared,simony2016dynamic}. 
Let $\bX(z)$ and $\bY(z)$ be measured data for two subjects at time point $Z=z$. 
Since the same natural stimulus is given to both subjects, this motivates the following statistical model:
\begin{equation}
\label{eq:two subjects}
\begin{split}
\bX(z) = \bS(z) + \bE_{X}(z), \qquad \bY(z) = \bS(z) + &\bE_{Y}(z),\qquad \bS (z)| Z= z \sim N_d\{\mathbf{0},\bSigma(z)\},\\
\bE_X(z)  |Z= z \sim N_d\{\mathbf{0},\Lb_X (z)\},  &~~~~  \bE_Y (z)| Z= z \sim N_d\{\mathbf{0},\Lb_Y(z)\},
\end{split}
\end{equation}
where $\bS(z)$ is the stimulus-induced signal, and $\bE_X(z)$ and $\bE_Y(z)$ are the subject specific effects at $Z=z$.  
Model~\eqref{eq:two subjects} motivates the calculation of \emph{inter-subject covariance} between two subjects rather than the within-subject covariance.  For a given time point $Z=z$, we have 
\[
\EE[\bX(z)\{\bY(z)\}^T  \mid Z= z] = \EE[\bS(z)\{\bS(z)\}^T \mid Z= z] + \EE[\bE_{X}(z)\{\bE_{Y}(z)\}^T \mid Z= z] = \bSigma(z).
\]
That is, we estimate $\bSigma(z)$ via the inter-subject covariance by treating $\Lb_X(z)$ and $\Lb_Y(z)$ as nuisance parameters.
In the neuroscience literature, several authors have been calculating inter-subject covariance matrix to estimate marginal dependencies among brain regions that are stimulus-locked \citep{chen2017shared,simony2016dynamic}. 
They have found that calculating the inter-subject covariance is able to better capture the stimulus-locked marginal relationships for pairs of brain regions.  

For simplicity, throughout the paper, we focus on two subjects.  When there are multiple subjects, we can split the subjects into two groups, and obtain an average of each group to estimate the stimulus-locked brain network.   
We also discuss a $U$-statistic type estimator for the case when there are multiple subjects in Appendix B.


\subsection{Inter-Subject Time-Varying Gaussian Graphical Models}
\label{section:Estimation}

We now propose inter-subject time-varying Gaussian graphical models for estimating stimulus-locked time-varying brain networks.
Let $(Z_1,\bX_1,\bY_1),\ldots,(Z_n,\bX_n,\bY_n)$ be $n$ independent realizations of the triplets $(Z,\bX,\bY)$.  Both subjects share the same $Z_1,\ldots,Z_n$ since they are given the  same continuous stimulus.
Let $K : \RR \rightarrow \RR$ be a symmetric kernel function.
To obtain an estimate for $\bSigma(z)$, we propose the inter-subject kernel smoothed covariance  estimator 
\begin{equation}
\label{Eq:estimator}
\hat{\bSigma} (z) = \frac{\sum_{i\in [n]}  K_h (Z_i-z) \bX_i \bY_i^T}{\sum_{i\in [n]} K_h (Z_i-z)},
\end{equation}
where $K_h (Z_i -z ) =  K\{(Z_i -z )/h\}/h$, $h>0$ is the bandwidth parameter, and $[n] = \{1,\ldots,n\}$.
For simplicity, we use the Epanechnikov kernel 
\begin{equation}
\label{Eq:kernel use}
K (u) = 0.75\cdot \left(1-u^2\right) \cdot \ind_{\{|u|\le 1\}},
\end{equation}
where $\ind_{\{|u|\le 1\}}$ is an indicator function that takes value one if $|u|\le 1$ and zero otherwise.
The choice of kernel is not essential as long as it satisfies regularity conditions in Section~\ref{subsection:estimation error}.

Let $\bTheta(z) = \{\bSigma(z)\}^{-1}$. Given the kernel smoothed inter-subject covariance estimator  in~(\ref{Eq:estimator}), there are multiple approaches to obtain an estimate of the inverse covariance matrix $\bTheta (z)$.
We consider the CLIME estimator proposed by 
\citet{cai2011constrained}.  
Let $\mathbf{e}_j$ be the $j$th canonical basis in $\RR^d$.
For a vector $\vb \in \RR^{d}$, let $\|\vb\|_1 = \sum_{j=1}^d |v_j|$ and let $\|\vb\|_{\infty} = \max_j |v_j|$.
For each $j \in [d]$, the CLIME estimator takes the  form
\begin{equation}
\label{Eq:clime}
\hat{\bTheta}_j(z) = \underset{\btheta \in \RR^d}{\argmin}  \; \|\btheta\|_1 \qquad \mathrm{subject\; to\;} \left\|\hat{\bSigma}(z) \cdot \btheta - \mathbf{e}_j\right\|_{\infty} \le \lambda,
\end{equation}
where $\lambda>0$ is a tuning parameter that controls the sparsity of $\hat{\bTheta}_j (z)$. We construct an estimator for the stimulus-locked brain network as $\hat{\bTheta}(z) = [    \{\hat{\bTheta}_1 (z)\}^T,\ldots,\{\hat{\bTheta}_d(z)\}^T ]$.

There are two tuning parameters in our proposed method: a bandwidth parameter $h$ that controls the smoothness of the estimated covariance matrix, and a tuning parameter $\lambda$ that controls the sparsity of the estimated network. 
The bandwidth parameter $h$ can be selected according to the scientific context. 
For instance, in many neuroscience applications that involve continuous natural stimuli, we select $h$ such that there are always at least 30\% of the time points that have non-zero kernel weights.
In the following, we propose a $L$-fold cross-validation type procedure to select $\lambda$.
We first partition the $n$ time points into $L$ folds. Let $C_{\ell}$ be an index set containing time points for the $\ell$th fold.
Let $\bTheta(z)^{(-\ell)}$ be the estimated inverse covariance matrix using data excluding the $\ell$th fold, and let $\bSigma(z)^{(\ell)}$ be the estimated kernel smoothed covariance estimated using data only from the $\ell$th fold.  We calculate the following quantity for various values of $\lambda$ :
\begin{equation}
\label{eq:cv}
\mathrm{cv}_\lambda= \frac{1}{L}\sum_{\ell=1}^L\sum_{i\in C_\ell} \|\hat{\bSigma}(z_i,\lambda)^{(\ell)} \hat{\bTheta}{(z_i,\lambda)}^{(-\ell)}-\Ib_d\|_{\max},
\end{equation}
where $\|\cdot\|_{\max}$ is the element-wise max norm for matrix. 
From performing extensive numerical studies, we find that picking $\lambda$ that minimizes the above quantity tend to be too conservative.
We instead propose to pick the smallest $\lambda$ with $\mathrm{cv}_{\lambda}$ smaller than the minimum plus two standard deviation.

\subsection{Inference on Maximum Degree}
\label{subsection:inference}
We consider testing the hypothesis:
\begin{equation}
\small
\label{max test}
\begin{split}
&H_0: \mathrm{\;for \; all\;} z\in [0,1], \mathrm{\;the \; maximum\; degree\; of\; the\; graph\; is \;not\; greater\; than\;} k,\\
&H_1: \mathrm{\;there \; exists\;a \;} z_0\in [0,1] \mathrm{\;such \; that\; the \; maximum\; degree\; of\; the\; graph\; is \; greater\; than\;}k.\\
\end{split}
\end{equation}
In the existing literature, many authors have proposed to test whether there is an edge between two nodes in a graph (see,  \citealp{neykov2015unified}, and the references therein).
Due to the $\ell_1$ penalty used to encourage a sparse graph, classical test statistics are no longer asymptotically normal.  
We employ the de-biased test statistic
\begin{equation}
\label{Eq:one-step}
\hat{\bTheta}^{\mathrm{de}}_{jk} (z) = \hat{\bTheta}_{jk} (z)
  - \frac{\left\{\hat{\bTheta}_j (z)\right\}^T \left\{ \hat{\bSigma}(z) \hat{\bTheta}_k (z)- \mathbf{e}_k \right\}     }{ \left\{ \hat{\bTheta}_j(z)\right\}^T \hat{\bSigma}_j(z)},
\end{equation}
where $\hat{\bTheta}_j(z)$ is the $j$th column of $\hat{\bTheta}(z)$.
The subtrahend in \eqref{Eq:one-step} is the bias introduced by imposing an $\ell_1$ penalty during the estimation procedure.

We use (\ref{Eq:one-step}) to construct a test statistic for testing the maximum degree of a time-varying graph. Let $G(z) = \{V,E(z)\}$ be an undirected graph, where $V= \{1,\ldots,d\}$ is a set of $d$ nodes and $E(z) \subseteq V\times V$ is a set of edges connecting pairs of nodes.
Let
\begin{equation}
\label{Eq:quantile}
T_E =  \underset{z\in [0,1]}{\sup} \;\underset{(j,k)\in E(z)}{\max} \; 
\sqrt{nh}\cdot  \left|\hat{\bTheta}^{\mathrm{de}}_{jk} (z) - \bTheta_{jk} (z)\right| \cdot \left\{\frac{1}{n}\sum_{i \in [n]} K_h(Z_i-z) \right\}.
\end{equation}
The edge set $E(z)$ is defined based on the hypothesis testing problem.   In the context of testing maximum degree of a time-varying graph as in \eqref{max test}, $E(z) = V\times V$, and therefore the maximum is taken over all possible edges between pairs of nodes.  
Throughout the manuscript, we will use the notation $E(z)$ to indicate some predefined known edge set. 
This general edge set will be different for testing different graph structures, and we refer the reader to Appendix A for details.

Since the test statistic~(\ref{Eq:quantile}) involves taking the supreme over $z$ and the maximum over all edges in $E(z)$, it is challenging to evaluate its asymptotic distribution.
To this end, we generalize the  Gaussian multiplier bootstrap proposed in \citet{chernozhukov2013gaussian} and \citet{chernozhukov2014gaussian} to approximate the distribution of the test statistic $T_E$.
Let $\xi_1,\ldots,\xi_n \stackrel{\mathrm{i.i.d.}}{\sim} N(0,1)$. We construct the bootstrap statistic as
\begin{equation}
\label{Eq:approx test}
T^B_E = \underset{z\in [0,1]}{\sup} \; \underset{(j,k) \in E(z) }{\max} \;\sqrt{nh}\cdot \left| \frac{ \sum_{i\in [n]} \left\{\hat{\bTheta}_j (z)\right\}^T K_h (Z_i-z) \left\{ \bX_i \bY_i^T\hat{\bTheta}_k (z)  - \mathbf{e}_k\right\} \xi_i/n}{ \left\{ \hat{\bTheta}_j(z)\right\}^T \hat{\bSigma}_j(z)}\right|.
\end{equation}
 We denote the conditional $(1-\alpha)$-quantile of $T^B_E$ given $\{(Z_i,\bX_i,\bY_i)\}_{i\in [n]}$ as 
\begin{equation}
\label{Eq:quantile estimate}
c(1-\alpha,E) = \inf \left( t\in \RR   \; | \;  P \left[ T^B_E \le t    \; |\;  \{(Z_i,\bX_i,\bY_i)\}_{i\in[n]} \right] \ge 1-\alpha     \right).
\end{equation}
 The quantity $c(1-\alpha,E)$ can be calculated numerically using Monte-Carlo.  
In Section~\ref{subsection:testing}, we show that the quantile of $T_E$ in (\ref{Eq:quantile}) can be estimated accurately by the conditional 	$(1-\alpha)$-quantile of the bootstrap statistic.

We now propose an inference framework for testing hypothesis problem of the form  (\ref{max test}). 
Our proposed method is motivated by the step-down method in \citet{romano2005exact} for multiple hypothesis tests. 
The details are summarized in Algorithm~\ref{al:stepdown}.  
Algorithm~\ref{al:stepdown} involves evaluating all values of $z\in [0,1]$.  
In practice, we implement the proposed method by discretizing values of $z\in[0,1]$ into a large number of time points.  
We note that there will be approximation error by taking the maximum over the discretized time points instead of the supremum of the continuous trajectory. The approximation error could be reduced to arbitrarily small if we increase the density of discretization.



\begin{algorithm}[ht]
\caption{\label{al:stepdown} Testing Maximum Degree of a Time-Varying Graph.}
\raggedright \textbf{Input:} type I error $\alpha$; pre-specified degree $k$; de-biased estimator $\hat{\bTheta}^{\mathrm{de}}(z)$ for $z\in [0,1]$.  
\begin{enumerate}
\item Compute the conditional  quantile 
\[
c(1-\alpha,E) =  \inf \left[ t\in \RR   \mid P(T^B_E)  \le t \mid \{(Z_i,\bX_i,\bY_i)\}_{i\in[n]}   \ge 1-\alpha     \right],
\]
where 
$T^B_E$ is the bootstrap statistic defined in (\ref{Eq:approx test}).
\item Construct the rejected edge set 
\[
\cR(z)= \left\{e\in E(z) \mid \sqrt{nh} \cdot |\hat{\bTheta}_e^d (z)| \cdot \sum_{i\in [n]} K_h(Z_i-z)/n > c (1-\alpha,E) \right\}.
\]

\item Compute $d_{\mathrm{rej}}$ as the maximum degree of the dynamic graph based on the rejected edge set.
\end{enumerate}

\raggedright  \textbf{Output:} Reject the null hypothesis if $d_{\mathrm{rej}}>k$.

\end{algorithm}

In Section~\ref{subsection:testing}, we will show that Algorithm~\ref{al:stepdown} is able to control the type I error at a pre-specified value $\alpha$. 
Moreover, the power of the proposed inferential method increases to one as we increase the number of time points $n$.
In fact, the proposed inferential method can be generalized to testing a wide variety of structures that satisfy the \emph{monotone graph property}. Some examples of monotone graph property are that the graph is connected, the graph has no more than $k$ connected components, the maximum degree of the graph is larger than $k$, the graph has no more than $k$ isolated nodes, and the graph contains a clique of size larger than $k$.  This generalization will be presented in  Appendix A.

\section{Simulation Studies} 
\label{section:simulation}
We perform numerical studies to evaluate the performance of our  proposal using the inter-subject covariance relative to the typical time-varying Gaussian graphical model using within-subject covariance.  
To this end, we define the true positive rate as the proportion of correctly identified non-zeros in the true inverse covariance matrix, and the false positive rate as the proportion of zeros that are incorrectly identified to be non-zeros. 
To evaluate our testing procedure, we calculate the type I error rate and power as the proportion of falsely rejected  $H_0$ and correctly rejected $H_0$, respectively, over a large number of data sets. 

To generate the data, we first construct the inverse covariance matrix $\bTheta(z)$ for $z=\{0,0.2,0.5\}$. 
At $z=0$, we set $(d-2)/4$ off-diagonal elements of $\bTheta(0)$ to equal 0.3 randomly with equal probability. At $z=0.2$, we set an additional $(d-2)/4$ off-diagonal elements of $\bTheta(0)$ to equal 0.3.  At $z=0.5$, we randomly select two columns of $\bTheta(0.2)$ and add $k+1$ edges to each of the two columns.  
This guarantees that the maximum degree of the graph is greater than $k$.
To ensure that the inverse covariance matrix is smooth, for $z \in [0,0.2]$, we construct $\bTheta(z)$ by taking linear interpolations between the elements of $\bTheta(0)$ and $\bTheta(0.2)$.  For $z \in [0.2,0.5]$, we construct $\bTheta(z)$ in a similar fashion based on $\bTheta(0.2)$ and $\bTheta(0.5)$.
The construction is illustrated in Figure~\ref{Fig:hub}.  

\begin{figure}[!htp]
\centering
\includegraphics[scale=0.33]{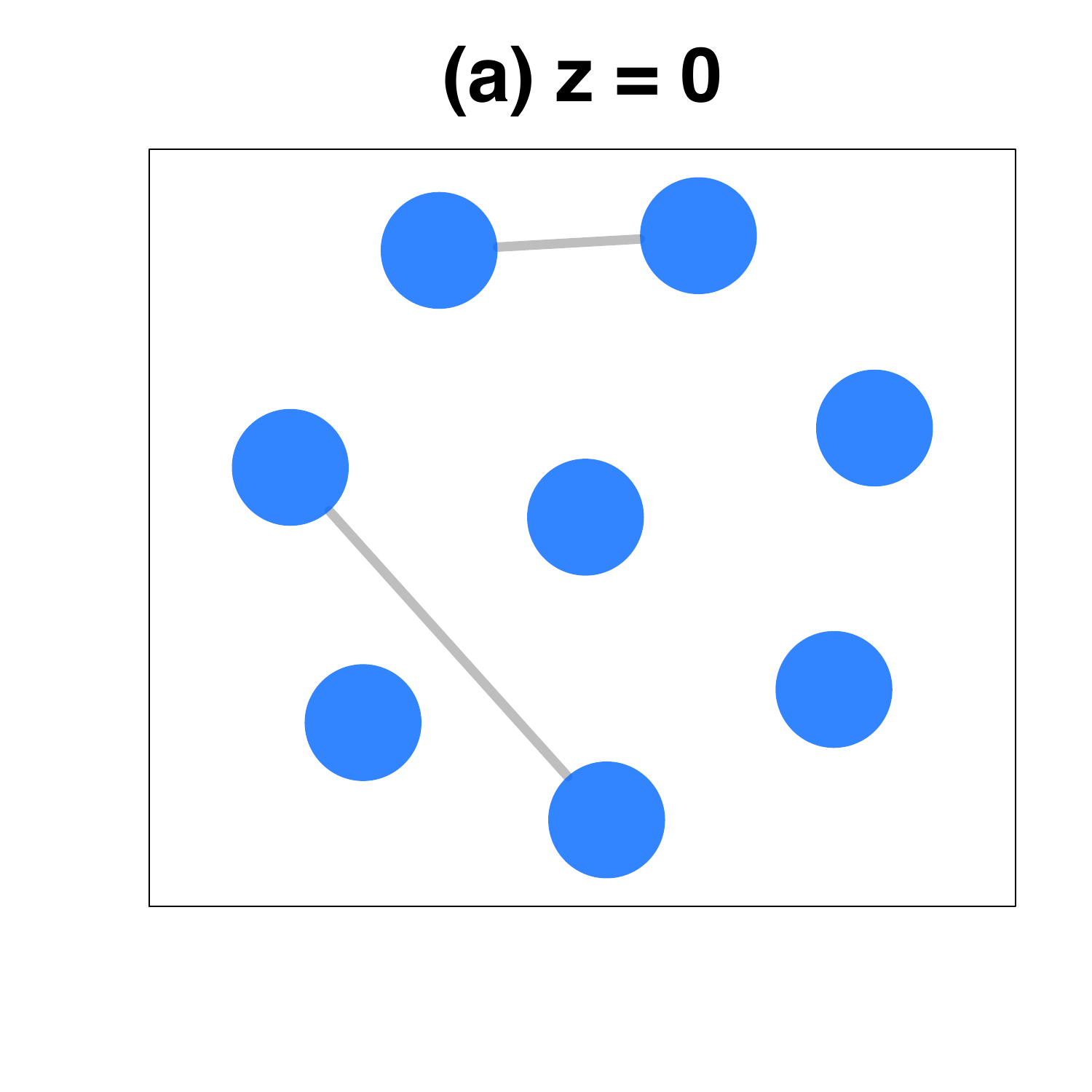}
\includegraphics[scale=0.33]{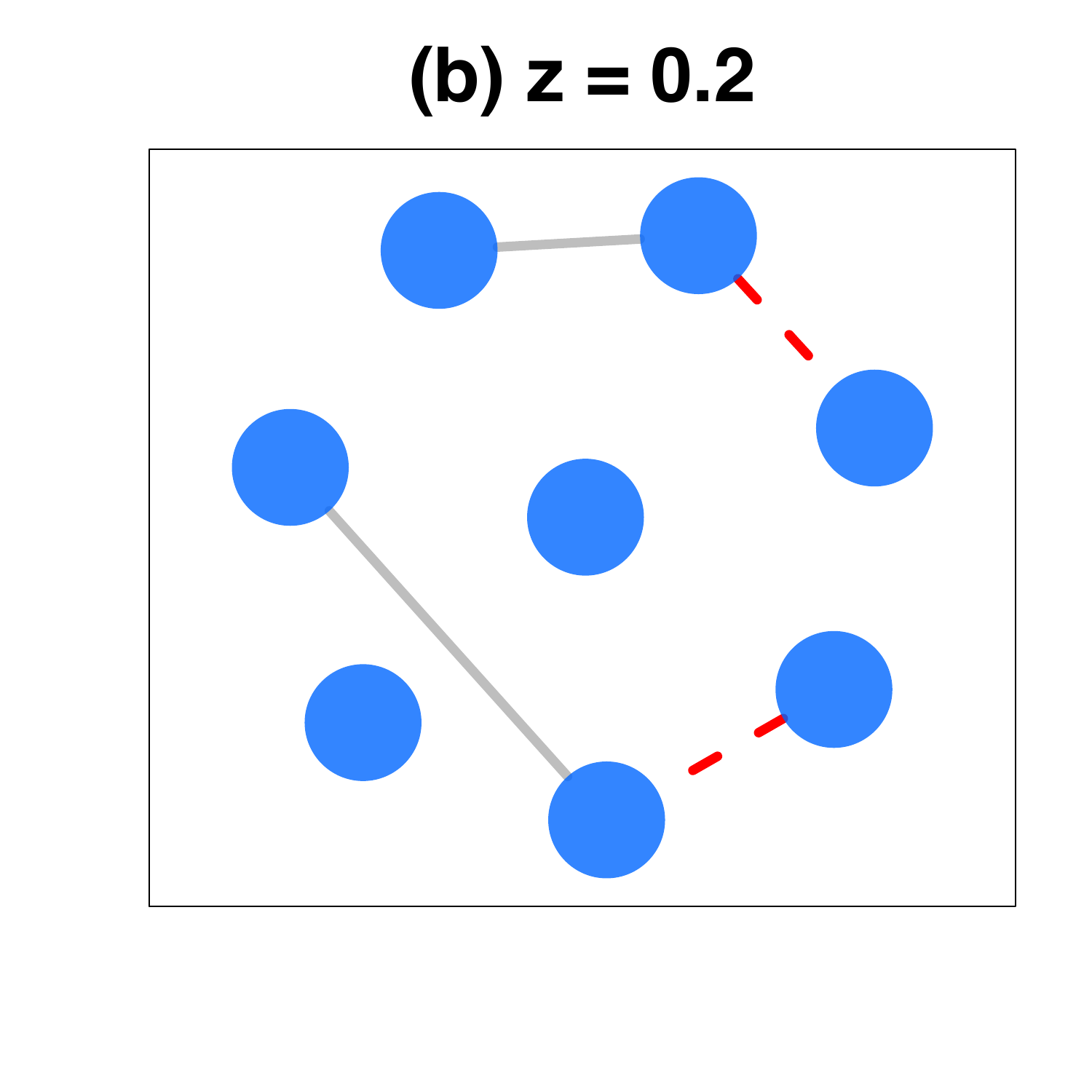}
\includegraphics[scale=0.33]{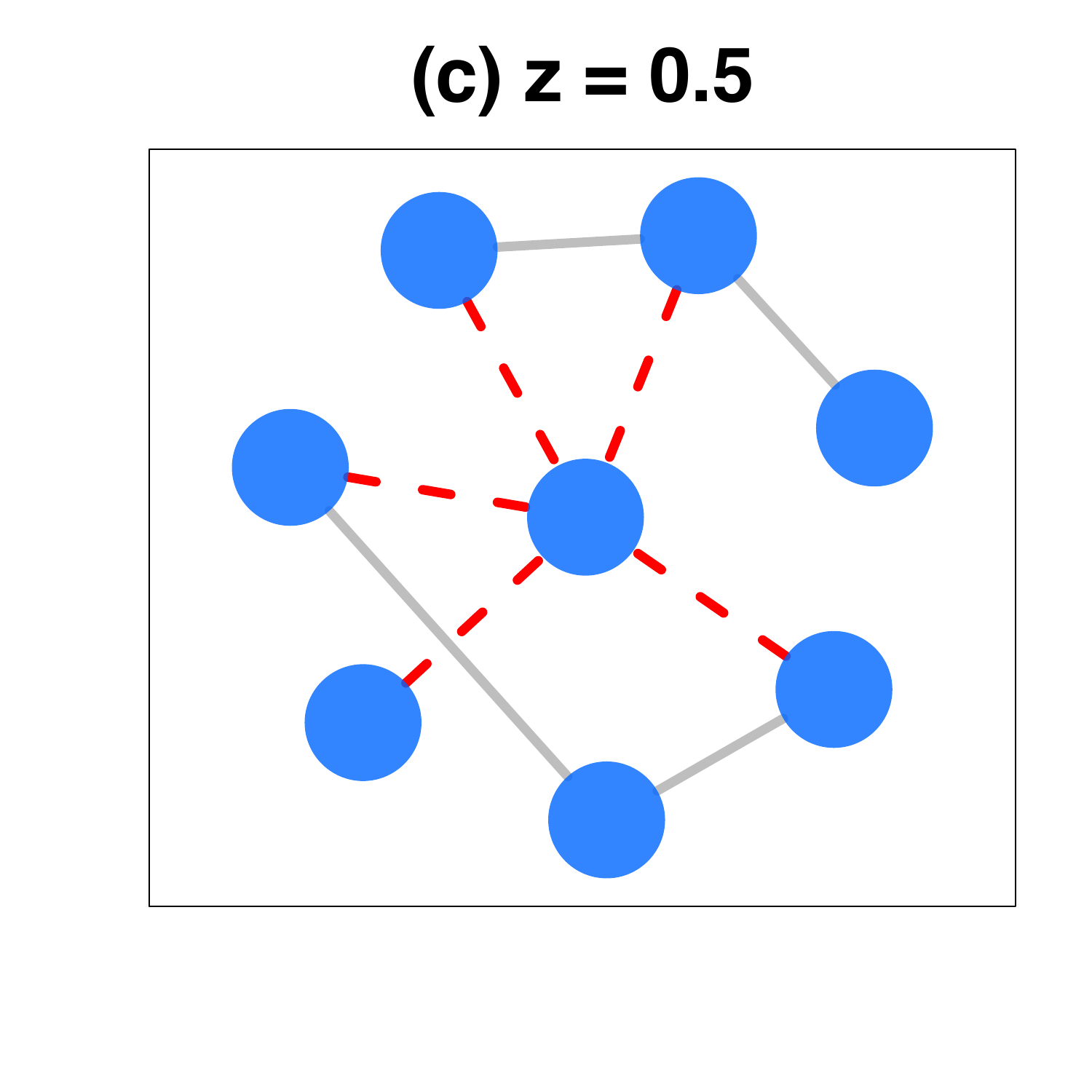}
\caption{\label{Fig:hub} (a): A graph corresponding to $\bTheta(0)$ with maximum degree no greater than  four. (b): A graph corresponding to $\bTheta(0.2)$ with maximum degree less than or equal to four.  The red dash edges are additional edges that are added to $\bTheta(0)$. (c): A graph corresponding to $\bTheta(0.5)$ with maximum degree larger than four.  The red dash edges are additional edges that are added to $\bTheta(0.2)$ such that the maximum degree of the graph is larger than four. } 
\end{figure}

To ensure that the inverse covariance matrix is positive definite, we set $\bTheta_{jj} (z) = |\Lambda_{\min} \{\bTheta(z)\}|+0.1$, where $\Lambda_{\min}\{\bTheta (z)\}$ is the minimum eigenvalue of $\bTheta(z)$.  We then rescale the matrix such that the diagonal elements of $\bTheta(z)$ equal one.  
The covariance $\bSigma(z)$ can be obtained by taking the inverse of $\bTheta(z)$ for each value of $z$.
Model~\eqref{eq:two subjects} involves the subject specific covariance matrix $\Lb_X (z)$ and $\Lb_Y(z)$. For simplicity, we assume that these covariance matrices stay constant over time.  We generate $\Lb_X$ by setting the diagonal elements to be one and the off-diagonal elements to be 0.3.  Then, we add random perturbations $\epsilon_k\epsilon_k^T$ to $\Lb_X$ for $k=1,\ldots,10$, where $\epsilon_k \sim N_d(\mathbf{0},\Ib_d)$.  The matrix $\Lb_Y$ is generated similarly.

To generate the data according to \eqref{eq:two subjects}, we first generate $Z_i \sim \mathrm{Unif}(0,1)$.
Given $Z_1,\ldots,Z_n$, we generate $\bS(Z_i) \mid Z=Z_i \sim N_d\{\mathbf{0},\bSigma(Z_i)\}$.  We then simulate $\bE_X(Z_i)\mid Z=Z_i \sim N_d(\mathbf{0},\Lb_X) $ and $\bE_Y(Z_i)\mid Z=Z_i \sim N_d(\mathbf{0},\Lb_Y) $.  Finally, for each value of $Z$, we generate 
\[
\bX(Z_i)= \bS(Z_i)+\bE_X(Z_i) \qquad \mathrm{and} \qquad \bY(Z_i)= \bS(Z_i)+\bE_Y(Z_i).
\]
Note that both $\bX(Z_i)$ and $\bY(Z_i)$ share the same generated $\bS(Z_i)$ since both subjects will be given the same natural continuous stimulus. 
In the following sections, we will assess the performance of our proposal relative to that of typical approach for time-varying Gaussian graphical models using the. within-subject covariance matrix as input.  
We then evaluate the proposed inferential procedure in Section~\ref{subsection:inference} by calculating its type I error and power.

\subsection{Estimation}
\label{sim:estimation}
To mimic the data application we consider, we generate the data with $n=945$, $d=172$, and $k=10$.
Given the data $(Z_1,\bX_1,\bY_1),\ldots,(Z_n,\bX_n,\bY_n)$, we estimate the covariance matrix at $Z=z$ using the inter-subject kernel smoothed covariance estimator as defined in (\ref{Eq:estimator}).
To obtain estimates of the inverse covariance matrix $\hat{\bTheta}(Z_1),\ldots,\hat{\bTheta}(Z_n)$,  we use the CLIME estimator as described in (\ref{Eq:clime}), implemented using the \verb=R= package \verb=clime=.
There are two tuning parameters $h$ and $\lambda$: we set $h=1.2 \cdot n^{-1/5}$ and vary the tuning parameter $\lambda$ to obtain the ROC curve in Figure~\ref{fig:roc}.  The smoothing parameter $h$ is selected such that there are always at least 30\% of the time points that have non-zero kernel weights.
We compare our proposal to time-varying Gaussian graphical models with the kernel smoothed within-subject covariance matrix.
The true and false positive rates, averaged over 100 data sets, are in Figure~\ref{fig:roc}.

\begin{figure}
\centering
\includegraphics[scale=0.28]{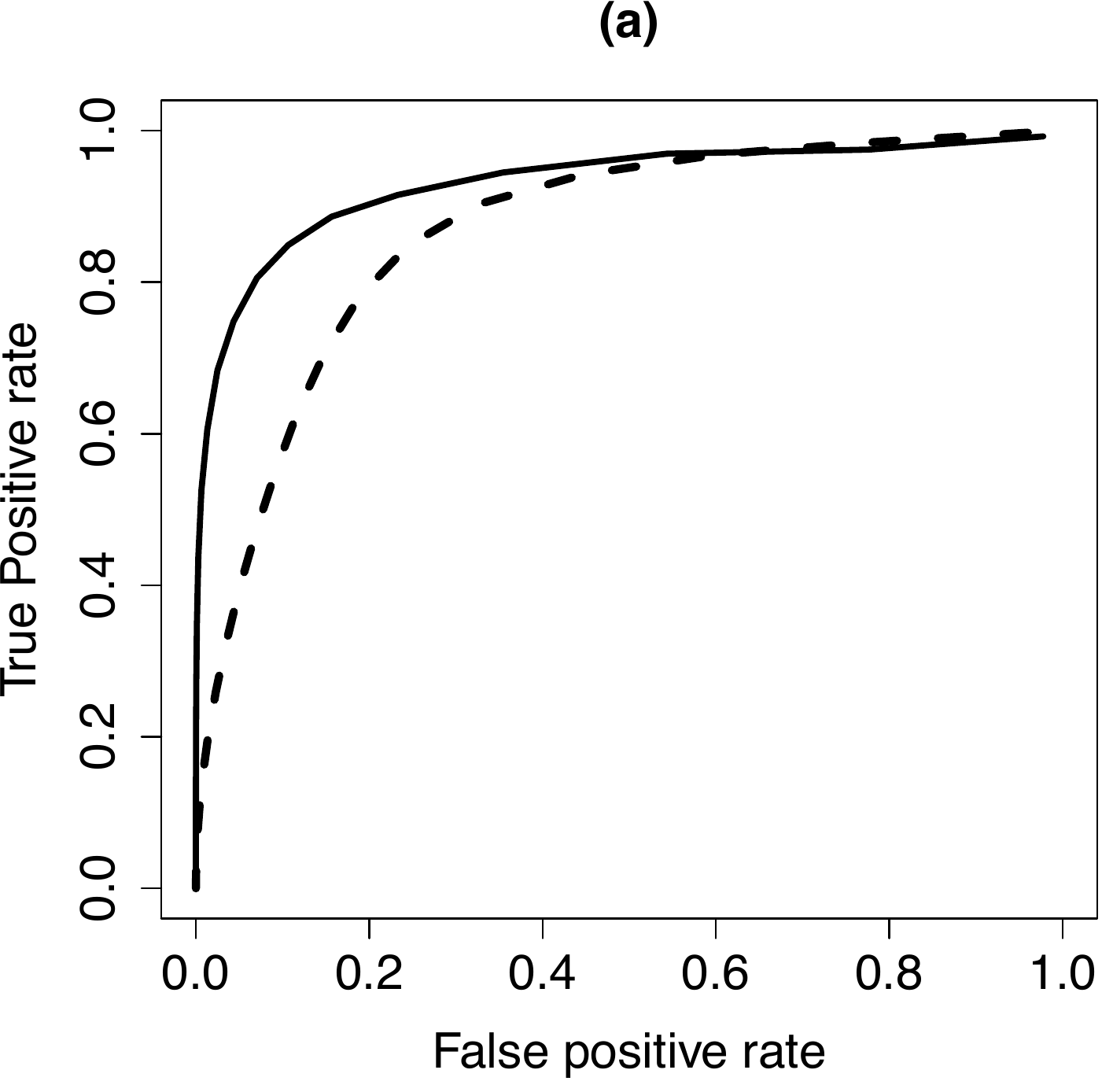}
\includegraphics[scale=0.28]{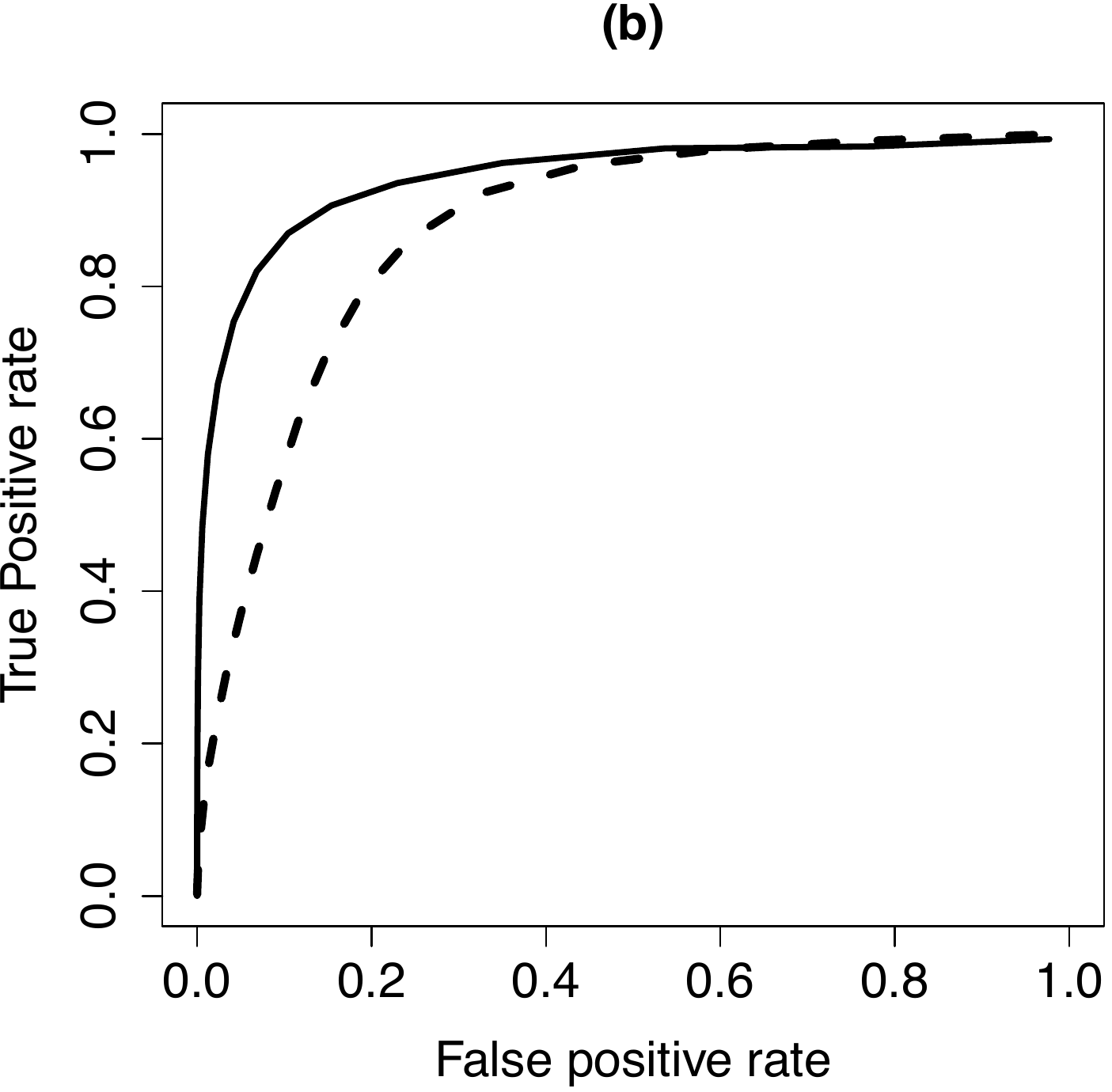}
\includegraphics[scale=0.28]{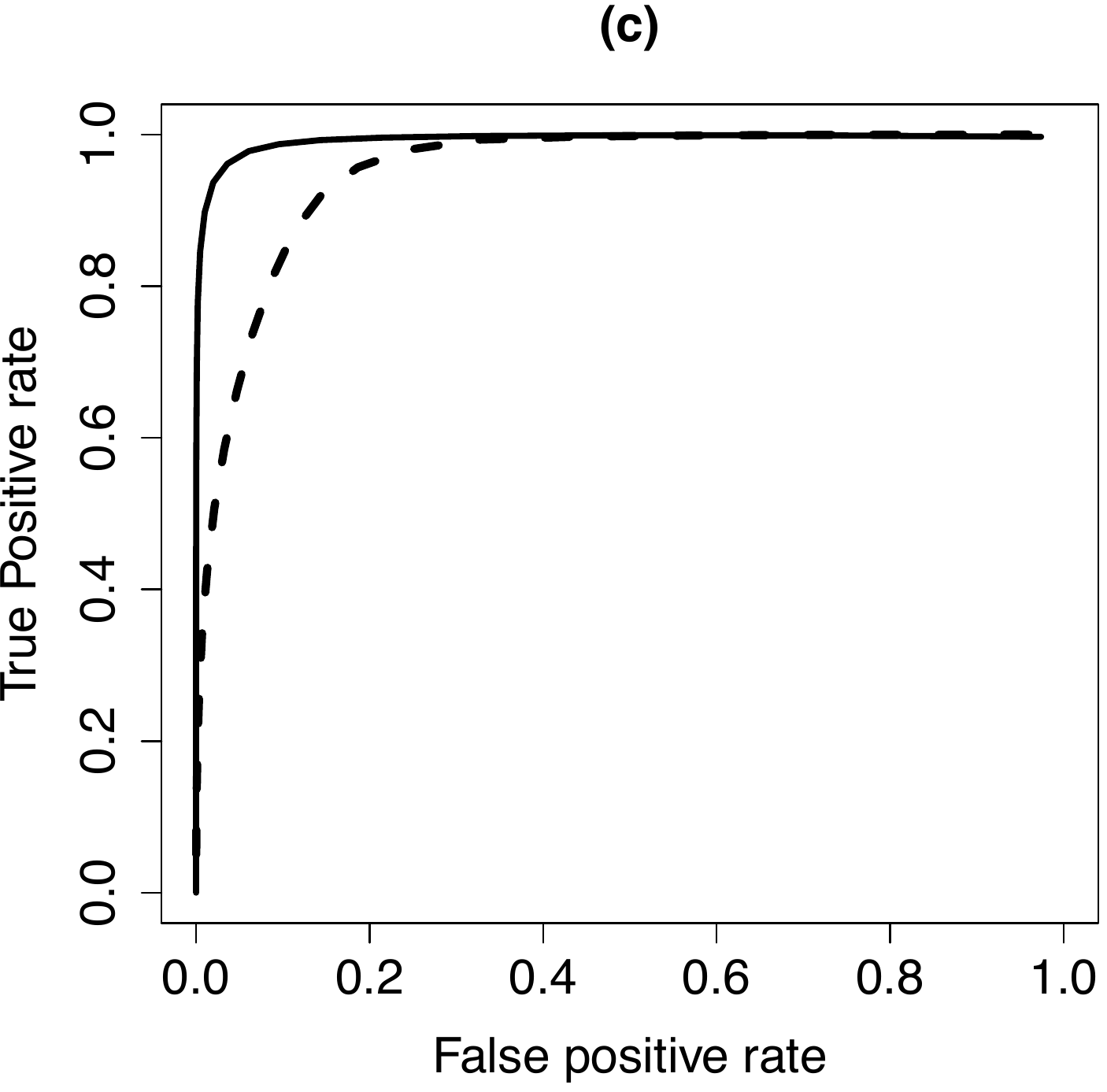}
\caption{\label{fig:roc}  The true and false positive rates for the numerical study with $n= 952$, $d = 172$, and $k = 10$. Panels (a), (b), and (c) correspond to $Z=\{0.25,0.50, 0.75\}$, respectively. The two curves represent our proposal (black solid line) and within-subject time-varying Gaussian graphical model (black dash), respectively.}
\end{figure}

From Figure~\ref{fig:roc}, we see that our proposed method outperforms the typical approach for time-varying Gaussian graphical models by calculating the within-subject covariance matrix. 
This is because the typical approach is not estimating the parameter of interest, as discussed in Section~\ref{section:Estimation}. 
Our proposed method treats the subject specific effects as nuisance parameters and is able to estimate the stimulus-locked graph accurately.

\subsection{Testing the Maximum Degree of a Time-Varying Graph}
\label{subsec:hub}
In this section, we evaluate the proposed inferential method in Algorithm~\ref{al:stepdown} by calculating its type I error and power.
 In all of our simulation studies, we consider $d=50$ and $B=500$ bootstrap samples, across a range of samples $n$. 
 Similarly, we select the smoothing parameter to be $h=1.2\cdot n^{-1/5}$.  The tuning parameter $\lambda$ is then selected using the cross-validation criterion  defined in \eqref{eq:cv}.
The tuning parameter $\lambda = 0.9 \cdot [h^2 + \sqrt{\{\log (d/h)\}/(nh)}]$ is selected for one of the simulated data set.  For computational purposes, we use this value of tuning parameter across all replications.

We construct the test statistic $T_E$ and the Gaussian multiplier bootstrap statistic $T^B_E$ as defined in (\ref{Eq:quantile}) and (\ref{Eq:approx test}), respectively.   Both the statistics $T_E$ and $T^B_E$ involve evaluating the supreme over $z\in [0,1]$.  In our simulation studies, we approximate the supreme by taking the maximum of the statistics over 50 evenly spaced grid $z\in [z_{\min},z_{\max}]$, where $z_{\min} = \min \left\{ Z_i\right\}_{i\in [n]}$ and  $z_{\max} = \max  \left\{ Z_i\right\}_{i\in [n]}$.

Our testing procedure tests the hypothesis
\begin{equation*}
\begin{split}
&H_0: \mathrm{\;for \; all\;} z\in [z_{\min},z_{\max}], \mathrm{\;the \; maximum\; degree\; of\; the\; graph\; is \;no\; greater\; than\;} k,\\
&H_1: \mathrm{\;there \; exists\;a \;} z_0\in [z_{\min},z_{\max}] \mathrm{\;such \; that\; the \; maximum\; degree\; of\; the\; graph\; is \; greater\; than\;}k.\\
\end{split}
\end{equation*}
For power analysis, we construct $\bTheta(z)$ according to Figure~\ref{Fig:hub} by randomly selecting two columns of $\bTheta(0.2)$ and adding $k+1$ edges to each of the two columns.  This ensure that the maximum degree of the graph is greater than $k$.
To evaluate the type I error under $H_0$, instead of adding $k+1$ edges to the two columns, we instead add sufficient edges such that the maximum degree of the graph is no greater than $k$.  For the purpose of illustrating the type I error and power in the finite sample setting, we increase the signal-to-noise ratio of the data by reducing the effect of the nuisance parameters in the data generating mechanism described in Section~\ref{section:simulation}.  
The type I error and power for $k=\{5,6\}$, averaged over $500$ data sets, are reported in Table~\ref{Table:inf2}. We see that the type I error is controlled and that the power increases to one as we increase the number of time points $n$.

\begin{table}[htp]
\begin{center}
\caption{The type I error and power for testing the maximum degree of the graph  at the 0.05 significance level are calculated as the proportion of falsely rejected and correctly rejected null hypotheses, respectively, over 500 data sets. Simulation results with $d=50$ and $k=\{5,6\}$, over a range of $n$ are shown.}
\begin{tabular}{cc   c cc cc}
\hline
\hline
&&$n=400$ & $n=600$&$n=800$&$n=1000$&$n=1500$ \\ \hline
$k$=5 &Type I error & 0.014 &  0.024  & 0.030 & 0.034  & 0.028  \\
&Power & 0.068 &  0.182  & 0.690  & 0.976  & 1  \\ \hline
$k$=6 &Type I error &  0.032&0.040   & 0.034 & 0.028  & 0.018  \\
&Power & 0.050  & 0.142   &  0.446 &   0.898&1  \\ \hline\hline
\end{tabular}
\label{Table:inf2}
\end{center}
\end{table}

\section{Sherlock Holmes Data} 
\label{section:real data}
We analyze a brain imaging data set studied in \citet{chen2017shared}.  This data set consists of fMRI measurements of 17 subjects while watching audio-visual movie stimuli in an fMRI scanner.  
More specifically,  the subjects were asked to watch a 23-minute segment of BBC television series \emph{Sherlock}, taken from the beginning of the first episode of the series.  The fMRI measurements were taken every 1.5 seconds of the movie, yielding $n=945$ brain images for each subject.  To understand the dynamics of the brain connectivity network under natural continuous stimuli, we partition the movie into 26 scenes \citep{chen2017shared}. 
The data were pre-processed for slice time correction, motion correction, linear detrending, high-pass filtering, and coregistration  to a template brain \citep{chen2017shared}. Furthermore, for each subject, we attempt to mitigate issues caused by non-neuronal signal sources by regressing out the average white matter signal.  

There are measurements for 271,633 voxels in this data set.  For interpretation purposes, 
we reduce the dimension from 271,633 voxels to $d=172$ regions of interest (ROIs) as described in \citet{baldassano2015parcellating}.    
We map the $n=945$ brain images taken across the 23 minutes into the interval $[0,1]$ chronologically. We then standardize each of the 172 ROIs to have mean zero and standard deviation one.

We first estimate the stimulus-locked time-varying brain connectivity network. To this end, we construct the inter-subject kernel smoothed covariance matrix $\hat{\bSigma}(z)$ as defined in (\ref{Eq:estimator}).  
Since there are 17 subjects, we randomly split the 17 subjects into two groups, and use the averaged data to construct \eqref{Eq:estimator}.
Note that we could also construct a brain connectivity network for each pair of subjects separately.
We then obtain estimates of the inverse covariance matrix using the CLIME estimator as in (\ref{Eq:clime}).
We set the smoothing parameter $h= 1.2\cdot n^{-1/5}$ so that at least 30\% of the kernel weights are non-zero across all time points $Z$. 
For the sparsity tuning parameter, our theoretical results suggest picking $\lambda = C \cdot \{h^2 +\sqrt{\log (d/h)/nh}\}$ to guarantee a consistent estimator. 
We select the constant $C$ by considering a sequence of numbers using a $5$-fold cross-validation procedure described in \eqref{eq:cv}, and this yields  $\lambda = 1.4 \cdot \{h^2+\sqrt{\log (d/h)/(nh)}\}$.
Heatmaps of the estimated stimulus-locked brain connectivity networks for three different scenes in Sherlock are in Figure~\ref{fig:heatmap}.

 \begin{figure}[!htp]
\begin{center}
\includegraphics[scale=0.5]{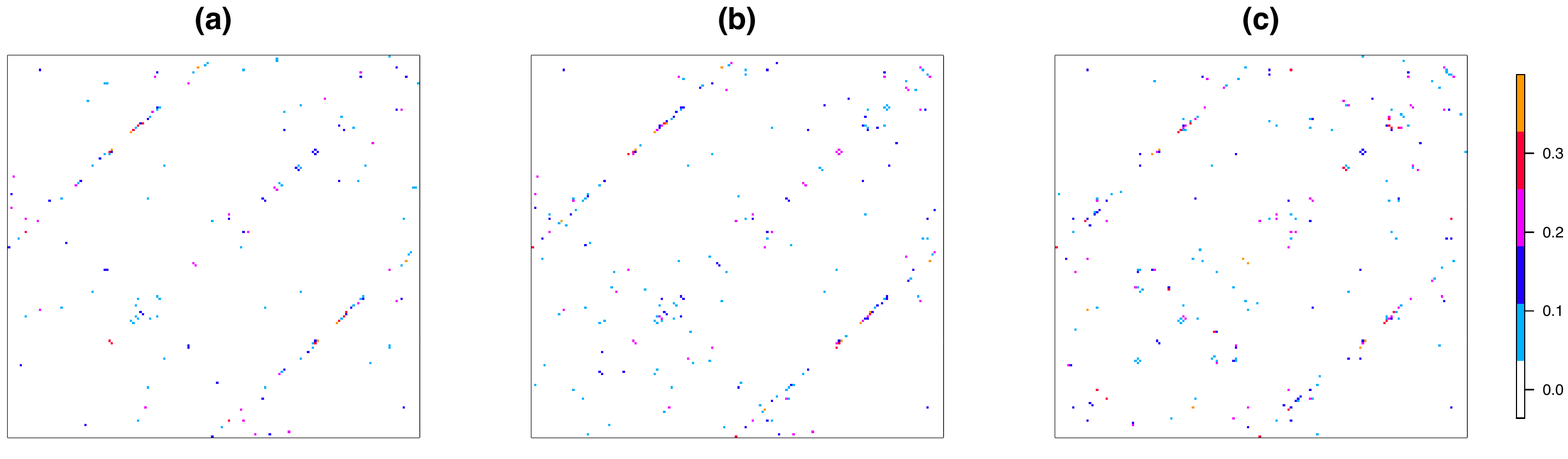}
 \end{center}
  \caption{Heatmaps of the estimated stimulus-locked brain connectivity network for three different scenes in Sherlock.  (a) Watson psychiatrist scene; (b) Park run in scene; and (c) Watson joins in scene.  Colored elements in the heatmaps correspond to edges in the estiamted brain network.}
  \label{fig:heatmap}
\end{figure}

From Figure~\ref{fig:heatmap}, we see that there are quite a number of connections between brain regions that remain the same across different scenes in the movie.  
It is also evident that the graph structure changes across different scenes. 
We see that most brain regions are very sparsely connected, with the exception of a few ROIs. 
This raises the question of identifying whether there are hub ROIs that are connected to many other ROIs under audio-visual stimuli.

To answer this question, we perform a hypothesis test to test whether there are hub nodes that are connected to many other nodes in the graph across the 26 scenes.  
If there are such hub nodes, which ROIs do they correspond to. 
  More formally, 
we test the hypothesis
\begin{equation*}
\begin{split}
&H_0: \mathrm{\;for \; all\;} z\in [0,1], \mathrm{\;the \; maximum\; degree\; of\; the\; graph\; is \;no\; greater\; than\;} 15,\\
&H_1: \mathrm{\;there \; exists\;a \;} z_0\in [0,1] \mathrm{\;such \; that\; the \; maximum\; degree\; of\; the\; graph\; is \; greater\; than \;}15.
\end{split}
\end{equation*}
The number 15 is chosen since we are interested in testing whether there is any brain region that is connected to more than 10\% of the total number of brain regions.
We apply Algorithm~\ref{al:stepdown} with 26 values of $z$ corresponding to the middle of the 26 scenes.  
Figure~\ref{Fig:holmes max degree} shows the ROIs that have more than 12 rejected edges across the 26 scenes based on Algorithm~\ref{al:stepdown}.  
Since the maximum degree of the rejected nodes in some scenes are larger than 15, we reject the null hypothesis that the maximum degree of the graph is no greater than 15.  
In Figure~\ref{Fig:holmes connectivity}, we plot the sagittal snapshot of the brain connectivity network, visualizing the rejected edges from Algorithm~\ref{al:stepdown} and the identified  hubs ROIs.

\begin{figure}[!htp]
\centering
\includegraphics[scale=0.45]{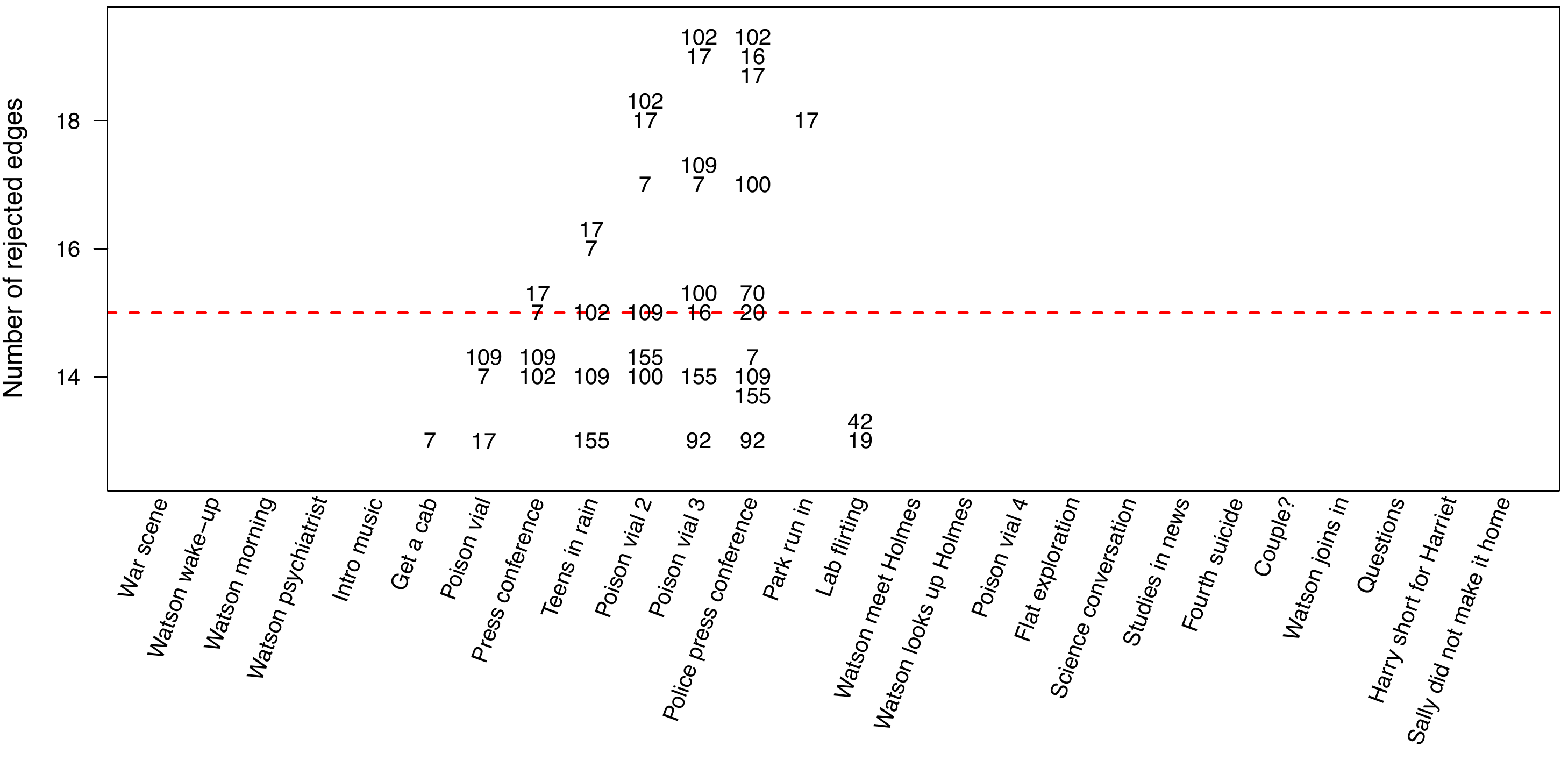}
\caption{\label{Fig:holmes max degree} The $x$-axis displays the 26 scenes in the movie and the $y$-axis displays the number of rejected edges from Algorithm~\ref{al:stepdown}.  The numbers correspond to the regions of interest (ROIs) in the brain. The ROIs correspond to frontal pole (7, 155), temporal fusiform cortex (16, 100), lingual gyrus  (17), cingulate gyrus (19), cingulate gyrus (20), temporal pole (42), paracingulate gyrus (70), precuneus cortex (102), and postcentral gyrus (109).} 
\end{figure}

\begin{figure}[!htp]
\centering
\subfigure[Teens in rain scene]{\includegraphics[scale=0.2]{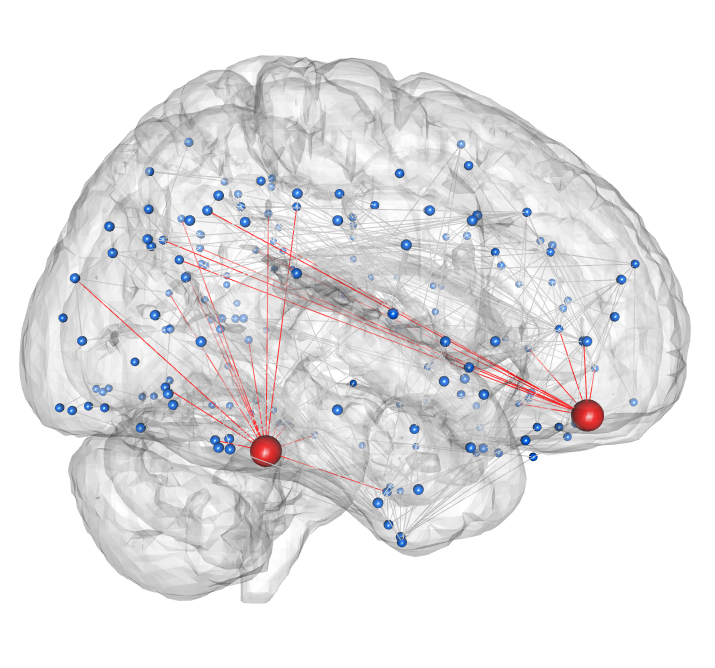}}\;\;
\subfigure[Police press conference scene]{\includegraphics[scale=0.2]{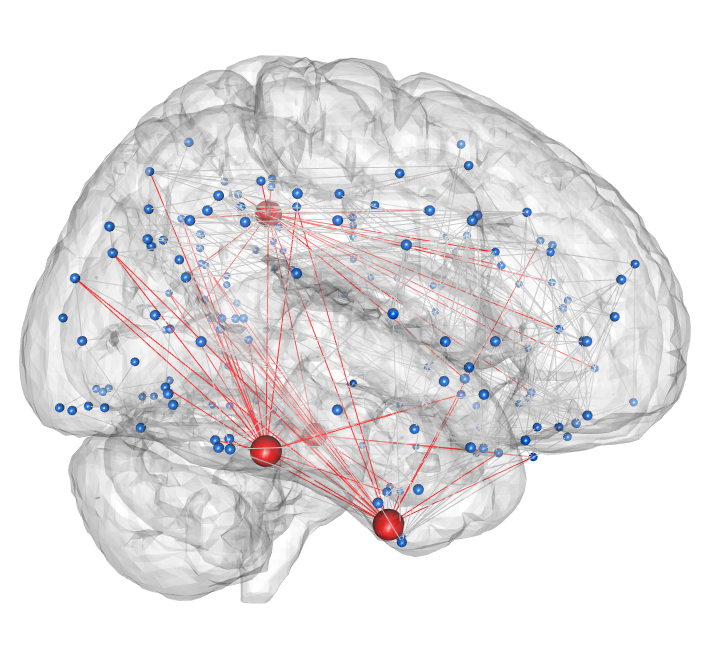}}\;\;
\subfigure[Lab flirting scene]{\includegraphics[scale=0.2]{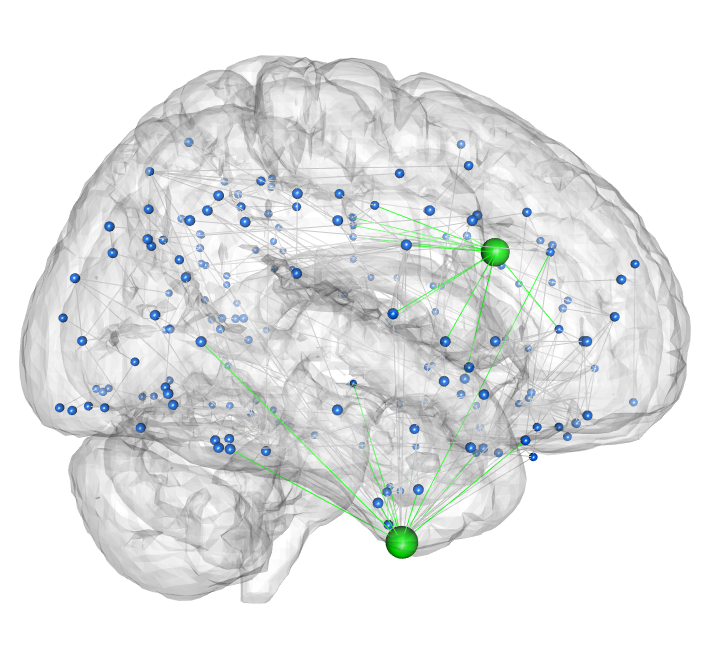}}
\caption{\label{Fig:holmes connectivity} Sagittal snapshots of the rejected edges based on Algorithm~\ref{al:stepdown}. Panels (a)-(c) contain the snapshots for the ``teens in rain", ``police press conference", and ``lab flirting" scenes, respectively.  The red nodes and red edges are regions of interest that have more than 15 rejected edges.  The grey edges are rejected edges from nodes that have no greater than  15 rejected edges. For (c), the green nodes and edges are regions of interest that have more than 12 rejected edges.} 
\end{figure}

From Figure~\ref{Fig:holmes max degree}, we see that the rejected hub nodes (nodes that have more than 15 rejected edges) correspond to the frontal pole (7),  temporal fusiform cortex (16, 100), lingual gyrus (17), and precuneus (102) regions of the brain.
Many studies have suggested that the frontal pole plays significant roles in higher order cognitive operations such as decision making and moral reasoning (among others, \citealp{okuda2003thinking}). 
The fusiform cortex is linked to face and body recognition (see, \citealp{iaria2008contribution}, and the references therein).  
In addition, the lingual gyrus is known for its involvement in processing of visual information about parts of human faces \citep{mccarthy1999electrophysiological}. 
  Thus, it is not surprising that both of these ROIs have more than 15 rejected edges since the brain images are collected while the subjects are exposed to an audio-visual movie stimulus.  

  Compared to the lingual gyrus, temporal fusiform cortex, and the frontal pole, the precuneus is the least well-understood brain literature in the current literature.   We see from Figure~\ref{Fig:holmes max degree} that the precuneus is the most connected ROI across many scenes.  This is supported by the observation in \citet{hagmann2008mapping} where the precuneus serves as  a hub region that is connected to many other parts of the brain.
    In recent years, \citet{lerner2011topographic} and \citet{ames2015contextual} conducted experiments where subjects were asked to listen to a story under an fMRI scanner.  Their results suggest that the precuneus represents high-level concepts in the story, integrating feature information arriving from many different ROIs of the brain.        
Interestingly, we find that 
the precuneus has the highest number of rejected edges during the first half of the movie and that the number of rejected edges decreases significantly during the second half of the movie. 
Our results correspond well to the findings of \citet{lerner2011topographic} and \citet{ames2015contextual} in which the precuneus is active when the subjects comprehend the story.  
However, it also raises an interesting scientific question for future study: is the precuneus active only when the subjects are trying to comprehend the story, that is, once the story is understood, the precuneus is less active.

\section{Theoretical Results} 
\label{section:theory}
We establish uniform rates of convergence for the proposed estimators, and show that the testing procedure in Algorithm~\ref{al:stepdown} is a uniformly valid test.  We study the asymptotic regime in which $n$, $d$, and $s$ are allowed to increase.   In the context of the Sherlock Holmes data set, $n$ is the total number of brain images obtained under the continuous stimulus, $d$ is the number of brain regions, and $s$ is the maximum number of connections for each brain region in the true stimulus-locked  brain connectivity network. The current theoretical results assume that $Z$ is a random variable with continuous density.  Our theoretical results can be easily generalized to the case when $\{Z_i\}_{i\in [n]}$ are fixed.

\subsection{Theoretical Results on Parameter Estimation} 
\label{subsection:estimation error}
Our proposed estimator involves a kernel function $K(\cdot)$: we require $K(\cdot)$ to be symmetric, bounded, unimodal, and compactly supported.  
More formally, for $l=1,2,3,4$, 
\begin{equation}
\label{prop:kernel}
\int K(u) du =1, \quad \int u K(u) du =0, \quad \int u^l K(u) du <\infty, \quad \int K^l(u) du <\infty.
\end{equation}
In addition, we require the total variation of $K(\cdot)$ to be bounded, i.e., $\|K\|_{\mathrm{TV}} < \infty$, where $\|K\|_{\mathrm{TV}} = \int |\dot{K}|$.
In other words, we require the kernel function to be a smooth function.
The Epanechnikov kernel we consider in \eqref{Eq:kernel use} for analyzing Sherlock data satisfies \eqref{prop:kernel}.
A unimodal kernel function is extremely plausible in our setting: for instance, to estimate brain network in the ``police press conference scene", we expect the brain images within that scene to play a larger role than brain images that are far away from the scene. 
One practical limitation of the conditions on the kernel function is the symmetric kernel condition.  
When we are estimating a stimulus-locked brain network for a particular time point, the ideal case is to weight the previous images more heavily than the  future brain images. 
The scientific reasoning is that there may be some time lag for information processing. 
In order to capture this effect, a more carefully designed kernel function is needed and it is out of the scope of this paper.


Next, we impose regularity conditions on the marginal density $f_Z (\cdot)$.
\begin{assumption}
\label{ass:marginal}
There exists a constant $\underline{f}_Z$ such that $\inf_{z\in [0,1]} \; f_Z (z)\ge \underline{f}_Z>0.$  Furthermore, $f_Z$ is twice continuously differentiable and that there exists a constant $\bar{f}_Z < \infty$  such that 
$ \max \;  \{ \|f_Z\|_{\infty},\|\dot{f}_Z\|_{\infty},\|\ddot{f}_Z\|_{\infty}\} \le \bar{f}_Z$.
\end{assumption} 

Next, we impose smoothness assumptions on the inter-subject covariance matrix $\bSigma(\cdot)$. 
Our theoretical results hold for any positive definite subject-specific covariance matrices $\Lb_X(z)$ and $\Lb_Y(z)$, since these matrices are treated as nuisance parameters.

\begin{assumption}
\label{ass:covariance}
There exists a constant $M_{\sigma}$  such that 
\[
\underset{z\in [0,1]}{\sup}\;  \underset{j,k \in [d]}{\max}\; \max  \left\{  
|\bSigma_{jk}(z)| , |\dot{\bSigma}_{jk}(z)| , |\ddot{\bSigma}_{jk}(z)|
 \right\}\le M_{\sigma}.
\]
\end{assumption} 
\noindent In other words, we assume that the inter-subject covariance matrices are smooth and do not change too rapidly in neighboring time points.  This assumption clearly holds in a dynamic brain network where we expect the brain network to change smoothly over time.
Assumptions~\ref{ass:marginal} and \ref{ass:covariance} on $f(z)$ and $\bSigma(z)$ are standard assumptions in the nonparametric statistics literature
 (see, for instance, Chapter 2 of \citealp{pagan1999nonparametric}).

The following theorem establishes the uniform rates of convergence for $\hat{\bSigma}(z)$.
\begin{theorem}
\label{theorem:estimation error}
Assume that $h = o(1)$ and that $\log^2 (d/h) / (nh) = o(1)$.
  Under Assumptions~\ref{ass:marginal}-\ref{ass:covariance}, we have 
  \[
  \underset{z \in [0,1]}{\sup}  \left\|\hat{\bSigma}(z)-\bSigma(z)\right\|_{\max}     =  \cO_P \left\{ h^2 + \sqrt{\frac{\log (d/h)}{nh}}\right\}.
  \]
\end{theorem}
\noindent Theorem~\ref{theorem:estimation error} guarantees that our estimator always converges to the population parameter under the max norm, if the smoothing parameter $h$ goes to zero asymptotically.  For instance, this will satisfy if $h = C \cdot n^{-1/5}$ for some constant $C>0$.  The quantity ${\sup}_{{z \in [0,1]}}  \|\hat{\bSigma}(z)-\bSigma(z)\|_{\max}$  can be upper bounded by the summation of two terms: ${\sup}_{{z \in [0,1]}}  \|\EE[\hat{\bSigma}(z)]-\bSigma(z)\|_{\max}$ and 
${\sup}_{{z \in [0,1]}}  \|\hat{\bSigma}(z)-\EE[\hat{\bSigma}(z)]\|_{\max}$, which are known as the bias and variance terms, respectively, in the kernel smoothing literature
(see, for instance, Chapter 2 of \citealp{pagan1999nonparametric}). 
The term  $h^2$ on the upper bound corresponds to the bias term and the term $\sqrt{\log(d/h)/(nh)}$ corresponds to the variance term.

Next, we establish theoretical results for $\hat{\bTheta}(z)$.
Recall that the stimulus-locked brain connectivity network is encoded by the support of the inverse covariance matrix  $\bTheta (z)$: $\bTheta_{jk}(z) = 0$ if and only if the $j$th and $k$th brain regions are conditionally independent given all of the other brain regions.
We consider the class of inverse covariance matrices:
\begin{equation}
\label{Eq:parameter space}
\cU_{s,M} = \left\{  \bTheta \in \RR^{d\times d} \mid \bTheta \succ 0,\;
\|\bTheta\|_2\le \rho, \;\underset{j\in [d]}{\max} \; \|\bTheta_j \|_0\le s,\; \underset{j\in [d]}{\max}\; \|\bTheta_j\|_1 \le M 
 \right\}.
\end{equation}
Here, $\|\bTheta\|_2$ is the largest singular value of $\bTheta$ and $\|\bTheta_j\|_0$ is the number of non-zeros in $\bTheta_j$. 

Brain connectivity networks are usually densely connected due to the intrinsic-neural and non-neuronal signals, such as background processing.
Instead of assuming an overall sparse brain network, we assume  that the stimulus-locked brain network $\bTheta(z)$ is sparse, and allow the intrinsic brain network unrelated to the stimulus to be dense.
The sparsity assumption on stimulus-locked brain network is plausible in this setting since it characterizes brain activities that are specific to the stimulus.  For instance, we may believe that only certain brain regions are active under cognitive process.
The other conditions are satisfied since $\bTheta(z)$ is the inverse of a positive definite covariance matrix.

Given Theorem~\ref{theorem:estimation error}, the following corollary establishes the uniform rates of convergence for $\hat{\bTheta}(z)$ using the CLIME estimator as defined in  (\ref{Eq:clime}).  
It follows directly from the proof of Theorem 6 in \citet{cai2011constrained}.

\begin{corollary}
\label{theorem:inverse estimation}
Assume that $\bTheta (z)  \in \cU_{s,M}$ for all $z\in [0,1]$.  Let $\lambda \ge C\cdot \{h^2 + \sqrt{\log (d/h)/(nh)}\} $ for $C>0$.  Under the same conditions in Theorem~\ref{theorem:estimation error}, 
\begin{equation}
\label{Eq:inverse1}
\underset{z\in [0,1]}{\sup} \; \left\|   \hat{\bTheta}(z)-\bTheta (z)       \right\|_{\max} =\cO_P \left\{h^2 + \sqrt{\frac{\log (d/h)}{nh}}\right\};
\end{equation}
\begin{equation}
\label{Eq:inverse2}
\underset{z\in [0,1]}{\sup} \; \underset{j\in [d]}{\max}\; \left\|   \hat{\bTheta}_j(z)-\bTheta_j (z)       \right\|_{1} = \cO_P \left[  s\cdot  \left\{h^2 + \sqrt{\frac{\log (d/h)}{nh}}\right\}\right];
\end{equation}
\begin{equation}
\label{Eq:inverse3}
\underset{z\in [0,1]}{\sup}\; \underset{j \in [d]}{\max}  \; \left\|   \left\{\hat{\bTheta}_j(z)\right\}^T \hat{\bSigma}(z)-\mathbf{e}_j       \right\|_{\infty} =\cO_P  \left\{h^2 + \sqrt{\frac{\log (d/h)}{nh}}\right\}.
\end{equation}
\end{corollary}
\noindent In the real data analysis, Corollary~\ref{theorem:inverse estimation} is helpful in terms of selecting the sparsity tuning parameter $\lambda$: it motivates a sparsity tuning parameter of the form $\lambda \ge C\cdot \{h^2 + \sqrt{\log (d/h)/(nh)}\}$ to guarantee statistically consistent estimated stimulus-locked brain network. 
To select the constant $C$, we consider a sequence of number and select the appropriate $C$ using a data-driven cross-validation procedure in \eqref{eq:cv}.

\subsection{Theoretical Results on Topological Inference} 
\label{subsection:testing}
In this section, we first  show that the distribution of the test statistic $T_E$ can be approximated by the conditional $(1-\alpha)$-quantile of the bootstrap statistic $T^B_E$.
  Next, we show that the proposed testing method in Algorithm~\ref{al:stepdown} is valid in the sense that the type I error can be controlled at a pre-specified level $\alpha$.

Recall from (\ref{Eq:quantile estimate}) the definition of $c(1-\alpha,E)$. 
The following theorem shows that the Gaussian multiplier bootstrap is valid for approximating the quantile of the test statistic $T_E$ in (\ref{Eq:quantile}).
Our results are based on the series of work on Gaussian approximation on multiplier bootstrap in high dimensions (see, e.g., \citealp{chernozhukov2013gaussian,chernozhukov2014gaussian}).   We see from \eqref{Eq:quantile} that $T_E$ involves taking the supremum over $z\in [0,1]$ and a dynamic edge set $E(z)$. 
Due to the dynamic edge set $E(z)$, existing theoretical results for the  Gaussian multiplier bootstrap methods cannot be directly applied. We construct a novel Gaussian approximation result for the supreme of empirical processes of  $T_E$ by carefully characterizing the capacity of the dynamic edge set $E(z)$. 

\begin{theorem}
\label{theorem:gaussian multiplier bootstrap}
Assume that  
$\sqrt{nh^5} + s\cdot \sqrt{nh^9} =o (1)$.
In addition, assume that 
$  s\sqrt{\log^4 (d/h)/(nh^2)}+ \log^{22}(s) \cdot \log^8(d/h)/(nh)  = o(1)$. 
Under the same conditions in Corollary~\ref{theorem:inverse estimation}, we have 
\[
\underset{n\rightarrow \infty}{\lim} \; \underset{\bTheta (\cdot) \in \cU_{s,M}}{\sup}
\; P_{\bTheta (\cdot)} \left\{ T_E \ge c(1-\alpha,E)      \right\} \le \alpha.
\]
\end{theorem}
Some of the scaling conditions are standard conditions in nonparametric estimation
 \citep{tsybakov09introduction}. 
The most notable scaling conditions are
$s\sqrt{\log^4 (d/h)/(nh^2)}=o(1)$ and $ \log^{22}(s) \cdot \log^8(d/h)/(nh) = o(1)$: these conditions arise from Gaussian approximation on multiplier bootstrap \citep{chernozhukov2013gaussian}.
These scaling conditions will hold asymptotically as long as the number of brain images $n$ is much larger than the maximum degree in the graph $s$.  
This corresponds well with the real data analysis where we expect only certain ROIs are active during information processing.

Recall the hypothesis testing problem in \eqref{max test}.
We now show that the type I error of the proposed inferential method for testing the maximum degree of a time-varying graph can be controlled at a pre-specified level $\alpha$.
\begin{theorem}
\label{theorem:algorithm}
Assume that the same conditions in Theorem~\ref{theorem:gaussian multiplier bootstrap} hold.  Under the null hypothesis in \eqref{max test}, we have 
\[
\underset{n\rightarrow \infty}{\lim} P_{\mathrm{null}}(\mathrm{Algorithm~\ref{al:stepdown}~rejects~the~null~hypothesis}) \le \alpha.
\]
\end{theorem}

To study the power analysis of the proposed method, we define the signal strength of a precision matrix $\bTheta$ as
\begin{equation}\label{eq:strength}
   \text{Sig}_{\mathrm{deg}}(\bTheta):= \max_{{E' \subseteq E(\bTheta), \text{Deg}(E) > k }} \min_{e \in E'} |\bTheta_e|,
\end{equation}
where $\text{Deg}(E)$ is the maximum degree of graph $G=(V,E)$.
Under the alternative hypothesis in \eqref{max test}, $\mathrm{there\; exists \;a\;} z_0\in [0,1] \mathrm{\;such\; that \; }$ the maximum degree of the graph is greater than $k$.  We define the parameter space under the alternative:
\begin{equation}
\label{eq:G1}
   \cG_{1}(\theta) = \Big[ \bTheta(\cdot) \in\cU_{s,M} \,\Big|\,  \text{Sig}_{\mathrm{deg}}\{\bTheta(z_0)\}\ge \theta\text{ for some $z_0 \in [0,1]$}\Big].
\end{equation}
The following theorem presents the power analysis of Algorithm~\ref{al:stepdown}.
\begin{theorem}
\label{theorem:algorithm-power}
Assume that the same conditions in Theorem~\ref{theorem:gaussian multiplier bootstrap} hold and select the smoothing parameter such that $h =o(n^{-1/5})$.  Under the alternative hypothesis in~\eqref{max test} and the assumption that $\theta \ge C \sqrt{\log (d/h)/nh}$, where $C$ is a fixed large constant, we have
  \begin{equation}\label{eq:power}
 \lim_{n \rightarrow \infty}   \inf_{\bTheta \in \cG_1(\theta)}\PP_{\bTheta}(\mathrm{Algorithm~\ref{al:stepdown}~rejects~the~null~hypothesis}) = 1,
\end{equation}
 for any fixed $\alpha \in (0,1)$.
\end{theorem}
\noindent The signal strength condition defined in \eqref{eq:strength} is weaker than the typical minimal signal strength condition required on testing a single edge on a conditional independent graph, $\min_{e \in E(\bTheta)} |\bTheta_e|$. 
The condition in \eqref{eq:strength} requires only that there exists a subgraph whose maximum degree is larger than $k$ and the minimal signal strength on that subgraph is above certain level.  
In our real data analysis, this requires only the edges for brain regions that are highly connected to many other brain regions to be strong, which is plausible since these regions should have high brain activity.

\section{Discussion} 
\label{section:discussion}
We consider estimating stimulus-locked brain connectivity networks from data obtained under natural continuous stimuli.
Due to lack of highly controlled experiments that remove all spontaneous and individual variations, the measured brain signal consists of not only stimulus-induced signal, but also intrinsic neural signal and non-neuronal signals that are subjects specific.
Typical approach for estimating time-varying Gaussian graphical models will fail to estimate the stimulus-locked brain connectivity network accurately due to the presence of subject specific effects.
By exploiting the experimental design aspect of the problem, we propose a simple approach to estimating stimulus-locked brain connectivity network.
In particular, rather than calculating within-subject smoothed covariance matrix as in the typical approach for modeling time-varying Gaussian graphical models, we propose to construct the inter-subject smoothed covariance matrix instead, treating the subject specific effects as nuisance parameters.
  
To answer the scientific question on whether there are any brain regions that are connected to many other brain regions during the given stimulus, we propose an inferential method for testing the maximum degree of a stimulus-locked time-varying graph.
In our analysis, we found that several interesting brain regions such as the fusiform cortex, lingual gyrus, and precuneus are highly connected.
From the neuroscience literature, these brain regions are mainly responsible for 
high order cognitive operations, face and body recognition, and serve as control region that integrates information from other brain regions.
We have also extended the proposed inferential framework to testing various topological graph structures.  These are detailed in Appendix A.

The practical limitation of our proposed method is on the Gaussian assumption on the data. While we focus on the time-varying Gaussian graphical model in this paper, our framework can be extended to other types of time-varying graphical models such as the time-varying discrete graphical model or the time-varying nonparanormal graphical model \citep{kolar2010estimating,lu2015post2}.
Another limitation is the independence assumption on the data across time points.  All of our theoretical results can be generalized to the case when the data across time points are correlated, and we leave such generalization for future work.

\section*{Acknowledgement}
We thank Hasson's lab at the Princeton Neuroscience Institute for providing us the fMRI data set under audio-visual movie stimulus.  We thank Janice Chen for very helpful conversations on preprocessing the fMRI data set and interpreting the results of our analysis.

\newpage
\appendix



\section{Inference on Topological Structure of Time-Varying Graph} 
\label{section:inference}
In this section, we generalize Algorithm~\ref{al:stepdown} in the main manuscript to testing various graph structures that satisfy the \emph{monotone graph property}.
In \ref{subsection:notation on graph}, we briefly introduce some concepts on graph theory.  These include the notion of isomorphism, graph property, monotone graph property, and critical edge set.  
In \ref{subsection:test statistics}, we provide a test statistic and an estimate of the quantile of the proposed test statistic using the Gaussian multiplier bootstrap.  We then develop an algorithm to test the dynamic topological structure of a time-varying graph which satisfies the monotone graph property.

\subsection{Graph Theory}
\label{subsection:notation on graph}
Let $G = (V,E)$ be an undirected graph where $V= \{1,\ldots,d\}$ is a set of nodes and $E \subseteq V\times V$ is a set of edges connecting pairs of nodes. 
Let $\cG$ be the set of all graphs with the same number of nodes.  For any two graphs $G= (V,E)$ and $G'= (V,E')$, we write $G\subseteq G'$
if $G$ is a subgraph of $G'$, that is, if $E\subseteq E'$.
We start with introducing some concepts on graph theory (see, for instance, Chapter 4 of \citealp{lovasz2012large}).
\begin{definition}
\label{def:isomorphic}
Two graphs $G= (V,E)$ and $G'=(V,E')$ are said to be isomorphic if 
there exists permutations $\pi : V \rightarrow V$ such that 
 $(j,k)\in E$ if and only if $\{\pi(j),\pi(k)\} \in E'$.
\end{definition} 
\noindent 
 The notion of isomorphism is used in the graph theory literature to quantify whether two graphs have the same topological structure, up to any permutation of the vertices (see Chapter 1.2 of \citealp{bondy1976graph}).  We provide two concrete examples on the notion of isomorphism  in Figure~\ref{Fig:isomorphism}.  
 
\begin{figure}[htp]
\centering
\includegraphics[scale=0.325]{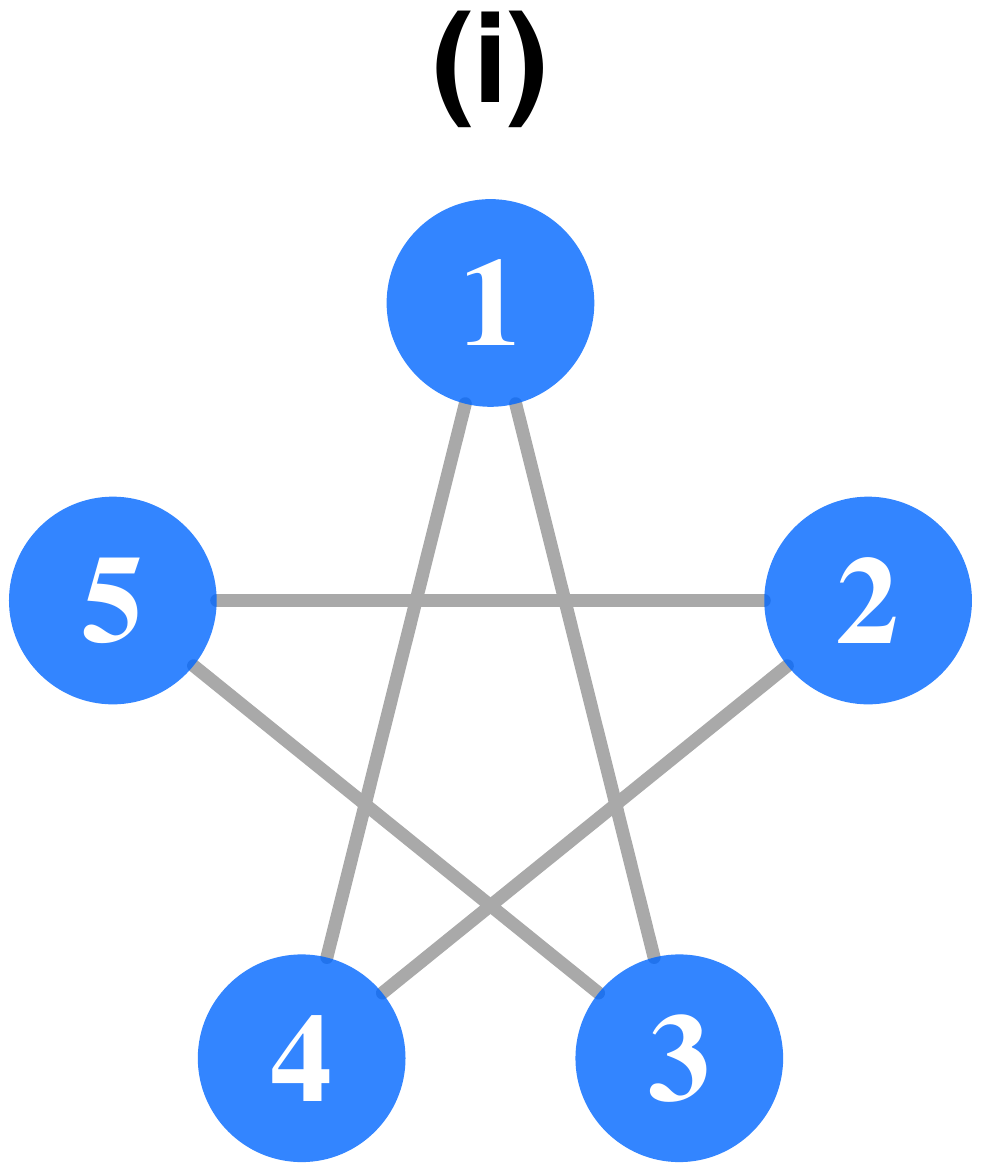}\qquad
\includegraphics[scale=0.325]{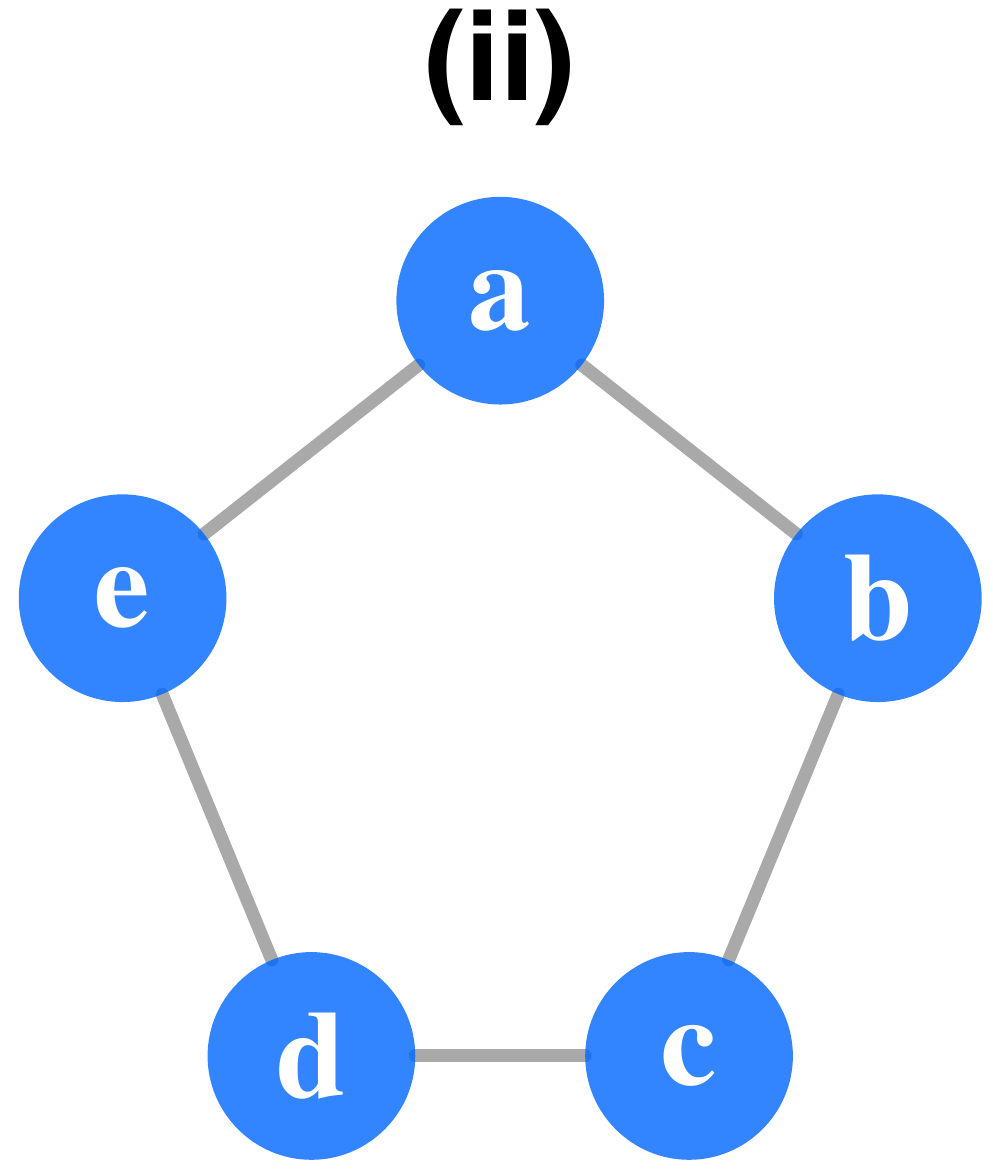}\qquad \qquad   
\includegraphics[scale=0.325]{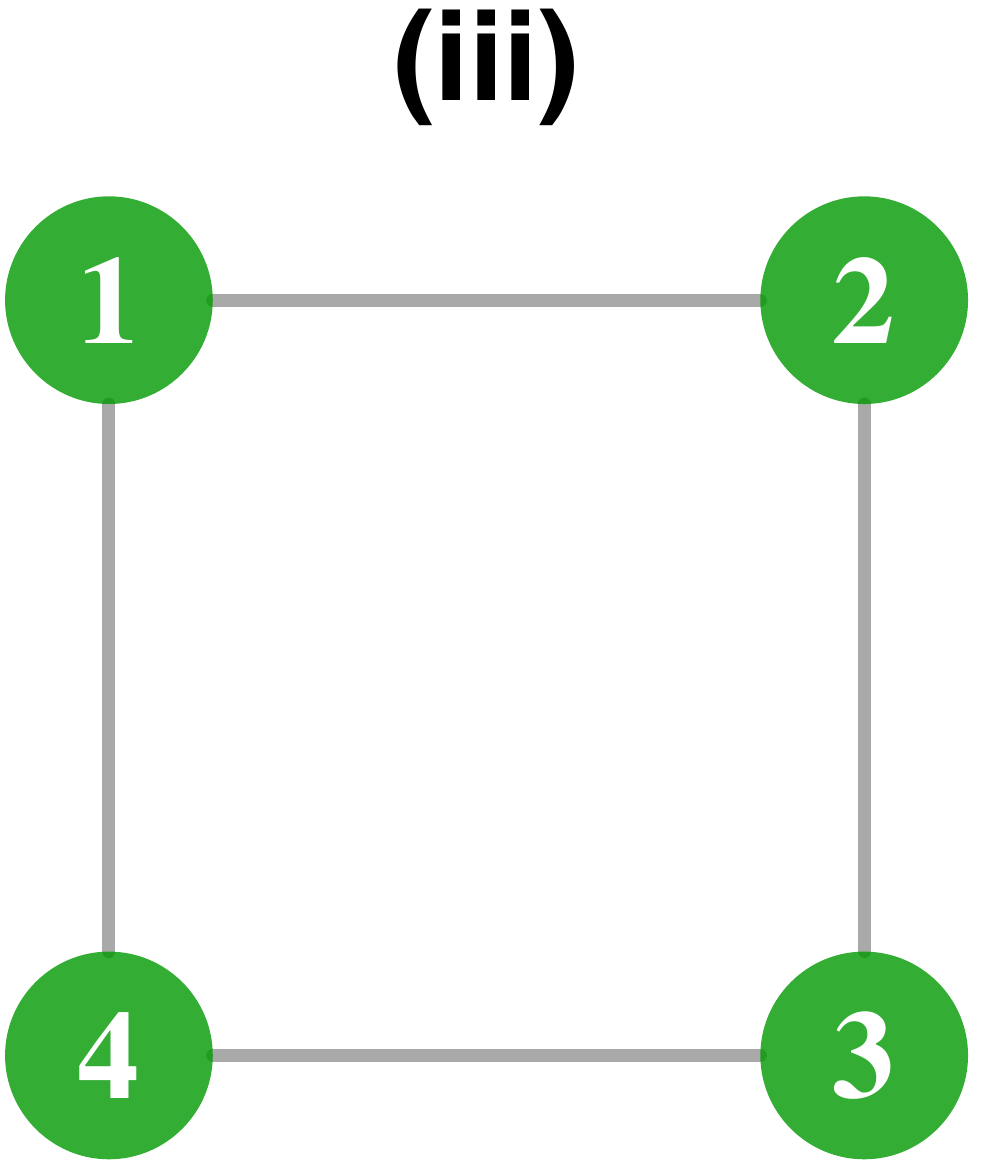}\qquad
\includegraphics[scale=0.325]{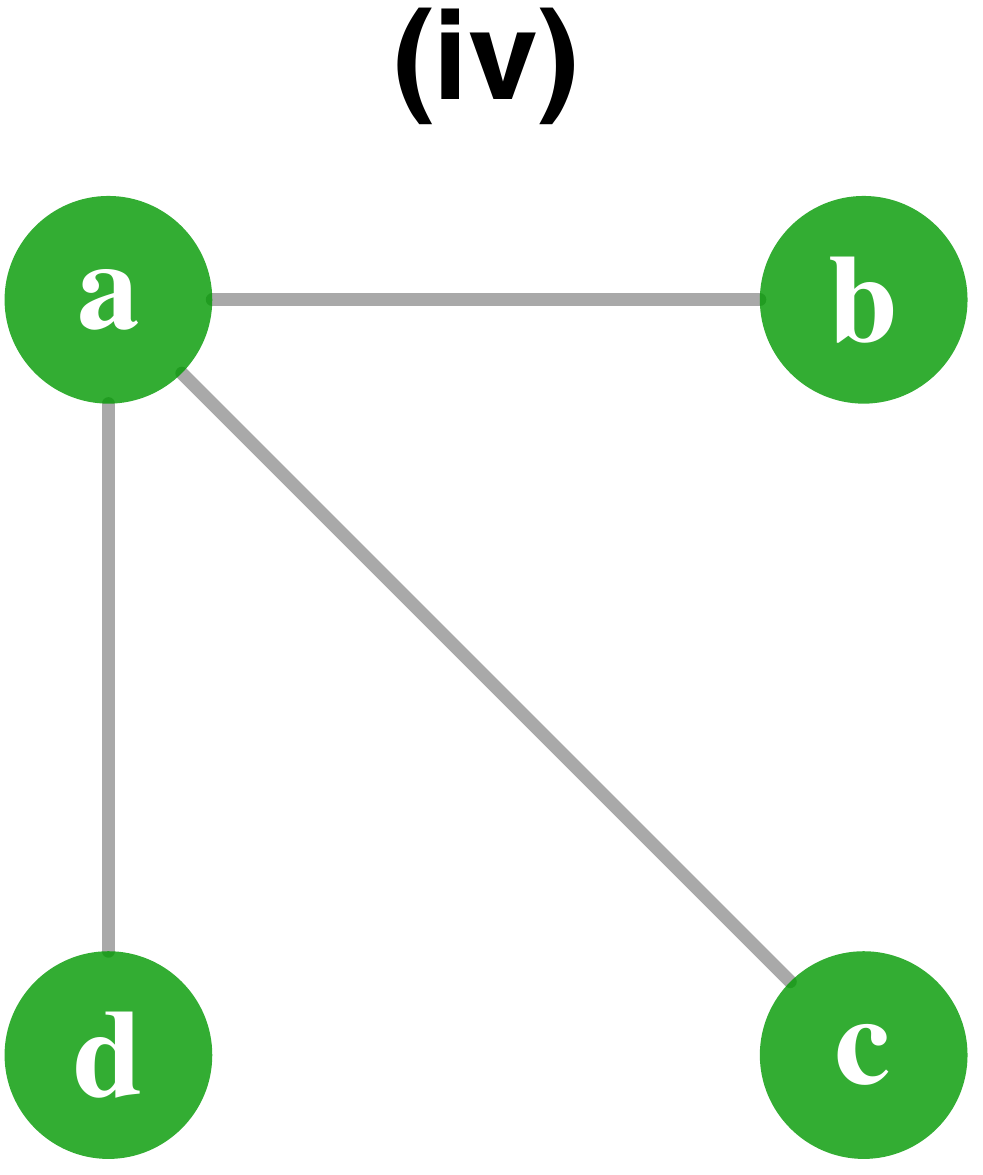}
\caption{Graphs (i) and (ii) are isomorphic.  Graphs (iii) and (iv) are not isomorphic.  }
\label{Fig:isomorphism}
\end{figure}

Next, we introduce the notion of graph property. A graph property is a property of graphs that depends only on the structure of the graphs, that is, a graph property is invariant under permutation of vertices.
A formal definition is given as follows.
\begin{definition}
\label{def:property}
For two graphs $G$ and $G'$ that are isomorphic,  a graph property is a function $\cP : \cG \rightarrow\{0,1\}$ such that $\cP(G)=\cP(G')$.
A graph $G$ satisfies the graph property $\cP$ if $\cP(G)=1$. 
\end{definition}
Some examples of graph property are that the graph is connected, the graph has no more than $k$ connected components,  
the maximum degree of the graph is larger than $k$, the graph has no more than $k$ isolated nodes, the graph contains a clique of size larger than $k$, and the graph contains a triangle.
For instance, the two graphs in Figures~\ref{Fig:isomorphism}(i) and~\ref{Fig:isomorphism}(ii) are isomorphic and satisfy the graph property of being connected.

\begin{definition}
\label{def:monotone}
For two graphs $G\subseteq G'$, a graph property $\cP$ is monotone if $\cP(G)=1$ implies that $\cP(G')=1$.
\end{definition}
In other words, we say that a graph property is monotone if the graph property is preserved under the addition of new edges.  
Many graph property that are of interest such as those given in the paragraph immediately after Definition~\ref{def:property} are monotone.  In Figure~\ref{Fig:monotone}, we present several examples of graph property that are monotone by showing that adding additional edges to the graph does not change the graph property. For instance, we see from Figure~\ref{Fig:monotone}(a) that the existing graph with gray edges are connected.  Adding the red edges to the existing graph, the graph remains connected and therefore the graph property is monotone.  
Another example is the graph with maximum degree at least three as in Figure~\ref{Fig:monotone}(c). We see that 
adding the red dash edges to the graph preserves the graph property of having maximum degree at least three.

\begin{figure}[!htp]
\centering
\includegraphics[scale=0.28]{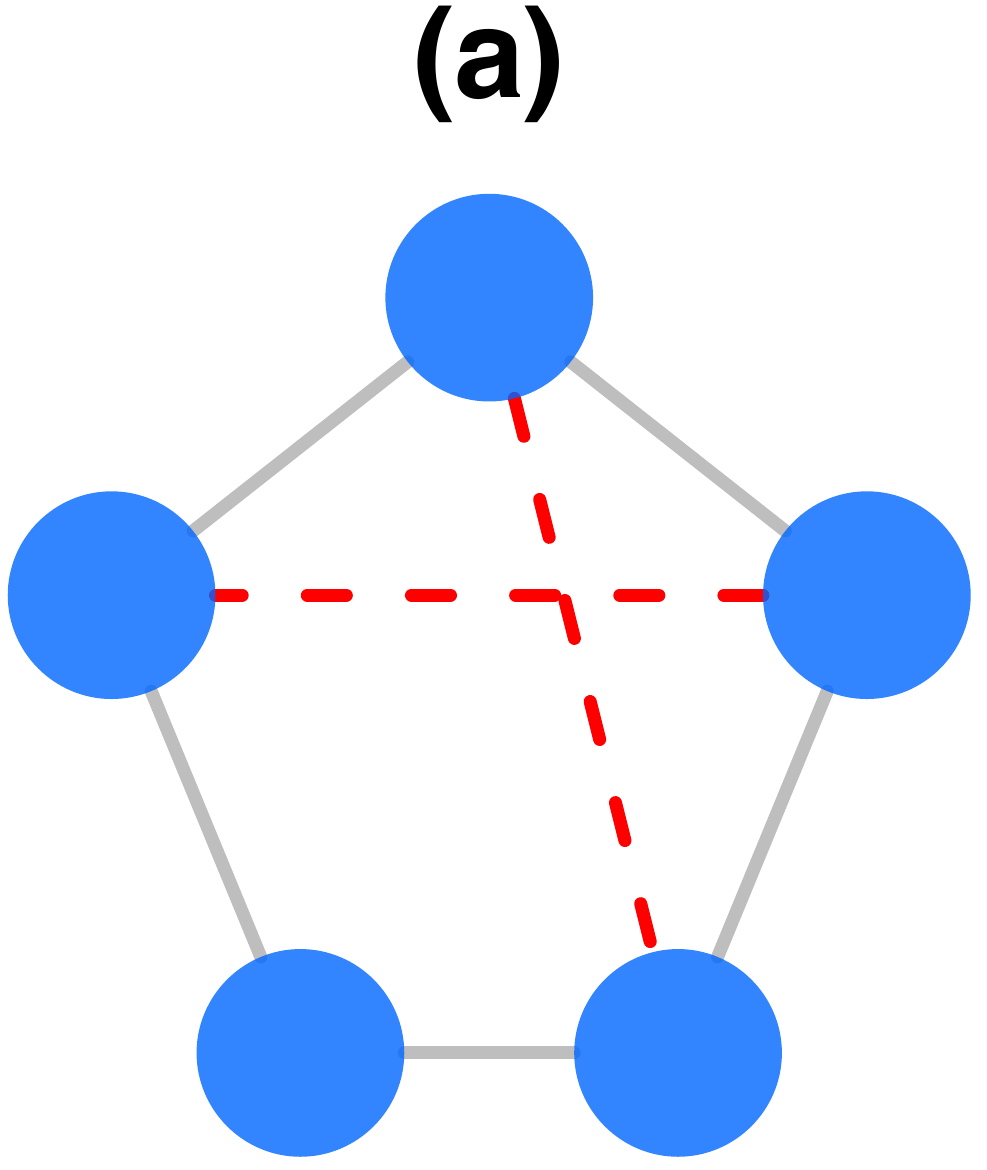}\quad \qquad
\includegraphics[scale=0.28]{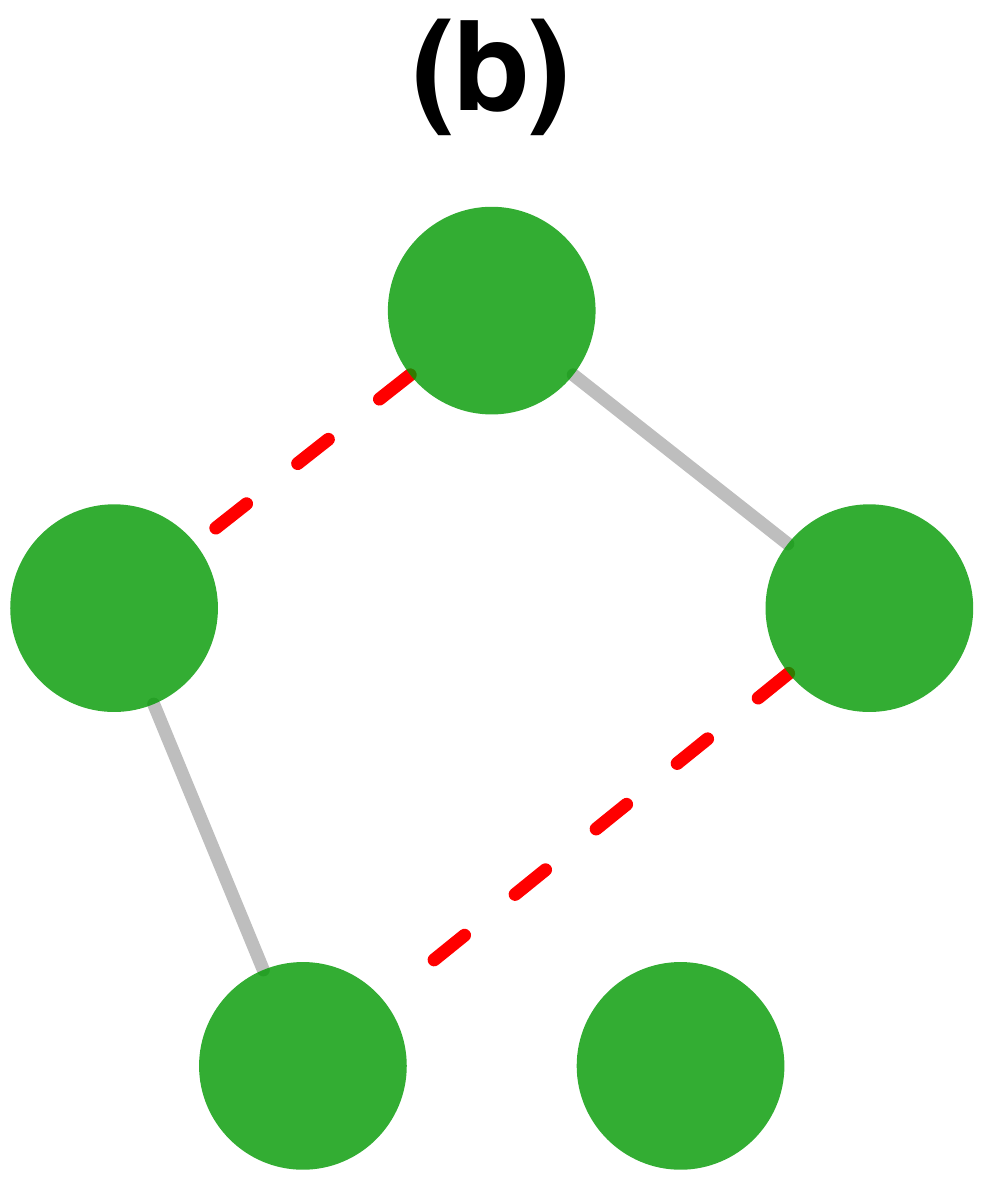}\quad \qquad
\includegraphics[scale=0.28]{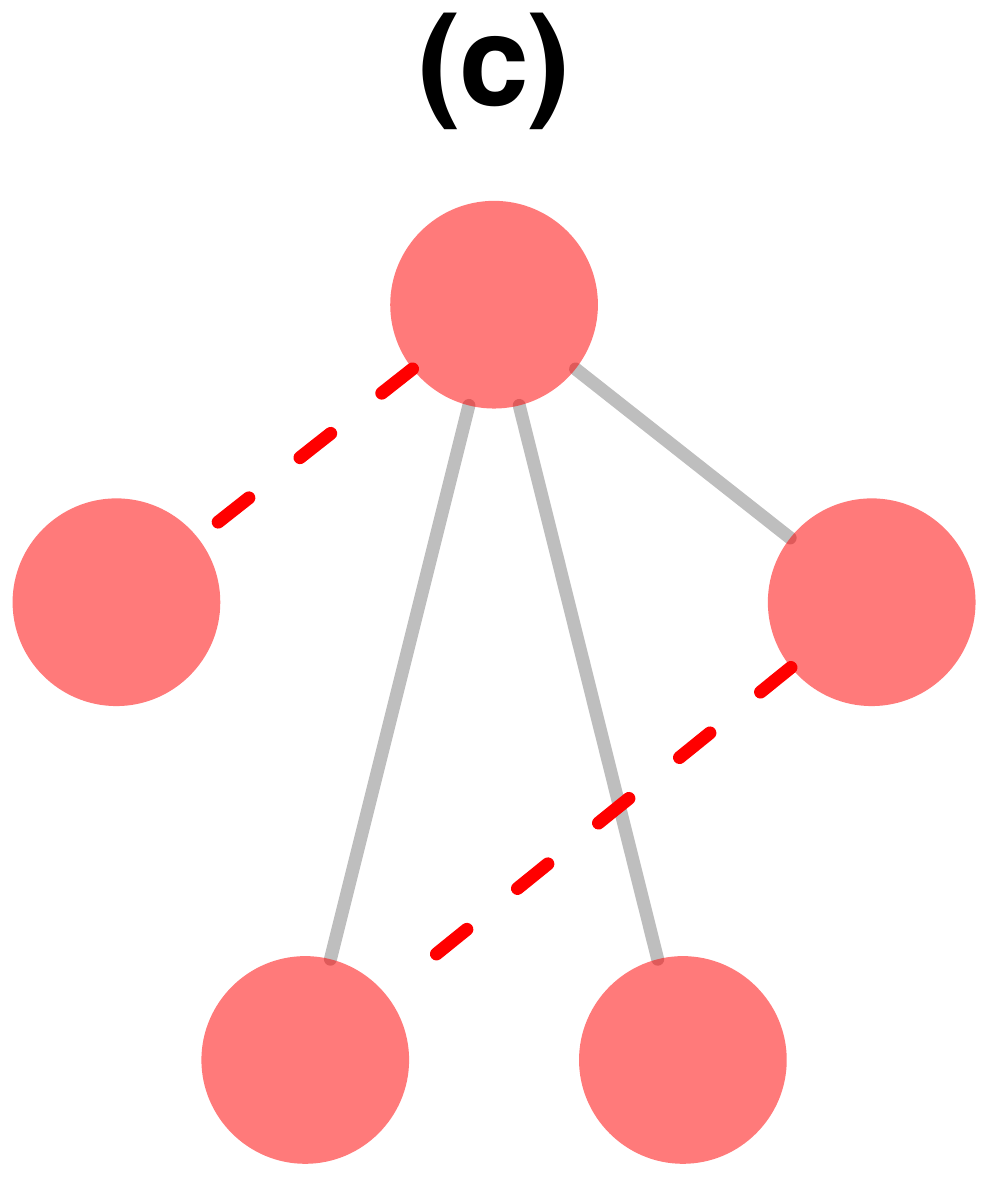}\quad \qquad
\includegraphics[scale=0.28]{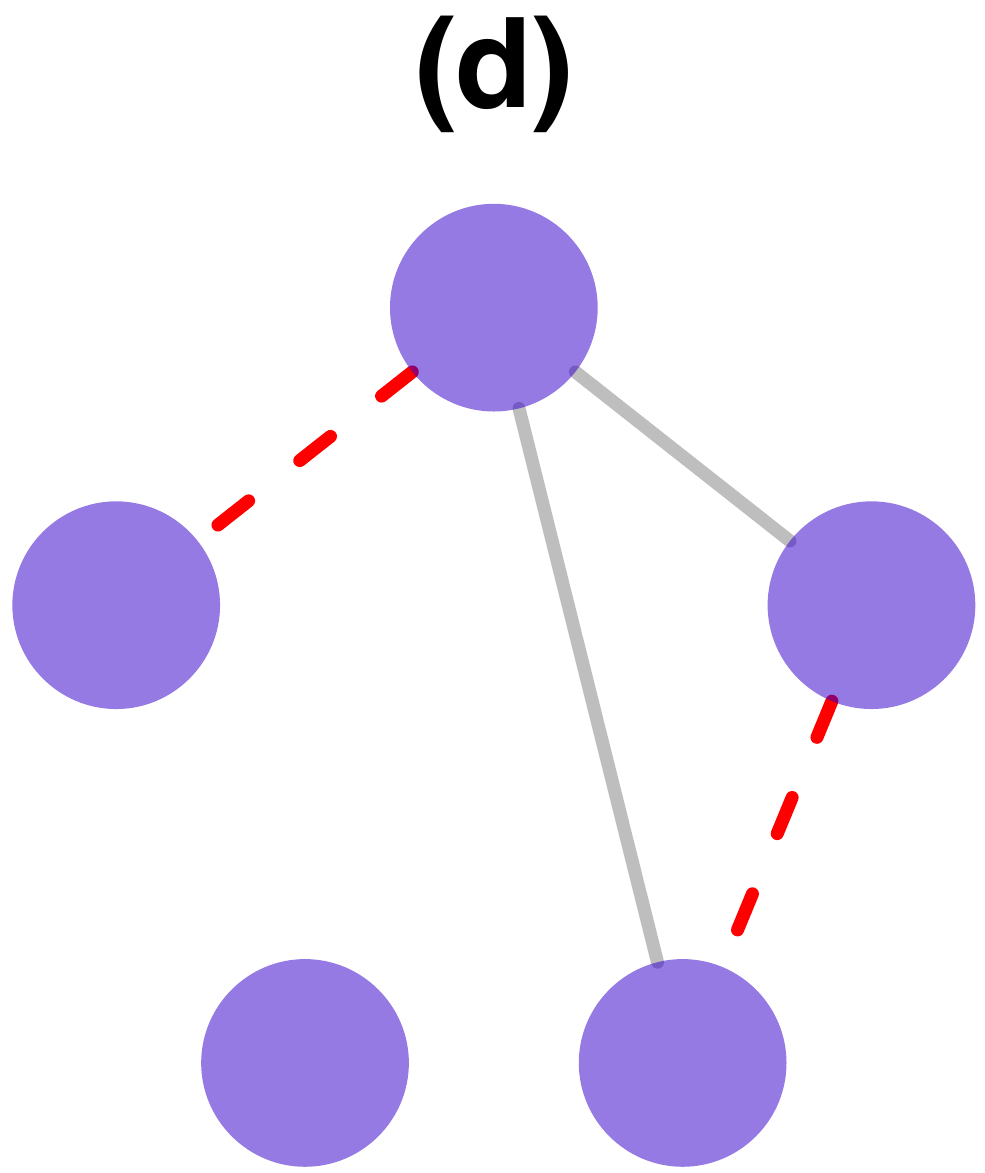}\quad \qquad
\caption{Some examples on graph property that are monotone.  The gray edges are the original edges and the red dash edges are additional edges added to the existing graph.  (a) Graph that is connected.  (b) Graph that has no more than three connected components. (c) Graph with maximum degree at least three.  (d) Graph with no more than two isolated nodes.  Adding the red dash edges to the existing graphs does not change the graph property.  }
\label{Fig:monotone}
\end{figure}

For a given graph $G=(V,E)$, we define the class of edge sets satisfying the graph property $\cP$ as
\begin{equation}
\label{Eq:class of edge sets}
\mathscr{P}= \{      E\subseteq V\times V \mid \cP(G)=1\}.
\end{equation}
Finally, we introduce the notion of critical edge set in the following definition.
\begin{definition}
\label{def:critical}
 Given any edge set $E \subseteq V \times V$, we define the critical edge set of $E$ for a given monotone graph property $\cP$ as
\begin{equation}\label{eq:critical}
 \cC(E, \cP) = \{e \mid e \not\in E, \;   \mathrm{there\; exists\;} E' \supseteq E \mathrm{ \;such\; that\; } E' \in \mathscr{P} \text{ and }   E' \backslash \{e\}  \notin \mathscr{P} \}.
\end{equation}
\end{definition}
\noindent For a given monotone graph property $\cP$, the critical edge set is the set of edges that will change the graph property of the graph once added to the existing graph.  We provide two examples in  Figure~\ref{Fig:crit edges}. 
Suppose that $\cP$ is the graph property of being  connected.  In Figure~\ref{Fig:crit edges}(a), we see that the graph is not connected, and thus $\cP(G)=0$.  Adding any of the red dash edges in Figure~\ref{Fig:crit edges}(b) changes $\cP(G)=0$ to $\cP(G)=1$. 

\begin{figure}[!htp]
\centering
\includegraphics[scale=0.315]{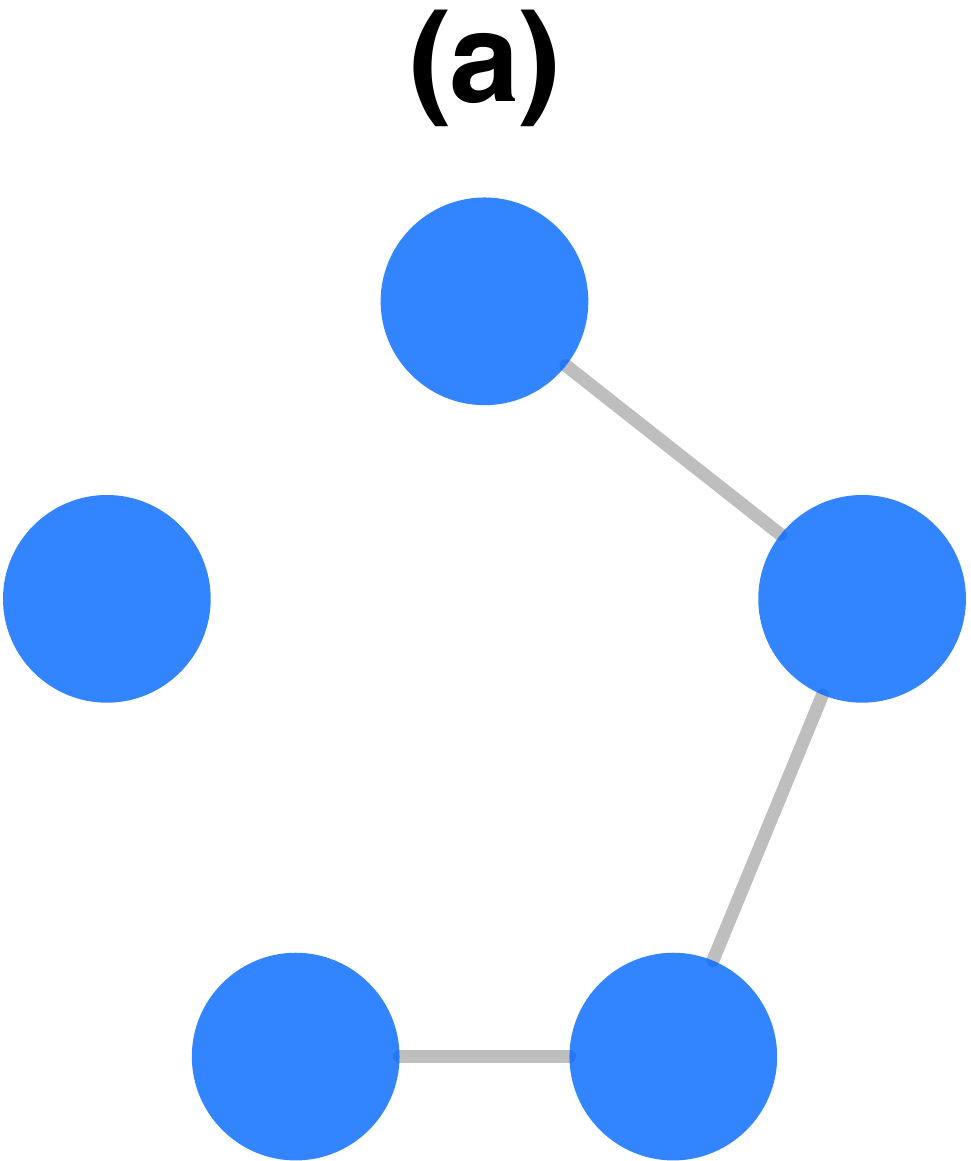}\quad \qquad
\includegraphics[scale=0.315]{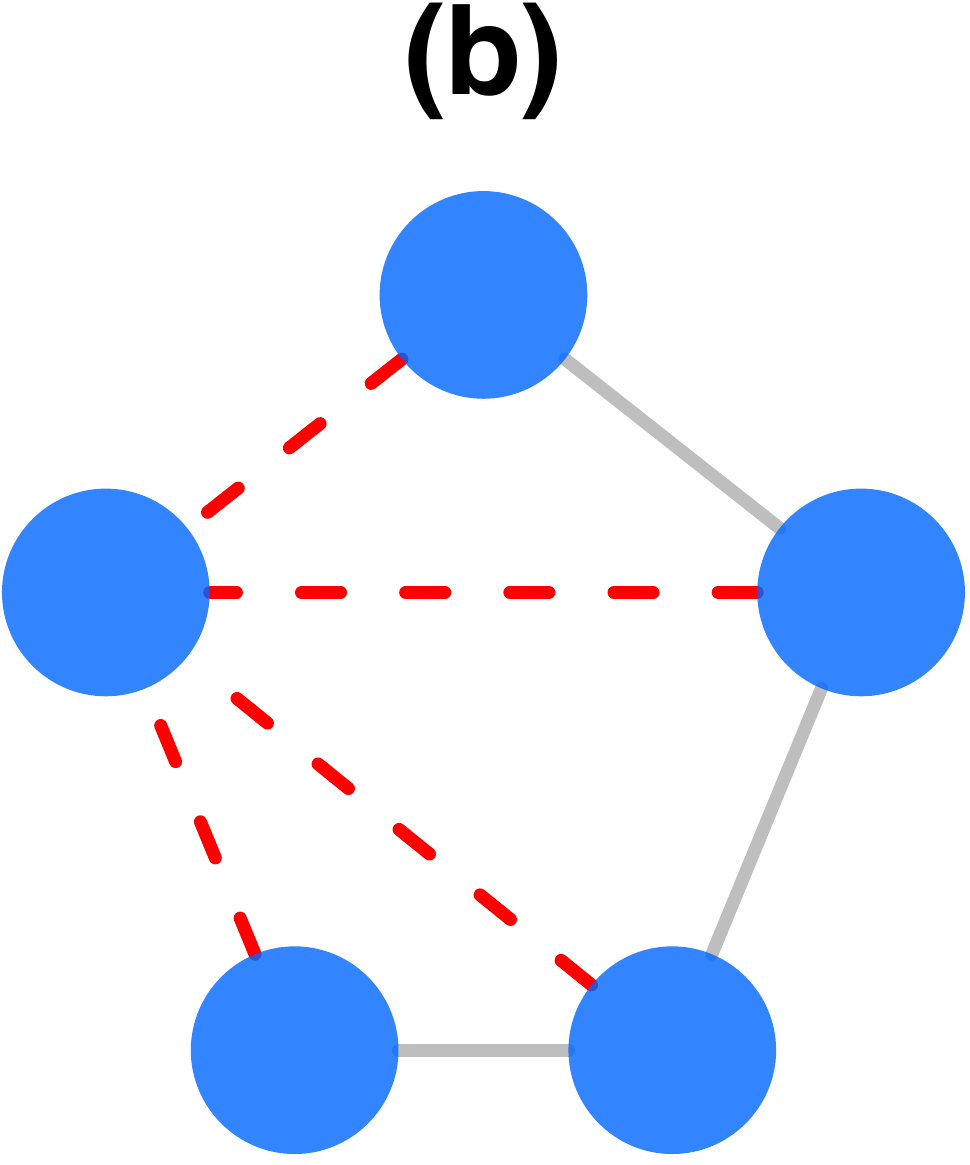}\quad \qquad
\includegraphics[scale=0.315]{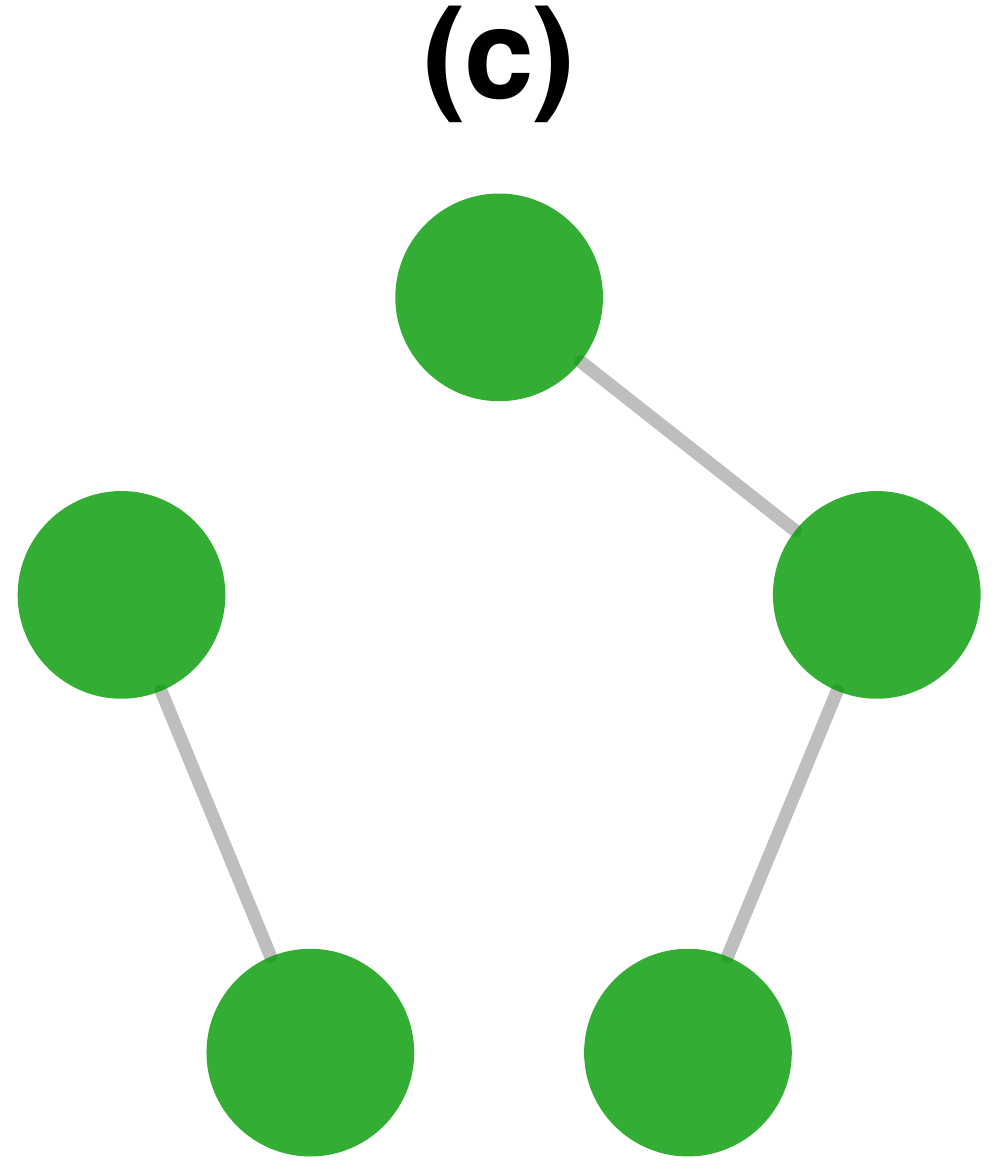}\quad \qquad
\includegraphics[scale=0.315]{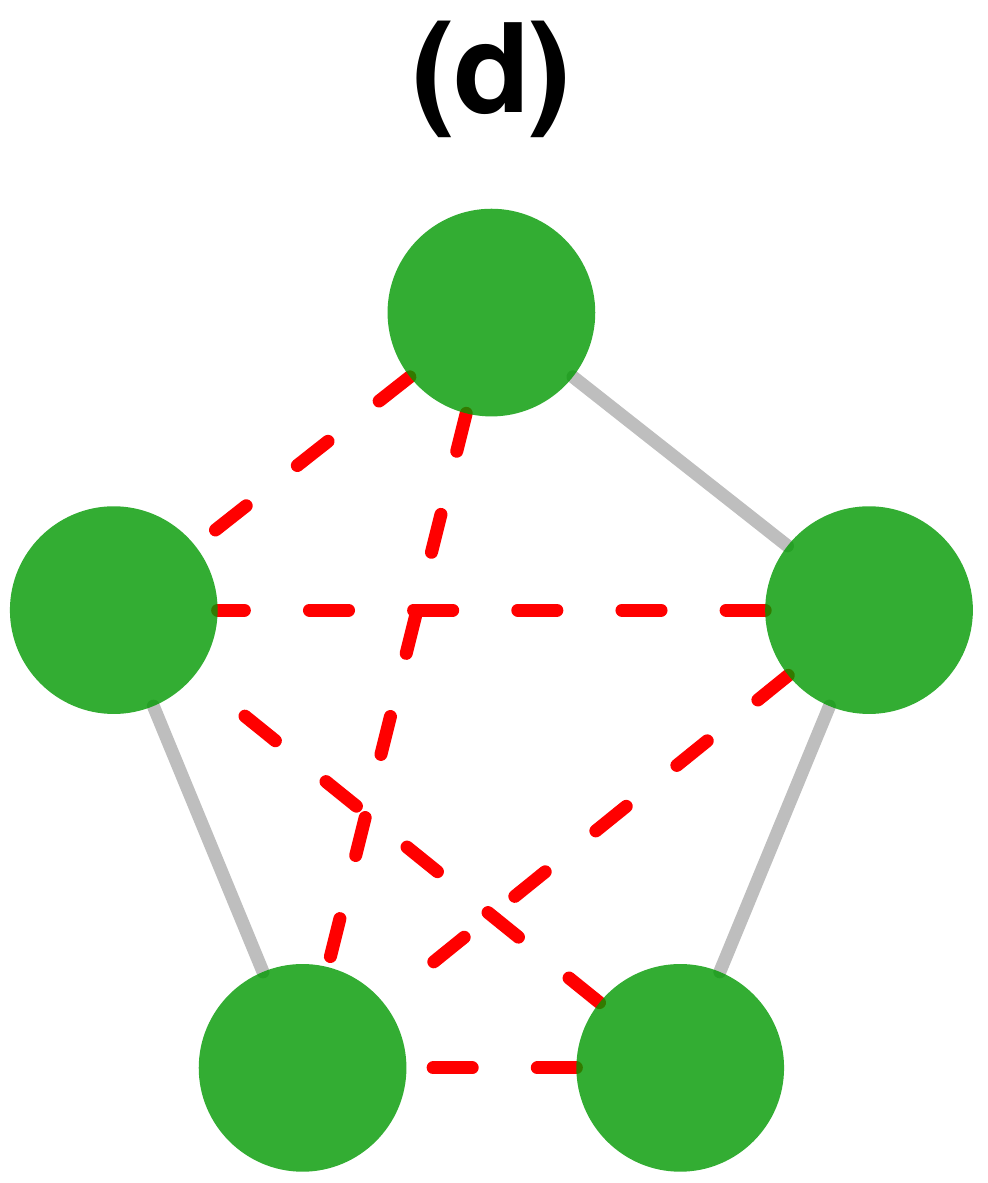}\quad \qquad
\caption{Let $\cP$ be the graph property of being connected. Gray edges are the original edges of a graph $G$ and the red dash edges are the critical edges that will change the graph property from $\cP(G)=0$ to $\cP(G)=1$.  (a) The graph satisfies $\cP(G)=0$.  (b) The graph property changes from $\cP(G)=0$ to  $\cP(G)=1$ if some red dash edges are added to the graph.  (c) The graph satisfies $\cP(G) = 0$.  (d) 
The graph property changes from $\cP(G)=0$ to $\cP(G)=1$ if some red dash edges are added to the existing graph.  }
\label{Fig:crit edges} 
\end{figure}

\subsection{An Algorithm for Topological Inference}
\label{subsection:test statistics}

Throughout the rest of the paper, we denote $G(z) = \{V,E(z)\}$ as the graph at  $Z=z$.  
We consider hypothesis testing problem of the  form
\begin{equation}
\label{Eq:hypothesis}
\begin{split}
&H_0 : \cP\{G(z)\} =0 \mathrm{\;for\; all\;} z\in [0,1] \\
&H_1: \mathrm{there\; exists \;a\;} z_0\in [0,1] \mathrm{\;such\; that \; }\cP\{G(z_0)\}=1,
\end{split}
\end{equation}
where $G(\cdot)$ is the true underlying graph and $\cP$ is a given monotone graph property as defined in Definition~\ref{def:monotone}. 
We provide two concrete examples of the hypothesis testing problem in (\ref{Eq:hypothesis}).
\begin{example}
\label{example:connectivity}
Number of connected components:
 \begin{equation*}
\begin{split}
&H_0: \mathrm{\;for \; all\;} z\in [0,1], \mathrm{\;the \; number\; of\; connected\; components \; is \; greater \; than\;} k,\\
&H_1: \mathrm{\;there \; exists\;a \;} z_0\in [0,1] \mathrm{\;such \; that\; the \; number\; of\; connected\; components\; is \; not\; greater \; than\;} k.\\
\end{split}
\end{equation*}
\end{example}
\begin{example}
\label{example:max degree}
Maximum degree of the graph:
\begin{equation*}
\begin{split}
&H_0: \mathrm{\;for \; all\;} z\in [0,1], \mathrm{\;the \; maximum\; degree\; of\; the\; graph\; is \;not\; greater\; than\;} k,\\
&H_1: \mathrm{\;there \; exists\;a \;} z_0\in [0,1] \mathrm{\;such \; that\; the \; maximum\; degree\; of\; the\; graph\; is \; greater\; than\;}k.\\
\end{split}
\end{equation*}
\end{example}

We now propose an algorithm to test the topological structure of a time-varying graph.  The proposed algorithm is very general and is able to test the hypothesis problem of the form in \eqref{Eq:hypothesis}. Our proposed algorithm is motivated by the step-down algorithm in \citet{romano2005exact} for testing multiple hypothesis simultaneously.
 The main crux of our algorithm is as follows. 
By Definition~\ref{def:critical},
 the critical edge set $\cC\{E_{t-1}(z),\cP\}$ contains edges that may change the graph property from $\cP\{G(z)\}=0$ to $\cP\{G(z)\}=1$.  Thus, at the $t$-th iteration of the proposed algorithm, it suffices to test whether the edges on the critical edge set $\cC\{E_{t-1}(z),\cP\}$ are rejected.
 Let $E_t (z) = E_{t-1}(z)\; \cup \; \cR(z)$, where $\cR(z)$ is the rejected edge set from the critical edge set $\cC\{E_{t-1}(z),\cP\}$.
 Since $\cP$ is a monotone graph property, 
 if there exists a $z_0\in [0,1]$ such that $E_t (z_0) \in \mathscr{P}$, we directly reject the null hypothesis $H_0 : \cP\{G(z)\}=0$ for all $z$.  This is due to the definition of monotone graph property that adding more edges does not change the graph property. 
 If $E_t(z_0)\notin \mathscr{P}$, we repeat this process until the null hypothesis is rejected or no more edges in the critical edge set are rejected.
We summarize the procedure in Algorithm~\ref{al:stepdown2}.  

\begin{algorithm}[ht]
\caption{\label{al:stepdown2} Dynamic skip-down method.}
\textbf{Input:} A monotone graph property $\cP$; $\hat{\bTheta}^\mathrm{de}(z)$ for $z\in [0,1]$.  \\
\textbf{Initialize:} $t=1$; $E_0 (z) = \emptyset$ for $z\in [0,1]$.  \\
\textbf{Repeat:} 
\begin{enumerate}
\item Compute the critical edge set $\cC\{E_{t-1}(z),\cP\}$ for $z\in[0,1]$ and the conditional  quantile $c\{1-\alpha,\cC(E_{t-1},\cP)\} =  \inf \left(t\in \RR   \mid P[T^B_{\cC(E_{t-1},\cP)}  \le t \mid \{(\bX_i,\bY_i,Z_i)\}_{i\in[n]}]   \ge 1-\alpha     \right)$,
where 
$T^B_{\cC(E_{t-1},\cP)}$ is the bootstrap statistic defined in (\ref{Eq:approx test}) with the maximum taken over the edge set $\cC\{E_{t-1}(z),\cP\}$.

\item Construct the rejected edge set 
\[
\cR(z)= \left[e\in\cC\{E_{t-1}(z),\cP\} \mid \sqrt{nh} \cdot |\hat{\bTheta}_e^\mathrm{de} (z)| \cdot \sum_{i\in [n]} K_h(Z_i-z)/n > c \{1-\alpha,\cC(E_{t-1},\cP)\} \right].
\]

\item Update the rejected edge set $E_t (z) \gets E_{t-1}(z) \cup \cR(z)$ for $z\in [0,1]$.
\item $t \gets t+1$.
\end{enumerate}
\textbf{Until:} There exists a $z_0 \in [0,1]$ such that $E_t(z_0)\in\mathscr{P}$, or $E_t(z)=E_{t-1}(z)$ for $z\in [0,1]$.

\textbf{Output:} $\psi_{\alpha}=1$ if there exists a $z_0\in [0,1]$ such that $E_t(z_0) \in \mathscr{P}$ and $\psi_{\alpha}=0$ otherwise.
\end{algorithm}

Finally, we generalize the theoretical results in Theorems~\ref{theorem:algorithm} and~\ref{theorem:algorithm-power} to the general testing procedure in Algorithm~\ref{al:stepdown2}. 
Given a monotone graph property $\cP$, let 
\[
\cG_0  = ( \bTheta (\cdot) \in \cU_{s,M}  \mid \cP[G\{\bTheta (z)\} ]= 0 \mathrm{\; for \; all\;} z\in [0,1]).
\]
We now show that  the type I error of the proposed inferential method in  Algorithm~\ref{al:stepdown2}
can be controlled at a pre-specified level $\alpha$.
\begin{theorem}
\label{theorem:algorithm2}
Under the same conditions in Theorem~\ref{theorem:gaussian multiplier bootstrap}, we have 
\[
\underset{n\rightarrow \infty}{\lim} \underset{\bTheta (\cdot) \in \cG_0}{\sup} P_{\bTheta (\cdot)} \left( \psi_{\alpha}=1 \right) \le \alpha.
\]
\end{theorem}


In order to study the power analysis for testing graph structure that satisfies the monotone graph property, we define  signal strength of a precision matrix $\bTheta$ as
\begin{equation}\label{eq:strength}
   \text{Sig}(\bTheta):= \max_{{E' \subseteq E(\bTheta),  \cP(E') =1}} \min_{e \in E'} |\bTheta_e|. 
\end{equation}
Under  $H_1:  \mathrm{there\; exists \;a\;} z_0\in [0,1] \mathrm{\;such\; that \; }\cP\{G(z_0)\}=1$, we define the parameter space 
\begin{equation}\label{eq:G1}
   \cG_1(\theta; \cP) = \Big( \bTheta(\cdot) \in\cU_{s,M} \,\Big|\,  \cP[G\{\bTheta(z_0)\}] = 1 \text{ and } \text{Sig}\{\bTheta(z_0)\}\ge \theta\text{ for some $z_0 \in [0,1]$}\Big).
\end{equation}
Again, we emphasize that the signal strength defined in \eqref{eq:strength} is weaker than the typical minimal signal strength for testing a single edge in a graph $\min_{e \in E(\bTheta)} |\bTheta_e|$. $\text{Sig}(\bTheta)$ only requires that there exists a subgraph satisfying the property of interest such that the minimal signal strength on that subgraph is above certain level. 
For example, for $\cP(G) = 1$ if and only if $G$ is connected,  it suffices for $\bTheta$ belongs to $\cG_1(\theta; \cP)$ if the minimal signal strength on a spanning tree is larger than $\theta$. The following theorem presents the power analysis of our test.

\begin{theorem}
\label{theorem:algorithm-power2}
Assume that the same conditions in Theorem~\ref{theorem:gaussian multiplier bootstrap} hold and select the smoothing parameter $h = o(1/n^{-1/5})$. Assume that $\theta \ge C \sqrt{\log (dn)}/n^{2/5}$ for some sufficiently large constant $C$.  Under the alternative hypothesis $H_1: \cP(G) = 1$ in \eqref{Eq:hypothesis}, we have
  \begin{equation}\label{eq:power}
 \lim_{n \rightarrow \infty}   \inf_{\bTheta \in \cG_1(\theta; \cP)}\PP_{\bTheta}(\psi_{\alpha} = 1) = 1
\end{equation}
for any fixed $\alpha \in (0,1)$.
\end{theorem}
Thus, we have shown in Theorem~\ref{theorem:algorithm-power2} that the power of the proposed inferential method increases to one asymptotically.

\section{A $U$-Statistic Type Estimator}
\label{appendix:general}
The main manuscript primarily concerns the case when there are two subjects.  In this section, we present a $U$-statistic type inter-subject covariance to accommodate the case when there are more than two subjects.
First, we note that the same natural stimuli is given to all subjects. This motivates the following statistical model for each $Z=z$:
\[
\bX^{(\ell)} = \boldsymbol{S} + \bE^{(\ell)},~~\boldsymbol{S} | Z =z \sim N_d\{\mathbf{0},\bSigma(z)\},~~\bE^{(\ell)}| Z=z \sim N_d\{\mathbf{0},\bL^{(\ell)}(z)\},
\]
where $\bX^{(\ell)}$, $\bE^{(\ell)}$, and $\bL^{(\ell)}(z)$ are the data, subject specific effect, and the covariance matrix for the subject specific effect for the $\ell$th subject, respectively.
Suppose that there $N$ subjects.  Then, the following $U$-statistic type inter-subject covariance matrix can be constructed to estimate $\bSigma(z)$:
\begin{equation}
\label{Eq:estimator}
\hat{\mathbf{\Sigma}}_U (z) =  \frac{1}{{N\choose 2}}\sum_{1\le \ell<\ell'\le N}\left[ \frac{\sum_{i\in [n]}  K_h (Z_i-z) \bX_i^{(\ell)} \{\bX_i^{(\ell')}\}^T}{\sum_{i\in [n]} K_h (Z_i-z)}\right].
\end{equation}
We leave the theoretical analysis of the above estimator for future work.

\section{Preliminaries}
\label{appendix:notation}
\noindent 
In this section, we define some notation that will be used throughout the Appendix. 
Let $[n]$ denote the set $\{1,\ldots,n \}$ and let $[d]$ denote the set $\{1,\ldots, d\}$.  For two scalars $a,b$, we define $a\vee b = \max(a,b)$.  
We denote the $\ell_q$-norm for the vector $\mathbf{v}$ as $\|\mathbf{v}\|_q = (\sum_{j\in [d]} |v_j|^q)^{1/q}$ for $1\le q < \infty$.  In addition, we let $\mathrm{supp}(\mathbf{v}) = \{j: v_j \ne 0\}$, $\|\mathbf{v}\|_0 = |\mathrm{supp}(\mathbf{v})|$,  and $\|\mathbf{v}\|_{\infty} ={\max}_{j\in [d]}\; |v_j|$, where $|\mathrm{supp} (\mathbf{v})|$ is the number of non-zero elements in $\mathbf{v}$.  
For a matrix $\Ab\in \mathbb{R}^{n_1\times n_2}$, we denote the $j$th column as $\Ab_j$.
We denote the Frobenius norm of $\Ab$ by $\|\Ab\|_F^2 = \sum_{i\in [n_1]} \sum_{j\in [n_2]} A_{ij}^2$, the max norm $\|\Ab\|_{\max} = \max_{i\in [n_1], j\in [n_2]}  |A_{ij}|$, and the operator norm $\|\Ab\|_2  = \sup_{\|\vb\|_2 =1}\; \|\Ab \vb \|_2$.
Given a function $f$, let $\dot{f}$ and $\ddot{f}$ be the first and second-order derivatives, respectively.  For $1\le p <\infty$, let $\|f\|_p = (\int f^{p})^{1/p}$ denote the $L_p$ norm of $f$ and let $\|f\|_{\infty} = \sup_x |f(x)|$.
The total variation of $f$ is defined as $\|f\|_{\mathrm{TV}} = \int |\dot{f}|$.
We use the Landau symbol $a_n= \mathcal{O} (b_n)$ to indicate the existence of a constant $C>0$ such that $a_n \le C \cdot b_n$ for two sequences $a_n$ and $b_n$. We write $a_n = o(b_n)$ if ${\lim}_{n\rightarrow \infty}\; a_n/b_n \rightarrow 0$. 
Let $C, C_1, C_2, \ldots$ be generic constants whose values may vary from line to line.

Let 
\begin{equation}
\label{EqA:empirical process}
\PP_n (f)  = \frac{1}{n} \sum_{i\in [n]} f(X_i) \qquad \mathrm{and} \qquad \GG_n (f) = \sqrt{n}\cdot  [ \PP_n (f) - \EE\{f(X_i)\}    ].
\end{equation}
For notational convenience, for fixed $j,k\in [d]$, let 
\begin{equation}
\label{EqA:gjk}
g_{z,jk} (Z_i,X_{ij},Y_{ik}) = K_h (Z_i-z) X_{ij}Y_{ik}, \qquad  \qquad w_z (Z_i) = K_h (Z_i-z),
\end{equation}
\begin{equation}
\label{EqA:qzjk}
q_{z,jk} (Z_i,X_{ij},Y_{ik}) = g_{z,jk}(Z_i,X_{ij},Y_{ik}) - \EE\{g_{z,jk}(Z,X_{j},Y_{k})\},
\end{equation}
and let 
\begin{equation}
\label{EqA:kz}
k_z (Z_i) = w_z(Z_i) - \EE\{w_z(Z)\}.
\end{equation}
Recall from \ref{section:theory} that $K(\cdot)$ can be any symmetric kernel function that satisfies \eqref{prop:kernel} and that $K_h (Z_i -z) = K \{ (Z_i-z)/h\}/h$. 
By the definition of $\hat{\bSigma} (z)$ in  (\ref{Eq:estimator}), we have 
\begin{equation}
\label{EqA:estimator}
\hat{\bSigma}_{jk} (z) = \frac{\sum_{i\in [n]} g_{z,jk} (Z_i,X_{ij},Y_{ik})}{\sum_{i\in [n]} w_z (Z_i)} = \frac{\PP_n (g_{z,jk})}{ \PP_n (w_z)}.
\end{equation}

In addition,  let
\begin{equation}
\label{EqA:jzjk1}
J_{z,jk}^{(1)} (Z_i,\bX_{i},\bY_{i}) = \sqrt{h} \cdot \left\{\bTheta_j (z)\right\}^T  \cdot \left[ K_h (Z_i-z) \bX_{i}\bY_{i}^T -   \EE\left\{K_h (Z-z) \bX\bY^T  \right\}    \right] \cdot \bTheta_k(z),
\end{equation}
\begin{equation}
\begin{split}
\label{EqA:jzjk2}
J_{z,jk}^{(2)} (Z_i) &= \sqrt{h} \cdot \left\{\bTheta_j (z)\right\}^T  \cdot \left[ K_h (Z_i-z) -   \EE\left\{K_h (Z-z)  \right\}    \right] \cdot \bSigma (z)\cdot \bTheta_k(z),
\end{split}
\end{equation}
\begin{equation}
\label{EqA:jzjk}
\begin{split}
J_{z,jk} (Z_i,\bX_i,\bY_{i}) = J_{z,jk}^{(1)} (Z_i,\bX_{i},\bY_{i})  -J_{z,jk}^{(2)} (Z_i),
\end{split}
\end{equation}
and let
\begin{equation}
\label{EqA:Wzjk}
W_{z,jk} (Z_i,X_{ij},Y_{ik}) = \sqrt{h} \cdot \left\{K_h (Z_i-z)X_{ij}Y_{ik}- K_h(Z_i-z)\bSigma_{jk}(z)   \right\}.
\end{equation}

For two functions $f$ and $g$, we define its convolution as 
\begin{equation}
\label{EqA:convolution}
(f*g)(x) = \int f(x-z)g(z) dz.
\end{equation}
In our proofs, we will use the following property of the derivative of a convolution
\begin{equation}
\label{EqA:convolution derivative}
\frac{\partial}{\partial x}(f*g) =  \frac{\partial f}{\partial x}*g.
\end{equation}
Finally, our proofs use the following inequality 
\begin{equation}
\label{EqA:covering number dudley}
\int_0^{b_1} \sqrt{  \log (b_2/\epsilon)   }d\epsilon \le \sqrt{b_1} \cdot \sqrt{\int_0^{b_1} \log (b_2/\epsilon) d  \epsilon}= b_1 \cdot \sqrt{1+\log (b_2/b_1) },
\end{equation}
where the first inequality holds by an application of  Jensen's inequality.

\section{Proof of Results in \ref{subsection:estimation error}}
\label{appendix:estimation error}
In this section, we establish the uniform rate of convergence for $\hat{\bSigma} (z) $ and $\hat{\bTheta} (z)$ over $z\in [0,1]$. 
To prove Theorem~\ref{theorem:estimation error}, we first observe that
\begin{equation}
\label{appendix:estimation error1}
\underset{z\in [0,1]}{\sup} \left\|\hat{\bSigma}(z) -\bSigma (z)\right\|_{\max} \le  \underset{z\in [0,1]}{\sup} \; \underset{j,k\in [d]}{\max}\; \left| \hat{\bSigma}_{jk}(z) - \EE\left\{\hat{\bSigma}_{jk}(z)\right\} \right|   + \underset{z\in [0,1]}{\sup} \;\underset{j,k\in [d]}{\max} \; \left|\EE\left\{\hat{\bSigma}_{jk}(z)\right\}  - \bSigma_{jk}(z)\right|.
\end{equation}
The first term is known as the variance term and the second term is known as the bias term in the kernel smoothing literature (see, for instance, Chapter 2 of \citealp{pagan1999nonparametric}). 
Both the variance and bias terms involve evaluating the quantity $\EE\{\hat{\bSigma}_{jk}(z)\}$.
From~(\ref{EqA:estimator}), we see that $\hat{\bSigma}_{jk}(z)$ involves the 
quotient of two averages and it is not straightforward to evaluate its expectation.  The following lemma quantifies $\EE\{\hat{\bSigma}_{jk}(z)\}$ in terms of the expectations of its numerator and its denominator.

\begin{lemma}
\label{lemma:expected sigma}
Under the following conditions 
\begin{equation}
\label{lemma:expected sigma1}
\left| \frac{\GG_n (w_z)}{\sqrt{n} \cdot  \EE\left\{ \PP_n(w_z)\right\}}   \right| < 1 \qquad \mathrm{and} \qquad \EE \left\{ \PP_n (w_z)  \right\}\ne 0,
\end{equation}
we have
\begin{equation}
\label{lemma:expected sigma2}
\EE \left\{ \hat{\bSigma}_{jk}(z)   \right\}  = \frac{\EE  \left\{  \PP_n (g_{z,jk})  \right\}   }{\EE  \left\{  \PP_n (w_z)  \right\}  } + 
\frac{1}{n}\cO\left[\EE \Big\{ \GG_n (w_z) \cdot \GG_n (g_{z,jk})     \Big\}  + \EE \left\{     \GG_n^2 (g_{z,jk})    \right\}   \right].
\end{equation}
\end{lemma}
We note that (\ref{lemma:expected sigma2}) only holds under the two conditions in (\ref{lemma:expected sigma1}).  In the proof of Theorem~\ref{theorem:estimation error}, we will show that the two conditions in (\ref{lemma:expected sigma1}) hold for $n$ sufficiently large. 
To obtain upper bounds for the bias and variance terms in (\ref{appendix:estimation error1}), we  use the following intermediate lemmas.

\begin{lemma}
\label{lemma:bias}
Assume that $h=o(1)$. Under Assumptions~\ref{ass:marginal}-\ref{ass:covariance}, we have
 \begin{equation}
 \label{lemma:bias1}
 \underset{z\in [0,1]}{\sup} \;\underset{j,k\in [d]}{\max}\; \Big|  \EE \{\PP_n (g_{z,jk})\}   - f_Z(z) \bSigma_{jk}(z)  \Big| = \cO (h^2),
 \end{equation}
 \begin{equation}
 \label{lemma:bias2}
  \underset{z\in [0,1]}{\sup} \;\Big|  \EE \{\PP_n (w_z)\}  - f_Z(z) \Big| = \cO (h^2),
 \end{equation}
  \begin{equation}
 \label{lemma:bias3}
  \underset{z\in [0,1]}{\sup} \;\underset{j,k\in [d]}{\max}\;\frac{1}{n} \left| \EE \Big\{    \GG_n (g_{z,jk})\cdot   \GG_n (w_z)  \Big\} \right|= \cO\left(\frac{1}{nh}\right),
 \end{equation}
 and 
  \begin{equation}
 \label{lemma:bias4}
  \underset{z\in [0,1]}{\sup} \; \frac{1}{n}\EE \Big\{ \GG_n^2 (w_z) \Big\}= \cO\left(\frac{1}{nh}\right).
 \end{equation}
  \end{lemma}

\begin{lemma}
\label{lemma:variance}
Assume that $h=o(1)$ and $\log^2 (d/h)/(nh) = o(1)$. Under Assumptions~\ref{ass:marginal}-\ref{ass:covariance}, there exists a universal constant $C>0$ such that 
 \begin{equation}
 \label{lemma:variance1}
\underset{z\in [0,1]}{\sup}\; \underset{j,k\in [d]}{\max} \;  \Big| \GG_n(w_z) \vee \GG_n (g_{z,jk}) \Big|  \le C \cdot  \sqrt{\frac{\log (d/h)}{h}},
 \end{equation}
 with probability at least $1-3/d$.
\end{lemma}
The proofs of Lemmas~\ref{lemma:expected sigma}-\ref{lemma:variance} are deferred to Sections~\ref{lemma:expected sigma}-\ref{lemma:variance}, respectively.
We  now provide a proof of Theorem~\ref{theorem:estimation error}.

\subsection{Proof of Theorem~\ref{theorem:estimation error}}
\label{appendix:theorem:estimation error}
Recall from (\ref{appendix:estimation error1}) that 
\begin{equation*}
\begin{split}
\underset{z\in [0,1]}{\sup} \left\|\hat{\bSigma}(z) -\bSigma (z)\right\|_{\max} &\le  \underset{z\in [0,1]}{\sup} \; \underset{j,k\in [d]}{\max}\; \left| \hat{\bSigma}_{jk}(z) - \EE\left\{\hat{\bSigma}_{jk}(z)\right\} \right|   + \underset{z\in [0,1]}{\sup} \;\underset{j,k\in [d]}{\max} \; \left|\EE\left\{\hat{\bSigma}_{jk}(z)\right\}  - \bSigma_{jk}(z)\right|\\
&= I_1 + I_2.
\end{split}
\end{equation*}
It suffices to obtain upper bounds for $I_1$ and $I_2$. 

We first verify that the two conditions in (\ref{lemma:expected sigma1}) hold. By Lemma~\ref{lemma:bias}, we have 
\[
\Big| \EE \{\PP_n (w_z) \} \Big|= \cO(h^2) + f_Z(z)\ge \underline{f}_Z (z)  > 0,
\]
where the last inequality follows from Assumption~\ref{ass:marginal}.  
Moreover, 
\begin{equation*}
\begin{split}
\left| \frac{\GG_n (w_z)}{\sqrt{n} \cdot  \EE\left\{ \PP_n(w_z)\right\}}   \right|   
&\le C\cdot   \frac{1}{\sqrt{n}} |\GG_n (w_z)| \cdot \frac{1}{f_Z(z)+\cO(h^2)}\\
&\le C_1 \cdot \sqrt{\frac{\log (d/h) }{nh}} \cdot \frac{1}{f_Z(z)+\cO(h^2)}\\
&< 1,
\end{split}
\end{equation*}
for sufficiently large $n$, where the first inequality is obtained by an application of Lemma~{\ref{lemma:bias}}, the second inequality is obtained by an application of Lemma~\ref{lemma:variance}, and the last inequality is obtained by the scaling assumptions $h= o(1)$ and $\log(d/h)/(nh)=o(1)$.\\

\textbf{Upper bound for $I_1$:} By (\ref{proof:lemma:expected sigma2}) in the proof of Lemma~\ref{lemma:expected sigma}, we have 
\[
\hat{\bSigma}_{jk}(z)=  \frac{  \GG_n (g_{z,jk}) }{\sqrt{n}    \EE \left\{   \PP_n (w_z)  \right\}    }+\frac{ \EE \{\PP_n (g_{z,jk})\} }{    \EE \left\{   \PP_n (w_z)  \right\}    } - 
\frac{  \GG_n (w_z)    \EE\{\PP_n (g_{z,jk})\}   }{\sqrt{n}   \EE^2 \{\PP_n (w_z)\}  }+ \frac{1}{n} \cO \left[   \Big\{ \GG_n (w_z)  \GG_n (g_{j,zk})    \Big\}  + \GG_n^2 (g_{z,jk})     \right].
\]
Thus, by Lemma~\ref{lemma:expected sigma}, we have 
\begin{equation}
\label{Eq:upper bound for I1}
\begin{split}
I_1 &=  \underset{z\in [0,1]}{\sup} \;\underset{j,k\in [d]}{\max}\; \left| 
\underbrace{\frac{  \GG_n (g_{z,jk}) }{\sqrt{n}\cdot    \EE \left\{   \PP_n (w_z)  \right\}    }}_{I_{11}} - 
\underbrace{\frac{  \GG_n (w_z)  \cdot  \EE\{\PP_n (g_{z,jk})\}   }{\sqrt{n} \cdot  \EE^2 \{\PP_n (w_z)\} }}_{I_{12}}
+  I_{13}\right|\\ 
&\le  \underset{z\in [0,1]}{\sup} \;\underset{j,k\in [d]}{\max}\;   \left\{ |I_{11}|+|I_{12}|+|I_{13}|\right\},
\end{split}
\end{equation}
where $I_{13} =  \cO [   \{ \GG_n (w_z)  \GG_n (g_{j,zk})    \}  + \GG_n^2 (g_{z,jk})    + \EE \{ \GG_n (w_z)\cdot  \GG_n (g_{j,zk})\}  +\EE \{\GG_n^2 (g_{z,jk})\}      ]/n$.

We now provide upper bounds for $I_{11}$, $I_{12}$, and $I_{13}$.
By an application of Lemmas~\ref{lemma:bias} and \ref{lemma:variance}, we obtain
\begin{equation}
\label{Eq:upper bound for I11}
 \underset{z\in [0,1]}{\sup} \;\underset{j,k\in [d]}{\max}\; |I_{11}| \le 
n^{-1/2} \cdot  \underset{z\in [0,1]}{\sup} \;\underset{j,k\in [d]}{\max}\; \left|\frac{\GG_n(g_{z,jk})}{f_Z(z) + \cO(h^2)} \right| \le C\cdot \sqrt{\frac{\log(d/h)}{nh}}.
\end{equation}
Similarly, we have 
\begin{equation}
\label{Eq:upper bound for I12}
 \underset{z\in [0,1]}{\sup} \;\underset{j,k\in [d]}{\max}\; |I_{12}| \le 
n^{-1/2} \cdot  \underset{z\in [0,1]}{\sup} \;\underset{j,k\in [d]}{\max}\; \left|\frac{\GG_n(g_{z,jk})   \{f_Z(z)\bSigma_{jk}(z) + \cO(h^2)\}  }{\{f_Z(z) + \cO(h^2)\}^2} \right| \le C\cdot \sqrt{\frac{\log(d/h)}{nh}}.
\end{equation}
For $I_{13}$, we have 
\begin{equation}
\begin{split}
\label{Eq:upper bound for I13}
 \underset{z\in [0,1]}{\sup} \;\underset{j,k\in [d]}{\max}\; |I_{13}| &\le   \underset{z\in [0,1]}{\sup} \;\underset{j,k\in [d]}{\max}\;  \left| \frac{1}{n} \cO \left[   \Big\{ \GG_n (w_z)\cdot  \GG_n (g_{z,jk})    \Big\}   + \GG_n^2 (g_{z,jk})   \right]\right|  + \cO \left( \frac{1}{nh}\right)\\
 &\le C\cdot  {\frac{\log (d/h)}{nh}} + \cO \left(\frac{1}{nh}\right)\\
 &\le C \cdot {\frac{\log (d/h)}{nh}},
 \end{split}
 \end{equation}
where the first and second inequalities follow from Lemmas~\ref{lemma:bias} and~\ref{lemma:variance}, respectively.
Combining (\ref{Eq:upper bound for I11}), (\ref{Eq:upper bound for I12}), and (\ref{Eq:upper bound for I13}), we have 
\begin{equation}
\label{Eq:upper bound for I1combine}
I_1 \le C\cdot \sqrt{\frac{\log (d/h)}{nh}},
\end{equation}
with probability at least $1-3/d$.\\

\textbf{Upper bound for $I_2$:} By Lemmas~\ref{lemma:expected sigma} and~\ref{lemma:bias}, we have 
\begin{equation}
\label{Eq:upper bound for I2}
\begin{split}
I_2 &=   \underset{z\in [0,1]}{\sup} \;\underset{j,k\in [d]}{\max}\; \left| 
\frac{\EE  \left\{  \PP_n (g_{z,jk})  \right\}   }{\EE  \left\{  \PP_n (w_z)  \right\}  } - \bSigma_{jk}(z) +\frac{1}{n}\cO\left[\EE \Big\{ \GG_n (w_z) \cdot \GG_n (g_{z,jk})     \Big\}  + \EE \left\{     \GG_n^2 (g_{z,jk})    \right\}   \right]
\right|\\
&\le \underset{z\in [0,1]}{\sup} \;\underset{j,k\in [d]}{\max}\;  \left| 
\frac{f_Z (z) \bSigma_{jk}(z) + \cO(h^2)  }{f_Z(z) + \cO(h^2)  } - \bSigma_{jk}(z) +\frac{1}{n}\cO\left[\EE \Big\{ \GG_n (w_z) \cdot \GG_n (g_{z,jk})     \Big\}  + \EE \left\{     \GG_n^2 (g_{z,jk})    \right\}   \right]
\right|\\
&=  \underset{z\in [0,1]}{\sup} \;\underset{j,k\in [d]}{\max}\;  \left| 
\frac{f_Z (z) \bSigma_{jk}(z) + \cO(h^2)  }{f_Z(z) + \cO(h^2)  } - \bSigma_{jk}(z) +\cO \left( \frac{1}{nh}\right)
\right|\\
&=  \underset{z\in [0,1]}{\sup} \;\underset{j,k\in [d]}{\max}\;  \left| 
\frac{\cO(h^2) \bSigma_{jk}(z)   }{f_Z(z) + \cO(h^2)  } +\cO \left( \frac{1}{nh}\right)
\right|\\
&\le C\cdot  \left( h^2 + \frac{1}{nh}\right),
\end{split}
\end{equation}
where the first inequality follows from (\ref{lemma:bias1}) and (\ref{lemma:bias2}), the second equality follows from (\ref{lemma:bias3}) and (\ref{lemma:bias4}), and the last inequality follows from the assumption that $h=o(1)$.

Combining the upper bounds (\ref{Eq:upper bound for I1combine}) and (\ref{Eq:upper bound for I2}), we obtain 
\[
\underset{z\in [0,1]}{\sup} \left\|\hat{\bSigma}(z) -\bSigma (z)\right\|_{\max}  \le C \cdot \left\{  h^2 + \sqrt{\frac{\log (d/h)}{nh}}   \right\}
\]
with probability at least $1-3/d$.

\section{Proof of Technical Lemmas in Appendix~\ref{appendix:estimation error}}
\label{appendix:technical1}
In this section, we provide the proofs of Lemmas~\ref{lemma:expected sigma}-\ref{lemma:variance}.

\subsection{Proof of Lemma~\ref{lemma:expected sigma}}
\label{proof:lemma:expected sigma}
The proof of the lemma uses the following fact
\begin{equation}
\label{proof:lemma:expected sigma1}
(1+x)^{-1} = 1-x + \cO(x^2) \qquad \mathrm{for \; any\;} |x|<1.
\end{equation}
From (\ref{EqA:estimator}), we have
\begin{equation*}
\begin{split}
\hat{\bSigma}_{jk} (z)  &= \frac{\PP_n (g_{z,jk})}{ \PP_n (w_z)}\\
&= \frac{ \PP_n (g_{z,jk}) -\EE \left\{   \PP_n (g_{z,jk}) \right\} +\EE \left\{   \PP_n (g_{z,jk})  \right\}}{\EE \left\{   \PP_n (w_z)  \right\}}   \cdot \left[ \frac{\EE \left\{   \PP_n (w_z)  \right\}}{\PP_n (w_z)}  \right]\\
&= \frac{ n^{-1/2}\cdot \GG_n (g_{z,jk})  +\EE \left\{\PP_n (g_{z,jk})  \right\}}{\EE \left\{   \PP_n (w_z)  \right\}}   \cdot \left[ 1+ \frac{\PP_n (w_z)- \EE \left[ \PP_n (w_z)\right]}{\EE \left[   \PP_n (w_z)  \right]}  \right]^{-1}\\
&= \frac{ n^{-1/2}\cdot \GG_n (g_{z,jk})  +\EE \left[   \PP_n (g_{z,jk})  \right]}{\EE \left\{   \PP_n (w_z)  \right\}}   \cdot \left[ 1+ \frac{\GG_n (w_z)}{\sqrt{n} \cdot \EE \left\{   \PP_n (w_z)  \right\}}  \right]^{-1}.
\end{split}
\end{equation*}
Under the conditions (\ref{lemma:expected sigma1}) and by applying~(\ref{proof:lemma:expected sigma1}), we have 
\begin{equation}
\label{proof:lemma:expected sigma2}
\begin{split}
\hat{\bSigma}_{jk} (z)  &= \frac{ n^{-1/2}\cdot \GG_n (g_{z,jk})  +\EE \left\{   \PP_n (g_{z,jk})  \right\}}{\EE \left\{   \PP_n (w_z)  \right\}}   \cdot \left( 1- \frac{\GG_n (w_z)}{\sqrt{n} \cdot \EE \left\{   \PP_n (w_z)  \right\}} + \cO \left[  \frac{\GG_n^2 (w_z)}{n \cdot \EE^2 \left\{   \PP_n (w_z)  \right\}}\right]  \right)\\
&=  \frac{  \GG_n (g_{z,jk}) }{\sqrt{n}    \EE \left\{   \PP_n (w_z)  \right\}    }+\frac{ \EE \{\PP_n (g_{z,jk})\} }{    \EE \left\{   \PP_n (w_z)  \right\}    } - 
\frac{  \GG_n (w_z)    \EE\{\PP_n (g_{z,jk})\}   }{\sqrt{n}   \EE^2 \{\PP_n (w_z)\}  }+ \frac{1}{n} \cO \left[   \Big\{ \GG_n (w_z)  \GG_n (g_{z,jk})    \Big\}   + \GG_n^2 (g_{z,jk})     \right].
\end{split}
\end{equation}

Note that $\EE\{\GG_n (f)\} = 0 $ by the definition of $\GG_n (f)$ in (\ref{EqA:empirical process}).
Taking expectation on both sides of~(\ref{proof:lemma:expected sigma2}), we obtain
\[
\EE \left\{ \hat{\bSigma}_{jk}(z)   \right\}  = \frac{\EE  \left\{  \PP_n (g_{z,jk})  \right\}   }{\EE  \left\{  \PP_n (w_z)  \right\}  } + 
\frac{1}{n}\cO\left[\EE \Big\{ \GG_n (w_z) \cdot \GG_n (g_{z,jk})     \Big\}  + \EE \left\{     \GG_n^2 (g_{z,jk})    \right\}   \right],
\]
as desired.

\subsection{Proof of Lemma~\ref{lemma:bias}}
\label{proof:lemma:bias}
To prove Lemma~\ref{lemma:bias}, we write the expectation as an integral and apply Taylor expansion to the density function and the covariance function.  We will show that the higher-order terms of the Taylor expansion can be bounded by $\cO(h^2)$.
 We start by proving (\ref{lemma:bias1}).

\textbf{Proof of (\ref{lemma:bias1}):}
Recall from (\ref{EqA:gjk}) the definition of $g_{z,jk}(Z_i,X_{ij},Y_{ik}) = K_h (Z_i-z) X_{ij}Y_{ik}$. Thus, we have
\begin{equation}
\label{proof:lemma:bias1}
\begin{split}
\EE \{\PP_n (g_{z,jk})\} &= \EE \left\{  \frac{1}{h} K\left( \frac{Z-z}{h}\right) X_j Y_k  \right\}\\
&=\EE \left\{  \frac{1}{h} K\left( \frac{Z-z}{h}\right) \EE(X_j Y_k\mid Z)  \right\}\\
&=\EE \left\{  \frac{1}{h} K\left( \frac{Z-z}{h}\right) \EE(S_j S_k\mid Z)  \right\}\\
&=\EE \left\{  \frac{1}{h} K\left( \frac{Z-z}{h}\right) \bSigma_{jk}(Z) \right\}\\
&= \int  \frac{1}{h} K\left( \frac{Z-z}{h}\right) \bSigma_{jk}(Z)f_Z (Z) dZ\\
&= \int   K(u) \bSigma_{jk}(uh+z)f_Z (uh+z) du,
\end{split}
\end{equation}
where the third equality hold using the fact that the subject-specific effects are independent between two subjects, and the last equality holds by a change of variable, $u = (Z-z)/h$.  Applying Taylor expansions to $\bSigma_{jk}(uh+z)$ and $f_Z(uh+z)$, we have 
 \begin{equation}
 \label{proof:lemma:bias2}
\bSigma_{jk} (u+zh) =\bSigma_{jk} (z)+uh \cdot \dot{\bSigma}_{jk} (z)   + u^2h^2 \cdot \ddot{\bSigma}_{jk}(z')
\end{equation} 
and 
 \begin{equation}
 \label{proof:lemma:bias3}
f_Z (u+zh) =f_Z(z)+uh \cdot \dot{f}_Z (z)   + u^2h^2 \cdot \ddot{f}_Z(z''),
\end{equation} 
 where $z'$ and $z''$ are between $z$ and $uh+z$.
Substituting (\ref{proof:lemma:bias2}) and (\ref{proof:lemma:bias3}) into the last expression of (\ref{proof:lemma:bias1}), we have 
  \begin{equation}
 \label{proof:lemma:bias4}
  \begin{split}
 \int   K(u) \left\{\bSigma_{jk} (z)+uh \cdot \dot{\bSigma}_{jk} (z)   + u^2h^2 \cdot \ddot{\bSigma}_{jk}(z')\right\} \cdot \left\{f_Z(z)+uh \cdot \dot{f}_Z (z)   + u^2h^2 \cdot \ddot{f}_Z(z'')\right\} du.
 \end{split}
 \end{equation} 
 By \eqref{prop:kernel}, we have $\int u K(u)du =0$ and $\int u^l K(u) du < \infty$ for $l=1,2,3,4$.
By Assumptions~\ref{ass:marginal}  and~\ref{ass:covariance}, we have
\begin{equation}
\label{proof:lemma:bias5}
\begin{split}
&h^2 \int u^2 K(u) \ddot{\bSigma}_{jk} (z') f_Z (z) du \le h^2 C M_{\sigma} \bar{f}_Z = \cO(h^2),\\
&h^2 \int u^2 K(u) \dot{\bSigma}_{jk} (z) \dot{f}_Z (z) du \le h^2 C M_{\sigma} \bar{f}_Z = \cO(h^2),\\
&h^2 \int u^2 K(u) \bSigma_{jk} (z) \ddot{f}_Z (z'') du \le h^2 C M_{\sigma} \bar{f}_Z = \cO(h^2).
\end{split}
\end{equation}
Substituting (\ref{proof:lemma:bias5}) into (\ref{proof:lemma:bias4}) and 
 bounding the other higher-order terms by $\cO(h^2)$, we obtain 
\[
\EE\{\PP_n (g_{z,jk})\} = \bSigma_{jk} (z) f_Z(z) + \cO(h^2),
\]
for all $z\in [0,1]$ and $j,k\in [d]$.  This implies that
\[
\underset{z\in [0,1]}{\sup} \; \underset{j,k\in[d]}{\max}\; | \EE\{\PP_n (g_{z,jk})\} - \bSigma_{jk} (z) f_Z(z) | = \cO(h^2).
\]
The proof of (\ref{lemma:bias2}) follows from the  same set of argument.\\

\textbf{Proof of (\ref{lemma:bias3}):} Recall from (\ref{EqA:gjk}) the definition of $w_z (Z_i)= K_h (Z_i-z)$.  Thus, we have
\begin{equation}
\label{proof:lemma:bias6}
\begin{split}
&\frac{1}{n} \EE \Big\{  \GG_n (g_{z,jk})\cdot  \GG_n(w_z)  \Big\}\\
&= \EE \Big\{ \PP_n (g_{z,jk})\cdot  \PP_n(w_z) \Big\} - \EE \{\PP_n (g_{z,jk})\}\cdot \EE \{\PP_n (w_z)\}\\
&= \EE \left[\left\{\frac{1}{n}\sum_{i\in [n]}K_h (Z_i-z) X_{ij}Y_{ik}  \right\}    \cdot \left\{ \frac{1}{n}\sum_{i\in [n]} K_h(Z_i-z) \right\} \right] - \EE \{\PP_n (g_{z,jk})\}\cdot \EE \{\PP_n (w_z)\}\\
&= \frac{1}{n} \EE \left\{  K^2_h(Z-z)  S_{j}S_k\right\} + \frac{1}{n^2} 
\EE\left\{   \sum_{i\in [n] }\sum_{i' \ne i} K_h (Z_i-z) K_h(Z_{i'}-z) X_{ij}Y_{ik}  \right\} - \EE \{\PP_n (g_{z,jk})\}\cdot \EE \{\PP_n (w_z)\}\\
&= \frac{1}{n} \EE \left\{  K^2_h(Z-z)  \bSigma_{jk}(Z) \right\} + \frac{n-1}{n} 
\left[ \EE\left\{ K_h(Z-z)\right\}   \cdot \EE \left\{ K_h (Z-z) \bSigma_{jk}(Z)   \right\}\right] -\EE \{\PP_n (g_{z,jk})\} \EE \{\PP_n (w_z)\}\\
&= \underbrace{\frac{1}{n} \EE \left\{  K^2_h(Z-z)  \bSigma_{jk}(Z) \right\}}_{I_1}  -\underbrace{\frac{1}{n}\EE \{\PP_n (g_{z,jk})\} \EE \{\PP_n (w_z)\}}_{I_2},
\end{split}
\end{equation}
where the second to the last equality follows from the fact that $Z_i$ and $Z_{i'}$ are independent. 

We now obtain an upper bound for $I_1$. By \eqref{prop:kernel} and Assumptions~\ref{ass:marginal}-\ref{ass:covariance}, we have 
\begin{equation}
\label{proof:lemma:bias7}
I_1 = \frac{1}{nh}  \int \frac{1}{h} K^2\left(\frac{Z-z}{h}\right) \bSigma_{jk} (Z) f_Z(Z) dZ 
\le \frac{1}{nh} \cdot M_{\sigma} \cdot \bar{f}_Z \int \frac{1}{h} K^2\left(\frac{Z-z}{h}\right)dZ
= \cO \left( \frac{1}{nh}\right),
\end{equation}
where the last equality holds by a change of variable. 
Moreover, by (\ref{lemma:bias1}) and (\ref{lemma:bias2}), we have 
\begin{equation}
\label{proof:lemma:bias8}
I_2 = \frac{1}{n} \left\{f_Z(z)\bSigma_{jk}(z) + \cO(h^2)\right\} \cdot \{f_Z(z)+\cO(h^2)\} = \cO \left( \frac{1}{n} \right).
\end{equation}
Substituting  (\ref{proof:lemma:bias7}) and  (\ref{proof:lemma:bias8}) into  (\ref{proof:lemma:bias6}), and taking the supreme over $z\in [0,1]$ and $j,k\in [d]$ on both sides of the equation, we obtain
\[
\underset{z\in [0,1]}{\sup} \;\underset{j,k\in [d]}{\max} \; \left| \frac{1}{n} \EE \Big\{  \GG_n (g_{z,jk})\cdot  \GG_n(w_z)  \Big\} \right| = \cO\left(\frac{1}{nh}\right) + \cO \left( \frac{1}{n} \right) =  \cO\left(\frac{1}{nh}\right), 
\]
where the last equality holds by the scaling assumption of $h= o(1)$.
The proof of (\ref{lemma:bias4}) follows from the  same set of argument.

\subsection{Proof of Lemma~\ref{lemma:variance}}
\label{proof:lemma:variance}
The proof of Lemma~\ref{lemma:variance} involves obtaining upper bounds for the supreme of the empirical processes  
$\GG_n(w_z)$ and $\GG_n (g_{z,jk})$.    
To this end, we apply the Talagrand's inequality in Lemma~\ref{lemma:ep1}. 
Let $\cF$ be a function class.  
In order to apply Talagrand's inequality, we need to evaluate the quantities $\eta$ and $\tau^2$ such that 
\[
\underset{f\in\cF}{\sup} \; \|f\|_{\infty} \le \eta \qquad \mathrm{and}\qquad \underset{f\in\cF}{\sup} \; \mathrm{Var} (f(X)) \le \tau^2.
\]
Talagrand's inequality in Lemma~\ref{lemma:ep1} provides an upper bound   for the supreme of an empirical process in terms of its expectation. 
By Lemma~\ref{lemma:ep2}, the expectation can then be upper bounded as a function of the covering number of the function class $\cF$, denoted as $N\{\cF,L_2(Q),\epsilon\}$.  
The following lemmas provide upper bounds for the supreme of the empirical processes $\GG_n(w_z)$ and $\GG_n (g_{z,jk})$, respectively.
The proofs are deferred to Sections~\ref{proof:lemma:empirical process w} and~\ref{proof:lemma:empirical process g}, respectively.
\begin{lemma}
\label{lemma:empirical process w}
Assume that $h= o(1)$ and $\log (d/h)/(nh) =o(1)$. Under Assumptions~\ref{ass:marginal}-\ref{ass:covariance}, for sufficiently large $n$, there exists a universal constant $C>0$ such that 
\begin{equation}
\label{lemma:empirical process w1}
\underset{z\in [0,1]}{\sup} \; \left|  \GG_n (w_z)   \right| \le C\cdot \sqrt{\frac{\log (d/h)}{h}},
\end{equation}
with probability at least $1-1/d$.
\end{lemma}

\begin{lemma}
\label{lemma:empirical process g}
Assume that $h= o(1)$ and $\log^2 (d/h)/(nh) =o(1)$. Under Assumptions~\ref{ass:marginal}-\ref{ass:covariance}, for sufficiently large $n$, there exists a universal constant $C>0$ such that 
\begin{equation}
\label{lemma:empirical process g1}
\underset{z\in [0,1]}{\sup} \;  \underset{j,k\in[d]}{\max} \; \left|  \GG_n (g_{z,jk})   \right| \le C\cdot \sqrt{\frac{\log (d/h)}{h}},
\end{equation}
with probability at least $1-2/d$.
\end{lemma}

Applying Lemmas~\ref{lemma:empirical process w} and \ref{lemma:empirical process g}, we obtain 
\begin{equation*}
\begin{split}
\underset{z\in [0,1]}{\sup}\; \underset{j,k\in [d]}{\max} \;  \Big| \GG_n(w_z) \vee \GG_n (g_{z,jk}) \Big|  
&\le  \underset{z\in [0,1]}{\sup} \; \left|  \GG_n (w_z)   \right|    +\underset{z\in [0,1]}{\sup} \;  \underset{j,k\in[d]}{\max} \; \left|  \GG_n (g_{z,jk})   \right| \\
&\le C \cdot  \sqrt{\frac{\log (d/h)}{h}},
\end{split}
\end{equation*}
with probability at least $1-3/d$, as desired.

\subsubsection{Proof of Lemma~\ref{lemma:empirical process w} }
\label{proof:lemma:empirical process w}
The proof of Lemma~\ref{lemma:empirical process w} uses the set of arguments as detailed in the beginning of \ref{proof:lemma:variance}.  
Recall from (\ref{EqA:gjk}) and (\ref{EqA:kz}) the definition of $w_z (Z_i) = K_h (Z_i-z)$ and $k_z (Z_i) = w_z(Z_i) - \EE\{w_z(Z)\}$, respectively.
We consider the class of function 
\begin{equation}
\label{proof:lemma:empirical process w1}
\cK = \left\{ k_z \mid  z\in [0,1]  \right\}.
\end{equation}

First, note that
\begin{equation}
\label{proof:lemma:empirical process w2}
\begin{split}
\underset{z\in [0,1]}{\sup}\; \|k_z\|_{\infty} 
&= \underset{z\in [0,1]}{\sup}\;  \|   w_z(Z_i) - \EE\{w_z(Z)\}   \|_{\infty}\\
&\le \frac{1}{h} \|K\|_{\infty} + \bar{f}_Z + \cO(h^2)\\
&\le \frac{2}{h} \|K\|_{\infty},
\end{split}
\end{equation}
where  the first inequality holds by \eqref{prop:kernel} and Lemma~\ref{lemma:bias}, and the last inequality holds by the scaling assumption $h=o(1)$ for sufficiently large $n$.

Next, we obtain an upper bound for the variance of $k_z(Z_i)$.   Note that 
\begin{equation*}
\begin{split}
\underset{z\in[0,1]}{\sup}\; \mathrm{Var}\{k_z(Z)\} &=\underset{z\in[0,1]}{\sup}\; \EE \left(  \left[ w_z(Z) - \EE\{w_z(Z)\}    \right]^2    \right)\\
&\le \underbrace{\underset{z\in[0,1]}{\sup}\;  2 \EE\{w_z^2(Z)\}}_{I_1}      + \underbrace{\underset{z\in[0,1]}{\sup}\; 2 \EE^2\{w_z(Z)\}}_{I_2},
\end{split}
\end{equation*}
where we apply the inequality $(x-y)^2 \le 2x^2+2y^2$ for two scalars $x,y$.
By Lemma~\ref{lemma:bias}, we have $I_2 \le 2\{\bar{f}_Z +\cO(h^2)\}^2$.  Also, by a change of variable and second-order Taylor expansion on the marginal density $f_Z (\cdot)$, we have 
\begin{equation}
\label{proof:lemma:empirical process w2-5}
\begin{split}
I_1 &= 2\underset{z\in[0,1]}{\sup}\; \int \frac{1}{h^2} K^2\left( \frac{Z-z}{h}\right) f_Z(Z) dZ\\
 &= 2\underset{z\in[0,1]}{\sup}\;\frac{1}{h} \int K^2(u) f_Z(uh+z) du\\
 &=2\underset{z\in[0,1]}{\sup}\;\frac{1}{h} \int K^2(u) \left\{ f_Z(z) + uh \dot{f}_Z(z) + u^2h^2 \ddot{f}_Z(z') \right\}du \quad \mathrm{for\;} z'\in (z,u+zh) \\ 
 &\le \frac{2}{h} \bar{f}_Z \|K\|_2^2 + \cO(1)+ \cO(h).
\end{split}
\end{equation}
  Thus, for sufficiently large $n$ and the assumption that $h=o(1)$, we have 
 \begin{equation}
 \label{proof:lemma:empirical process w3}
\underset{z\in[0,1]}{\sup}\; \mathrm{Var}\{k_z(Z)\}  \le \frac{3}{h} \cdot \bar{f}_Z\cdot  \|K\|_2^2.
 \end{equation}
 By Lemma~\ref{lemma:coveringnumberkz}, the covering number for the function class $\cK$ satisfies
\begin{equation}
\label{proof:lemma:empirical process w11}
\underset{Q}{\sup}\;N\{\cK,L_2(Q),\epsilon\} \le \left(\frac{4 \cdot \|K\|_{\mathrm{TV}} \cdot C_K^{4/5} \cdot \bar{f}_Z^{1/5}}{h\epsilon}\right)^{5}.
\end{equation}

We are now ready  to obtain an upper bound for the supreme of the empirical process,  $\sup_{z\in [0,1]}  \left|  \GG_n (w_z)   \right|$.
By Lemma~\ref{lemma:ep2} with $A = 2 \cdot \|K\|_{\mathrm{TV}} \cdot C_K^{4/5}\cdot  \bar{f}_Z^{1/5}/\|K\|_{\infty}$, $\|F\|_{L_2(\PP_n)} = 2\cdot \|K\|_\infty/h$, $V=5$, $\sigma^2_P = 3 \cdot \bar{f}_Z \cdot \|K\|_2^2/h$, for sufficiently large $n$,  we obtain 
\begin{equation}
\label{proof:lemma:empirical process w12}
\begin{split}
\EE \left\{ \underset{z\in [0,1]}{\sup} \; \frac{1}{\sqrt{n}} \cdot \left|\GG_n (w_z) \right|\right\} &= \EE \left( \underset{z\in [0,1]}{\sup} \;\frac{1}{n} \left|\sum_{i\in [n]} [ w_z(Z_i) - \EE\{w_z (Z)\}] \right| \right) \\
&\le  C \cdot \left\{ \sqrt{\frac{\log (1/h)}{nh}}  + \frac{\log (1/h)}{n} \right\}\\
&\le  C\cdot  \sqrt{\frac{\log (1/h)}{nh}},
\end{split}
\end{equation}
where $C>0$ is some sufficiently large constant.
By Lemma~\ref{lemma:ep1} with $\tau^2 = 3 \bar{f}_Z\cdot \|K\|_2^2/h$, $\eta= 2\cdot \|K\|_{\infty}/h$, $\EE[Y] \le C\cdot \sqrt{\log (1/h)/(nh)}$, and picking $t = \sqrt{\log (d)/n}$, for sufficiently large $n$, we have 
\begin{equation*}
\begin{split}
 \underset{z\in [0,1]}{\sup} \; \frac{1}{\sqrt{n}}\cdot \left|\GG_n (w_z) \right| &= 
  \underset{z\in [0,1]}{\sup} \; \frac{1}{n} \left|  \sum_{i\in [n]} (w_z(Z_i) - \EE\{w_z(Z)\} \right|
\\
 &\le C\cdot \left( \sqrt{\frac{\log (1/h)}{nh}} +  \sqrt{\frac{\log (d)}{nh}} \cdot \sqrt{ 1 + \sqrt{\frac{\log (1/h)}{nh}}} + \frac{\log (d)}{nh}\right)\\
 &\le C \cdot  \sqrt{\frac{\log (d/h)}{nh}},
\end{split}
\end{equation*}
with probability $1-1/d$, where the last expression holds by the assumption that $\log (d/h)/(nh) = o(1)$ and $h=o(1)$.  Multiplying both sides of the above equation by $\sqrt{n}$ completes the proof of Lemma~\ref{lemma:empirical process w}.

\subsubsection{Proof of Lemma~\ref{lemma:empirical process g} }
\label{proof:lemma:empirical process g}
The proof of Lemma~\ref{lemma:empirical process g} uses the set of arguments as detailed in the beginning of \ref{proof:lemma:variance}.  
For convenience, we prove Lemma~\ref{lemma:empirical process g} by conditioning on the event 
\begin{equation}
\label{proof:lemma:empirical process g0}
\cA = \left\{\underset{i\in [n]}{\max} \; \underset{j\in [d]}{\max} \; \max(|X_{ij}|,|Y_{ij}|) \le M_X \cdot \sqrt{\log d}  \right\}.
\end{equation}
Since $X_{ij}$ and $Y_{ij}$ conditioned on $Z$ are Gaussian random variables, the event $\cA$ occurs with probability at least $1-1/d$ for sufficiently large constant $M_X>0$.  

Recall from (\ref{EqA:gjk}) and (\ref{EqA:qzjk}) the definition of $g_{z,jk} (Z_i,X_{ij},Y_{ik}) = K_h (Z_i-z) X_{ij}Y_{ik}$ and  $q_{z,jk} (Z_i,X_{ij},Y_{ik}) = g_{z,jk}(Z_i,X_{ij},Y_{ik}) - \EE\{g_{z,jk}(Z,X_{j},Y_{k})\}$, respectively.
We consider the function class
\begin{equation}
\label{proof:lemma:empirical process g1}
\cQ = \left\{ q_{z,jk} \mid  z\in [0,1], j,k\in [d]  \right\}.
\end{equation}
We first obtain an upper bound for the function class
\begin{equation}
\label{proof:lemma:empirical process g2}
\begin{split}
\underset{z\in [0,1]}{\sup}\; \underset{j,k\in [d]}{\max}\; \|q_{z,jk}\|_{\infty} 
&= \underset{z\in [0,1]}{\sup}\; \underset{j,k\in [d]}{\max}\;  \|    g_{z,jk}(Z_i,X_{ij},Y_{ik}) - \EE\{g_{z,jk}(Z,X_{j},Y_{k})\}  \|_{\infty}\\
&\le \underset{z\in [0,1]}{\sup}\; \underset{j,k\in [d]}{\max}\;  \|    g_{z,jk}(Z_i,X_{ij},Y_{ik})\|_{\infty} +  \underset{z\in [0,1]}{\sup}\; \underset{j,k\in [d]}{\max}\;  \|\EE\{g_{z,jk}(Z,X_{j},Y_{k})\}  \|_{\infty}\\
&\le \underset{z\in [0,1]}{\sup}\; \underset{j,k\in [d]}{\max}\;  \|       
K_h (Z_i-z) X_{ij}Y_{ik}
\|_{\infty} + \bar{f}_Z \cdot M_{\sigma} + \cO(h^2)\\
&\le  \frac{1}{h} \cdot M^2_X \cdot \|K\|_{\infty} \cdot \log d       
 + \bar{f}_Z \cdot M_{\sigma} + \cO(h^2)\\
&\le \frac{2}{h}\cdot M^2_X \cdot  \|K\|_{\infty}\cdot \log d,
\end{split}
\end{equation}
where  the second inequality holds by Assumptions~\ref{ass:marginal}-\ref{ass:covariance} and Lemma~\ref{lemma:bias}, the third inequality holds by \eqref{prop:kernel} and by conditioning on the event $\cA$, and the last inequality holds by the scaling assumption $h=o(1)$ for sufficiently large $n$.

Next, we obtain an upper bound for the variance of $q_{z,jk}(Z_i,X_{ij},Y_{ik})$.   Note that 
\begin{equation*}
\begin{split}
\underset{z\in [0,1]}{\sup}\; \underset{j,k\in [d]}{\max}\;  \mathrm{Var}\{q_{z,jk}(Z,X_{j},Y_{k})\} &=\underset{z\in [0,1]}{\sup}\; \underset{j,k\in [d]}{\max}\; \EE \left[  \left(g_{z,jk}(Z,X_j,Y_k) - \EE\{g_{z,jk}(Z,X_j,Y_k)\}    \right)^2    \right]\\
&\le \underbrace{\underset{z\in [0,1]}{\sup}\; \underset{j,k\in [d]}{\max}\;   2 \EE\left\{g_{z,jk}^2(Z,X_j,Y_k)\right\}}_{I_1}      + \underbrace{\underset{z\in [0,1]}{\sup}\; \underset{j,k\in [d]}{\max}\;  2 \EE^2\{g_{z,jk}(Z,X_j,Y_k)\}}_{I_2},
\end{split}
\end{equation*}
where we apply the inequality $(x-y)^2 \le 2x^2+2y^2$ for two scalars $x,y$.
By Lemma~\ref{lemma:bias}, we have $I_2 \le 2\left\{\bar{f}_Z\cdot  M_\sigma +\cO(h^2)\right\}^2$.  Also, by a change of variable and second-order Taylor expansion on the marginal density $f_Z (\cdot)$ as in (\ref{proof:lemma:empirical process w2-5}), we have 
\begin{equation*}
\begin{split}
I_1 &= 2\underset{z\in[0,1]}{\sup}\; \underset{j,k\in[d]}{\max}\;  
\EE \left\{K_h^2 (Z-z) \cdot \EE\left(X_{j}^2Y_k^2 \mid  Z\right)    \right\}
  \\
  &\le  2 \kappa \underset{z\in[0,1]}{\sup}\; \EE \left\{K_h^2 (Z-z)   \right\}
 \\ 
  &\le \frac{2 \kappa}{h} \cdot \bar{f}_Z \cdot \|K\|_2^2 + \cO(1)+ \cO(h),
\end{split}
\end{equation*}
where the first inequality follows from the fact that $ |\EE(X_{j}^2Y_k^2 \mid Z)| \le \kappa$ for some $\kappa < \infty$ since these are Gaussian random variables, and the second inequality follows from (\ref{proof:lemma:empirical process w2-5}).
  Thus, for sufficiently large $n$ and the assumption that $h=o(1)$, we have 
 \begin{equation}
 \label{proof:lemma:empirical process g3}
\underset{z\in[0,1]}{\sup}\; \underset{j,k\in[d]}{\max}\;  \mathrm{Var}\{q_{z,j,k}(Z,X_{j},Y_{k})\}  \le \frac{3\kappa}{h} \cdot \bar{f}_Z \cdot \|K\|_2^2.
 \end{equation}
By Lemma~\ref{lemma:coveringnumbergzjk}, the covering number for the function class $\cQ$ satisfies 
\begin{equation}
\label{proof:lemma:empirical process g10}
\underset{Q}{\sup}\;N\{\cQ,L_2(Q),\epsilon\} \le \left(\frac{4 \|K\|_{\mathrm{TV}} \cdot C_K^{4/5} \cdot \bar{f}_Z^{1/5}\cdot M_\sigma^{1/5} \cdot d^{1/10} \cdot M^{2/5}_X\cdot  \log^{2/5} d}{h\epsilon}\right)^{5}.
\end{equation}

We now obtain an upper bound for the supreme of the empirical process,  $\underset{z\in [0,1]}{\sup} \; \underset{j,k\in [d]}{\max} \; \left|  \GG_n (g_{z,jk})   \right|$.
By Lemma~\ref{lemma:ep2} with $A = 2 \cdot \|K\|_{\mathrm{TV}}\cdot  C_K^{4/5} \cdot\bar{f}_Z^{1/5} \cdot M_{\sigma}^{1/5}\cdot d^{1/10}/ \|K\|_{\infty} $, $\|F\|_{L_2(\PP_n)} =2\cdot \|K\|_\infty \cdot M^2_X \cdot \log d/h$, $V=5$, $\sigma^2_P = (3\kappa /h) \cdot \bar{f}_Z \cdot \|K\|_2^2$, for sufficiently large $n$,  we obtain 
\begin{equation}
\label{proof:lemma:empirical process g11}
\begin{split}
\EE \left\{ \underset{z\in [0,1]}{\sup} \;   \underset{j,k\in [d]}{\max} \;  \frac{1}{\sqrt{n}} \cdot \left|\GG_n (g_{z,jk})\right| \right\}
&=\EE \left( \underset{z\in [0,1]}{\sup} \;   \underset{j,k\in [d]}{\max} \;  \frac{1}{n} \cdot \left|\sum_{i\in [n]} [g_{z,jk}(Z_i,X_{ij},Y_{ik}) - \EE\{g_{z,jk} (Z,X_j,Y_k)\}]\right| \right) \\
&\le C \cdot \left\{  \sqrt{\frac{\log (d/h)}{nh}} + \frac{\log (d/h)}{n}    \right\}\\
&\le  C\cdot \sqrt{\frac{ \log (d/h)}{nh}},
\end{split}
\end{equation}
where the last inequality holds by the assumption $\log (d/h)/nh = o(1)$.
By Lemma~\ref{lemma:ep1} with $\tau^2 = 3\cdot\kappa\cdot \bar{f}_Z\cdot \|K\|_2^2/h$, $\eta= 2\cdot \|K\|_{\infty}\cdot M^2_X\cdot  \log d/h$, $\EE[Y] \le C\cdot \sqrt{\log (d/h)/(nh)}$, and picking $t = \sqrt{\log d/n}$, for sufficiently large $n$, we have 
\begin{equation*}
\begin{split}
 \underset{z\in [0,1]}{\sup} \; \underset{j,k\in [d]}{\max} \; \frac{1}{\sqrt{n}}\cdot\left|\GG_n (g_{z,jk}) \right| &=  \underset{z\in [0,1]}{\sup} \; \underset{j,k\in [d]}{\max} \; \frac{1}{n} \cdot \left|\sum_{i\in[n]} [g_{z,jk} (Z_i,X_{ij},Y_{ik})- \EE\{g_{z,jk}(Z,X_j,Y_k)\}] \right| \\
 &\le 
 C\cdot \left\{ \sqrt{\frac{\log (d/h)}{nh}} +  \sqrt{\frac{\log d}{nh}} \cdot \sqrt{1+ \log d\cdot \sqrt{\frac{\log (d/h)}{nh}}}+ \frac{\log^2 d}{nh}\right\}\\
 &\le C \cdot  \sqrt{\frac{\log (d/h)}{nh}},
\end{split}
\end{equation*}
with probability at least $1-2/d$.  The second inequality holds by the assumption that $\log^2 (d/h)/(nh) =o(1)$. 
Multiplying both sides of the equation by $\sqrt{n}$, we completed the proof of Lemma~\ref{lemma:empirical process g}.

\section{Proof of Theorem~\ref{theorem:gaussian multiplier bootstrap}}
\label{appendix:proof:theorem:gaussian multiplier bootstrap}
In this section, we provide the proof of Theorem \ref{theorem:gaussian multiplier bootstrap}.  To prove Theorem~\ref{theorem:gaussian multiplier bootstrap}, we use a similar set of arguments in the series of work on Gaussian multiplier bootstrap of the supreme of empirical process (see, for instance, \citealp{chernozhukov2013gaussian,chernozhukov2014anti,chernozhukov2014gaussian}).
Recall from (\ref{Eq:quantile}) and (\ref{Eq:approx test}) that 
\begin{equation}
\label{proof:theorem:gaussian multiplier bootstrap T}
T_E=  \underset{z\in [0,1]}{\sup} \;\underset{(j,k)\in E(z)}{\max} \; 
\sqrt{nh}\cdot  \left|\hat{\bTheta}_{jk}^\mathrm{de} (z) - \bTheta_{jk} (z)\right| \cdot \PP_n (w_z)
\end{equation}
and
\begin{equation}
\label{proof:theorem:gaussian multiplier bootstrap TB}
T^B_E =  \underset{z\in [0,1]}{\sup} \; \underset{(j,k) \in E(z) }{\max} \;\sqrt{nh}\cdot \left| \frac{ \sum_{i\in [n]} \left\{\hat{\bTheta}_j (z)\right\}^T K_h (Z_i-z) \left\{ \bX_i \bY_i^T\hat{\bTheta}_k (z)  - \mathbf{e}_k\right\} \xi_i/n}{ \left\{ \hat{\bTheta}_j(z)\right\}^T \hat{\bSigma}_j(z)}\right|,
\end{equation}
respectively, where $\xi_i \sim N(0,1)$. 
Note that for notational convenience, we drop the subscript $E$ from $T_E$ and $T_E^B$ throughout the proof.

We aim to show that $T^B$ is a good approximation of $T$.  However, $T$ and $T^B$ are not exact averages. To apply the results in \citealp{chernozhukov2014anti},   
we define four intermediate processes:
\begin{equation}
\label{proof:theorem:gaussian multiplier bootstrap1}
T_0 = \underset{z\in [0,1]}{\sup} \; \underset{(j,k)\in E(z)}{\max} \;
\sqrt{nh}\cdot \left| \sum_{i\in [n]}  \left\{ \bTheta_j (z) \right\}^T
 K_h (Z_i-z) \left\{ \bX_i \bY_i^T  - \bSigma (z) \right\} \bTheta_k (z)/n\right|;
\end{equation}
\begin{equation}
\begin{split}
\label{proof:theorem:gaussian multiplier bootstrap2}
T_{00}&= \underset{z\in [0,1]}{\sup} \; \underset{(j,k)\in E(z)}{\max} \;
\sqrt{nh} \cdot \bigg|   \sum_{i\in [n]}  \left\{ \bTheta_j (z) \right\}^T
 K_h (Z_i-z) \left\{ \bX_i \bY_i^T  - \bSigma (z) \right\} \bTheta_k (z)/n\\
 &\qquad \qquad \qquad \qquad -    \left\{ \bTheta_j (z) \right\}^T \bigg(
\Big[ \EE\{K_h (Z-z)\bX \bY^T\} - \EE\{K_h(Z-z)\} \bSigma (z)   \Big]  \bigg) \bTheta_k (z) /n\bigg|;
\end{split}
\end{equation}

\begin{equation}
\label{proof:theorem:gaussian multiplier bootstrap3}
T_0^B = \underset{z\in [0,1]}{\sup} \; \underset{(j,k)\in E(z)}{\max} \;
\sqrt{nh} \cdot \left| \sum_{i\in [n]} \left[ \left\{ \bTheta_j (z) \right\}^T
 K_h (Z_i-z) \left\{ \bX_i \bY_i^T  - \bSigma (z) \right\} \bTheta_k (z) \right]\xi _i/n\right|,
 \end{equation}

\begin{equation}
\begin{split}
\label{proof:theorem:gaussian multiplier bootstrap4}
T_{00}^B &=   \underset{z\in [0,1]}{\sup} \; \underset{(j,k)\in E(z)}{\max} \;
\sqrt{nh} \cdot \bigg|   \sum_{i\in [n]}  \Big\{\left\{ \bTheta_j (z) \right\}^T
 K_h (Z_i-z) \left\{ \bX_i \bY_i^T  - \bSigma (z) \right\} \bTheta_k (z)/n\\
 &\qquad \qquad \qquad \qquad -    \left\{ \bTheta_j (z) \right\}^T \Big( 
\Big[ \EE\{K_h (Z-z)\bX \bY^T\} - \EE\{K_h(Z-z)\} \bSigma (z)   \Big]  \Big) \bTheta_k (z)  \Big\} \cdot \xi_i/n\bigg|;
 \end{split}
\end{equation}
where $\xi_i \stackrel{\mathrm{i.i.d.}}{\sim} N(0,1)$.

To prove Theorem~\ref{theorem:gaussian multiplier bootstrap}, we show that $T_{00}$ is a good approximation of $T$ and that $T_{00}^B$ is a good approximation of $T^B$.
We then show that there exists a Gaussian process $W$ such that 
both $T_{00}^B$ and $T_{00}$ can be accurately approximated by $W$.  
This is done by applications of Theorems A.1 and A.2 in \citet{chernozhukov2014anti}.  The following summarizes the chain of empirical and Gaussian processes that we are going to study
\[
T \longleftrightarrow T_0 \longleftrightarrow T_{00} \longleftrightarrow W \longleftrightarrow T_{00}^B \longleftrightarrow T_{0}^B \longleftrightarrow T^B.
\]

The following lemma provides an approximation error between the statistic $T$ and the intermediate empirical process $T_{00}$.
\begin{lemma}
\label{lemma:T0andT}
Assume that $h^2 + \sqrt{\log(d/h)/nh} = o(1)$. Under Assumptions~\ref{ass:marginal}-\ref{ass:covariance}, for sufficiently large $n$, there exists a universal constant $C>0$ such that 
\[
 |T-T_{00}| \le C \cdot \left\{\sqrt{nh^5} + s\cdot \sqrt{nh^9} + \frac{s \cdot \log (d/h)}{\sqrt{nh}} +\cdot  s\cdot h^2\cdot \sqrt{\log (d/h)} \right\},
\]
with probability at least $1-1/d$.
\end{lemma}
\begin{proof}
The proof is deferred to \ref{proof:lemma:T0andT}.
\end{proof}

We now apply Theorems A.1 and A.2 in \citet{chernozhukov2014anti} to show that there exists a Gaussian process $W$ such that the quantities $|T_{00}-W|$ and $|T_{00}^B-W|$ can be controlled, respectively. 
The results are stated in the following lemmas.
\begin{lemma}
\label{lemma:T00andW}
Assume that $\log^{6} s \cdot \log^4 (d/h)/(nh) = o(1)$.  Under Assumptions~\ref{ass:marginal}-\ref{ass:covariance}, for sufficiently large $n$, there exists universal constants $C,C'>0$ such that 
\[
P\left[|T_{00}-W| \ge C\cdot   \left\{\frac{\log^{6} (s) \cdot \log^4 (d/h)}{nh}\right\}^{1/8} \right]  \le C'\cdot  \left\{\frac{\log^{6} (s) \cdot \log^4 (d/h)}{nh}\right\}^{1/8}.
\]
\end{lemma}
\begin{proof}
The proof is deferred to \ref{proof:lemma:T00andW}.
\end{proof}

\begin{lemma}
\label{lemma:T00BandW}
Assume that $\log^4 (s) \cdot \log^3 (d/h)/(nh) = o(1)$.  Under Assumptions~\ref{ass:marginal}-\ref{ass:covariance}, for sufficiently large $n$, there exists universal constants $C,C''>0$ such that 
\[
P\left[ |T_{00}^B- W| > C\cdot  \left\{ \frac{\log^4 (s) \cdot \log^3 (d/h)}{nh}  \right\}^{1/8 } \; \;\Big|\; \left\{(Z_i,\bX_i,\bY_i)\right\}_{i\in [n]} \right] \le C''\cdot  \left\{ \frac{\log^4 (s) \cdot \log^3 (d/h)}{nh}  \right\}^{1/8},
\]
with probability at least $1-3/n$.
\end{lemma}
\begin{proof}
The proof is deferred to \ref{proof:lemma:T00BandW}.
\end{proof}

Finally, the following lemma provides an upper bound on the difference between $T^B$ and $T_{00}^B$, conditioned on the data $\left\{(Z_i,\bX_i,\bY_i)\right\}_{i\in [n]}$.

\begin{lemma}
\label{lemma:T00BandTB}
Assume that $ s \cdot \sqrt{h^3 \log^3 (d/h)} + s \cdot \sqrt{\log^4 (d/h)/nh^2} + \sqrt{h^5 \log n} = o(1)$.  Under Assumptions~\ref{ass:marginal}-\ref{ass:covariance}, for sufficiently large $n$, there exists universal constants $C,C''>0$ such that, with probability at least $1-1/d$,  
\[
\small
P\left[ |T^B- T_{00}^B| > C \cdot \sqrt{h^3 \log^3 (d/h)} + s \cdot \sqrt{\frac{\log^4 (d/h)}{nh^2}} + \sqrt{h^5 \log n} \; \;\Big|\; \left\{(Z_i,\bX_i,\bY_i)\right\}_{i\in [n]} \right] \le 2/d + 1/n.
\]

\end{lemma}
\begin{proof}
The proof is deferred to \ref{proof:lemma:T00BandTB}.
\end{proof}

With Lemmas~\ref{lemma:T0andT}-\ref{lemma:T00BandTB}, we are now ready to prove Theorem~\ref{theorem:gaussian multiplier bootstrap}.
\subsection{Proof of Theorem~\ref{theorem:gaussian multiplier bootstrap}}
\label{proof:theorem:gaussian multiplier bootstrap}
Recall that for notational convenience, we drop the subscript $E$ from $T_E$ and $T_E^B$ throughout the proof.
In this section, we show that $T$ can be well-approximated by the $(1-\alpha)$-conditional quantile of $T^B$, i.e., $P\{T\ge c(1-\alpha)\}\le \alpha$.
For notational convenience, we let $r=r_1+r_2+r_3+r_4$, where 
\begin{equation*}
\begin{split}
r_1 &= \sqrt{nh^5} + s\cdot \sqrt{nh^9} + \frac{s \cdot \log (d/h)}{\sqrt{nh}} +\cdot  s\cdot h^2\cdot \sqrt{\log (d/h)}\\
r_2 &=   \left\{\frac{\log^{6} s \cdot \log^4 (d/h)}{nh}\right\}^{1/8}\\
r_3 &=   \left\{\frac{\log^{4} s \cdot \log^3 (d/h)}{nh}\right\}^{1/8}\\
r_4 &=   \sqrt{h^3 \log^3 (d/h)} + s \cdot \sqrt{\frac{\log^4 (d/h)}{nh^2}} + \sqrt{h^5 \log n}.
\end{split}
\end{equation*}
These are the scaling that appears in Lemmas~\ref{lemma:T0andT}-\ref{lemma:T00BandTB}.
By Lemmas~\ref{lemma:T0andT} and~\ref{lemma:T00andW}, it can be shown that 
\begin{equation}
\label{Eq:proof of theorem4.7-1}
P(|T-W| \ge 2r_2) \le P(|T-T_{00}| + |T_{00}-W| \ge 2r_2) \le 2r_2,
\end{equation}
since $r_2 \ge r_1$ and $r_2\ge 1/d$. 
With some abuse of notation, throughout the proof, we write $P_{\xi}(T^B \ge t)$ to indicate $P[T^B \ge t \mid \{(Z_i,\bX_i,\bY_i)\}_{i\in[n]}]$.
By Lemmas~\ref{lemma:T00BandW} and~\ref{lemma:T00BandTB},
we have 
\begin{equation}
\label{Eq:proof of theorem4.7-2}
P_{\xi}(|T^B-W| \ge 2r_2) \le P_{\xi}(|T^B-T^B_{00}| + |T^B_{00}-W| \ge2r_2) \le 2r_2,
\end{equation}
since $r_2 \ge r_3$ and $r_2\ge 2/d+1/n$. Define the event 
\[
\cE = \left(  P[  |T_{00}^B-W| > r_2 \mid \{(Z_i,\bX_i,\bY_i)\}_{i\in [n]} \le r_2]    \right),
\]
and note that $P(\cE) \ge 1-2/d-4/n$ by Lemmas~\ref{lemma:T00BandW} and~\ref{lemma:T00BandTB}. Throughout the proof, we condition on the event $\cE$.

By the triangle inequality, we obtain 
\begin{equation}
\label{Eq:proof of theorem4.7-3}
\begin{split}
P\{T\le c(1-\alpha)\} &\ge 1-P\{T-W+W+2r_2 \ge c(1-\alpha) +2 r_2\}\\
&\ge 1- P(|T-W|\ge 2r_2) - P\{W\ge c(1-\alpha) -2r_2\}\\
&\ge P\{|W|\le c(1-\alpha)-2r_2\} - 2r_2,
\end{split}
\end{equation}
where the last inequality follows from (\ref{Eq:proof of theorem4.7-1}).
By a similar argument and by (\ref{Eq:proof of theorem4.7-2}), we have 
\begin{equation}
\label{Eq:proof of theorem4.7-4}
\begin{split}
P\{W \le c(1-\alpha)-2r_2\} &\ge P_{\xi}\{T^B \le c(1-\alpha)-4r_2\} - 2r_2\\
&\ge P_{\xi}\{T^B \le c(1-\alpha)\} - 2r_2 - P_{\xi}\{|T^B-c(1-\alpha)|\le r_2\},
\end{split}
\end{equation}
where the last inequality follows from the fact that $P(X\le t-\epsilon) - P(X\le t) \ge -P(|X-t|\le \epsilon)$ for any $\epsilon>0$.
Thus, combining (\ref{Eq:proof of theorem4.7-3}) and (\ref{Eq:proof of theorem4.7-4}), we obtain 
\begin{equation}
\label{Eq:proof of theorem4.7-5}
P\{T\le c(1-\alpha)\} \ge 1-\alpha - 4r_2 -  P_{\xi}\{|T^B-c(1-\alpha)|\le r_2\}.
\end{equation}
It remains to show that the quantity $ P_{\xi}\{|T^B-c(1-\alpha)|\le r_2\}$ converges to zero as we increase $n$.

By the definition of $T_{00}$ and from (\ref{EqA:jzjk}), we have 
\[
T_{00} = \underset{z\in [0,1]}{\sup}\; \underset{j,k\in [d]}{\max }\;\frac{1}{\sqrt{n}}  \left|\sum_{i\in [n]}J_{z,jk} (Z_i,\bX_i,\bY_i)\right| \qquad  \mathrm{and} \qquad
T_{00}^B = \underset{z\in [0,1]}{\sup}\; \underset{j,k\in [d]}{\max }\;\frac{1}{\sqrt{n}}\left|  \sum_{i\in [n]}J_{z,jk} (Z_i,\bX_i,\bY_i) \xi_i\right|.
\]
Let $\hat{\sigma}^2_{z,jk} = \sum_{i=1}^n J^2_{z,jk} (Z_i,\bX_i,\bY_i) /n$ be the conditional variance, and let $\underline{\sigma} = \inf_{z,jk} \hat{\sigma}_{z,jk}$ and $\bar{\sigma}=\sup_{z,jk} \hat{\sigma}_{z,jk}$.
By Lemma A.1 of~\citet{chernozhukov2014gaussian} and Theorem 3 of~\citet{chernozhukov2013gaussian}, we obtain
\begin{equation}
\label{Eq:proof of theorem4.7-6}
\begin{split}
&P_{\xi}\{|T^B-c(1-\alpha)|\le r_2\}\\
 &\le C\cdot \bar{\sigma}/\underline{\sigma} \cdot r_2\cdot \{  \EE [T^B \mid  \{(Z_i,\bX_i,\bY_i)\}_{i\in [n]} ]
+ \sqrt{1\vee \log (\underline{\sigma}/r_2)}\}\\
&\le C\cdot \bar{\sigma}/\underline{\sigma} \cdot r_2\cdot \{  \EE [T^B_{00} \mid  \{(Z_i,\bX_i,\bY_i)\}_{i\in [n]} ] + \EE [|T^B-T_{00}^B| \mid  \{(Z_i,\bX_i,\bY_i)\}_{i\in [n]} ]
+ \sqrt{1\vee \log (\underline{\sigma}/r_2)}\}.
\end{split}
\end{equation}

We first calculate the quantity $\bar{\sigma}$.  By~(\ref{proof:lemma:T00andW1}), we have 
\begin{equation}
\label{Eq:proof of theorem4.7-7}
\sup_{z\in [0,1]}\max_{j,k\in [d]} \|J_{z,jk}^2 (Z_i,\bX_i,\bY_i)\|_{\infty} \le C \cdot \log^2 s /h.
\end{equation}
Moreover, by (\ref{proof:lemma:T00andW1}), we have 
\begin{equation}
\label{Eq:proof of theorem4.7-8}
\sup_{z\in [0,1]}\max_{j,k\in [d]} \EE[J^4_{z,jk}(Z_i,\bX_i,\bY_i)]  \le  C \cdot \log^4 s /h^2.
\end{equation}
Define the function class  $\cJ'  = \{J^2_{z,jk}(\cdot) \mid z\in[0,1], j,k\in [d]\}$.  By Lemmas~\ref{lemma:coveringnumber2}, \ref{lemma:coveringnumberjzjk1} and \ref{lemma:coveringnumberjzjk2}, we have 
\begin{equation}
\label{Eq:proof of theorem4.7-9}
\underset{Q}{\sup}\; N\{\cJ',L_2(Q),\epsilon\}  \le C \cdot d^2\cdot \left(  
\frac{d^{17/24} \cdot \log^{3/4} d }{h^{11/12}\cdot \epsilon}
\right)^{24}.
\end{equation}
Thus, applying Lemma~\ref{lemma:ep2} with $\sigma^2_P = C \cdot \log^4 s/h^2$ and $\|F\|_{L_2(\PP_n)} \le C\cdot d^2 \cdot (d^{17/24}\cdot \log^{3/4}d / h^{11/12})^{24}$, we have 
\[
\EE \left[\sup_{z\in [0,1]} \max_{j,k\in [d]} \frac{1}{n}  \left|  \sum_{i\in [n]} J_{z,jk}^2(Z_i,\bX_i,\bY_i) - \EE\{J_{z,jk}^2(Z,\bX,\bY)\}  \right|    \right] \le C\cdot \sqrt{\frac{\log^5 (d/h)}{nh^2}}.
\]
By an application of the Markov's inequality, we obtain
\begin{equation}
\small
\label{Eq:proof of theorem4.7-10}
P\left( \sup_{z\in [0,1]} \max_{j,k\in [d]} \left[ \frac{1}{n}\sum_{i\in [n]} 
J^2_{z,jk} (Z_i,\bX_i,\bY_i) - \EE \{J^2_{z,jk}(Z_i,\bX_i,\bY_i)\}\right] \ge C \cdot \left\{ \frac{\log^5 (d/h)}{nh^2}\right\}^{1/4}
   \right) \le  C \cdot \left\{ \frac{\log^5 (d/h)}{nh^2}\right\}^{1/4}.
\end{equation}
Thus, we have with probability at least $1-C \cdot \left\{ \log^5 (d/h)/(nh^2)\right\}^{1/4}$,
\begin{equation}
\small
\label{Eq:proof of theorem4.7-11}
\bar{\sigma}^2 = \sup_{z\in[0,1]} \max_{j,k\in[d]} \frac{1}{n}\sum_{i\in [n]} J^2_{z,jk}(Z_i,\bX_i,\bY_i) \le  \sup_{z\in[0,1]} \max_{j,k\in[d]} \EE \{J^2_{z,jk}(Z_i,\bX_i,\bY_i)\}   + 
C \cdot \left\{ \frac{\log^5 (d/h)}{nh^2}\right\}^{1/4} \le C \cdot \log^2 s,
\end{equation}
where the last inequality follows from (\ref{proof:lemma:T00andW6}) for sufficiently large $n$.
By Lemma \ref{lm:var-lower-bd}, we have
$
\inf_{z,j,k}\EE\{J^2_{z,jk}(Z,\bX,\bY)\} \ge c>0.
$
Therefore, we have
\[
\underline{\sigma}^2 =\inf_{z,j,k} \frac{1}{n}\sum_{i=1}^n J^2_{z,jk}(Z_i,\bX_i,\bY_i) \ge c - \sup_{z,j,k} \frac{1}{n}\sum_{i=1}^n [J^2_{z,jk}(Z_i,\bX_i,\bY_i)  - \EE\{ J^2_{z\mid(j,k)}(Z,\bX,\bY)\}]  \ge c/2>0,
\]
with probability  at least $1-C \cdot \left\{ \log^5 (d/h)/(nh^2)\right\}^{1/4}$.

Next, we calculate the quantity $  \EE [T^B_{00} \mid  \{(Z_i,\bX_i,\bY_i)\}_{i\in [n]} ] $. By Dudley's inequality  (see, e.g., Corollary 2.2.8 in \citealp{van1996weak}) and (\ref{proof:lemma:T00andW7}), we obtain
\begin{equation}
\label{Eq:proof of theorem4.7-12}
\EE[T_{00}^B \mid  \{(Z_i,\bX_i,\bY_i)\}_{i\in [n]}] \le C\cdot \log s\cdot \sqrt{\log (d/h)}.
\end{equation}
Moreover, by Lemma~\ref{lemma:T00BandTB}, we have 
\begin{equation}
\label{Eq:proof of theorem4.7-13}
\EE[|T^B-T_{00}^B| \mid  \{(Z_i,\bX_i,\bY_i)\}_{i\in [n]}] \le C \cdot \sqrt{h^3 \log^3 (d/h)} + s \cdot \sqrt{\frac{\log^4 (d/h)}{nh^2}} + \sqrt{h^5 \log n} \le r_2,
\end{equation}
with probability at least $1-2/d-1/n$.  
Substituting~(\ref{Eq:proof of theorem4.7-11}), (\ref{Eq:proof of theorem4.7-12}), and (\ref{Eq:proof of theorem4.7-13}) into (\ref{Eq:proof of theorem4.7-6}), we obtain
\begin{equation}
\label{Eq:proof of theorem4.7-14}
P_{\xi}\{|T^B-c(1-\alpha)|\le r_2\}\le C\cdot  \left\{ \frac{\log^{22} s \cdot \log^8(d/h)}{nh} \right\}^{1/8}.
\end{equation}
 Thus, substituting (\ref{Eq:proof of theorem4.7-14}) into (\ref{Eq:proof of theorem4.7-5}), we have
\[
P\{T\le c(1-\alpha)\} \ge 1-\alpha -4r_2 -  \frac{\log^{22} s \cdot \log^8(d/h)}{nh} .
\]
By the scaling assumptions, $r_2 = o(1)$ and $ \log^{22} s \cdot \log^8(d/h)/(nh) = o(1)$. Thus, this implies that 
\[
\lim_{n\rightarrow\infty} P\{T\le c(1-\alpha)\} \ge 1-\alpha,
\]
which implies that 
\[
\lim_{n\rightarrow\infty} P\{T\ge c(1-\alpha)\} \le \alpha,
\]
as desired.

\subsection{Proof of Lemma~\ref{lemma:T0andT}}
\label{proof:lemma:T0andT}
In this section, we show that $|T-T_{00}|$ is upper bounded by the quantity  
\[
 C \cdot \left\{\sqrt{nh^5} + s\cdot \sqrt{nh^9} + \frac{s \cdot \log (d/h)}{\sqrt{nh}} +\cdot  s\cdot h^2\cdot \sqrt{\log (d/h)} \right\}
 \]  
with high probability for sufficiently large constant $C>0$.  
By the triangle inequality, we have $|T-T_{00}| \le |T-T_0| + |T_0-T_{00}|$.  Thus, is suffices to obtain upper bounds for the terms $|T-T_{0}|$ and $|T_0-T_{00}|$. \\

\textbf{Upper Bound for $|T-T_0|$:}
Let $\tilde{\bTheta}_k = \left( \hat{\bTheta}_{1k},\ldots,\hat{\bTheta}_{(j-1)k},\bTheta_{jk},\hat{\bTheta}_{(j+1)k},\ldots,\hat{\bTheta}_{dk}   \right)^T \in \RR^d $.
Then, the statistics $T$ can be rewritten as  
\begin{equation}
\label{Eq:proof:lemma:T0andT1}
\begin{split}
T&= \underset{z\in [0,1]}{\sup} \;\underset{(j,k)\in E(z)}{\max} \; 
\sqrt{nh}\cdot  \left|\hat{\bTheta}_{jk}^\mathrm{de} (z) - \bTheta_{jk} (z)\right| \cdot \PP_n (w_z) \\
&= \underset{z\in [0,1]}{\sup} \;\underset{(j,k)\in E(z)}{\max} \; 
\sqrt{nh}\cdot  \left|\hat{\bTheta}_{jk} (z) - \bTheta_{jk} (z)
- \frac{\left\{ \hat{\bTheta}_j(z)\right\}^T  \left\{ \hat{\bSigma}(z) \hat{\bTheta}_k - \mathbf{e}_k  \right\}}{ \left\{\hat{\bTheta}_j(z)\right\}^T \hat{\bSigma}_j (z)}
\right| \cdot \PP_n (w_z) \\
&=  \underset{z\in [0,1]}{\sup} \;\underset{(j,k)\in E(z)}{\max} \; 
 \sqrt{nh}\cdot  \left| \frac{\left\{ \hat{\bTheta}_j(z)\right\}^T  \left\{ \hat{\bSigma}(z) \tilde{\bTheta}_k - \mathbf{e}_k  \right\}}{ \left\{\hat{\bTheta}_j(z)\right\}^T \hat{\bSigma}_j (z)}
\right| \cdot \PP_n (w_z).
\end{split}
\end{equation}
To obtain an upper bound on the difference between $T$ and $T_0$, we make use of the following inequality:
\begin{equation}
\label{Eq:proof:lemma:T0andT2}
\left| \frac{x}{1+\delta} - y \right| \le 2\cdot y\cdot |\delta| + 2\cdot |x-y| \qquad \qquad \mathrm{for\;any\;} |\delta|\le \frac{1}{2}.
\end{equation}
Recall  from (\ref{proof:theorem:gaussian multiplier bootstrap1}) that 
\[
T_0 = \underset{z\in [0,1]}{\sup} \; \underset{(j,k)\in E(z)}{\max} \; \sqrt{nh}\cdot \left| \sum_{i\in [n]} \left\{ \bTheta_j (z) \right\}^T
 K_h (Z_i-z) \left\{ \bX_i \bY_i^T  - \bSigma (z) \right\} \bTheta_k (z)/n\right|.
\]
 Applying (\ref{Eq:proof:lemma:T0andT2}) with $x=\{ \hat{\bTheta}_j(z)\}^T\{\hat{\bSigma}(z) \tilde{\bTheta}_k - \mathbf{e}_k  \}  $, $\delta= \{\hat{\bTheta}_j(z)\}^T \hat{\bSigma}_j (z)-1$, and $y=\{ \bTheta_j (z) \}^T
\{ \hat{\bSigma} (z) - \bSigma (z) \} \bTheta_k (z)$, and by the triangle inequality, we have 
\begin{equation}
\label{Eq:proof:lemma:T0andT3}
\small
\begin{split}
&|T-T_0|\\ &\le  \underset{z\in [0,1]}{\sup} \;\underset{(j,k)\in E(z)}{\max} \; 
\sqrt{nh} \cdot \Bigg|     \frac{\left\{ \hat{\bTheta}_j(z)\right\}^T  \left\{ \hat{\bSigma}(z) \tilde{\bTheta}_k - \mathbf{e}_k  \right\} \cdot \PP_n (w_z)}{ \left\{\hat{\bTheta}_j(z)\right\}^T \hat{\bSigma}_j (z)} - \frac{1}{n} \sum_{i\in [n]} \left\{ \bTheta_j (z) \right\}^T
 K_h (Z_i-z) \left\{ \bX_i \bY_i^T  - \bSigma (z) \right\} \bTheta_k (z) \Bigg|\\
&\le  \underset{z\in [0,1]}{\sup} \;\underset{(j,k)\in E(z)}{\max} \; 
\sqrt{nh} \cdot \Bigg|    \frac{\left\{ \hat{\bTheta}_j(z)\right\}^T  \left\{ \hat{\bSigma}(z) \tilde{\bTheta}_k - \mathbf{e}_k  \right\}}{ \left\{\hat{\bTheta}_j(z)\right\}^T \hat{\bSigma}_j (z)} - \left\{ \bTheta_j (z) \right\}^T
 \left\{ \hat{\bSigma}(z)   - \bSigma (z) \right\} \bTheta_k (z) \Bigg|  \cdot \left|\PP_n (w_z)\right|\\
 &\le \underbrace{2 \underset{z\in [0,1]}{\sup} \;\underset{(j,k)\in E(z)}{\max}\; 
\sqrt{nh} \cdot  \left[  \left\{ \bTheta_j (z) \right\}^T
 \left\{ \hat{\bSigma} (z)  - \bSigma (z) \right\} \bTheta_k (z) \cdot  \left| \left\{\hat{\bTheta}_j(z)\right\}^T \hat{\bSigma}_j (z)-1 \right|  \right] \cdot \left|\PP_n (w_z)\right| }_{I_1}\\
 &+\underbrace{2 \underset{z\in [0,1]}{\sup} \;\underset{(j,k)\in E(z)}{\max} \; 
\sqrt{nh} \cdot  \left[\left\{ \hat{\bTheta}_j(z)\right\}^T\left\{\hat{\bSigma}(z) \tilde{\bTheta}_k - \mathbf{e}_k  \right\} - 
 \left\{ \bTheta_j (z) \right\}^T
\left\{ \hat{\bSigma}(z)  - \bSigma (z) \right\} \bTheta_k (z)
\right] \cdot \left|\PP_n (w_z)\right| }_{I_2}.
\end{split}
\end{equation}
It remains to obtain upper bounds for $I_1$ and $I_2$ in (\ref{Eq:proof:lemma:T0andT3}).\\

\textbf{Upper bound for $I_1$:} By Corollary~\ref{theorem:inverse estimation}, we have  
\begin{equation}
\label{Eq:proof:lemma:T0andT4}
\underset{z\in [0,1]}{\sup} \;\underset{j\in [d]}{\max} \; \left| \left\{\hat{\bTheta}_j(z)\right\}^T \hat{\bSigma}_j (z)-1 \right|  \le 
C \cdot \left[ h^2 + \sqrt{\frac{\log(d/h)}{nh}}\right].
\end{equation}
Moreover, by Lemmas~\ref{lemma:empirical process w} and~\ref{lemma:bias}, we have
\begin{equation}
\label{Eq:proof:lemma:T0andT5}
\begin{split}
\underset{z\in [0,1]}{\sup} \; \left|\PP_n (w_z)\right| &\le  \left|\EE \left\{\PP_n (w_z) \right\}\right| + C\cdot \sqrt{\frac{\log (d/h)}{nh}}= \bar{f}_Z + \cO\left\{h^2+\sqrt{\frac{\log (d/h)}{nh}}  \right\},
\end{split}
\end{equation}
with probability at least $1-1/d$.
Thus, by Holder's inequality, we have 
\begin{equation}
\label{Eq:proof:lemma:T0andT6}
\small
\begin{split}
I_1 &\le 2 \underset{z\in [0,1]}{\sup} \;\underset{(j,k)\in E(z)}{\max} \; 
\sqrt{nh} \cdot  \left| \left\{\hat{\bTheta}_j(z)\right\}^T \hat{\bSigma}_j (z)-1 \right|  \cdot \left|\PP_n (w_z)\right|\cdot \left| 
\left\{\bTheta_j(z)\right\}^T  \left\{ \hat{\bSigma}(z)-\bSigma(z) \right\} \bTheta_k (z)
\right|\\
&\le 2 \underset{z\in [0,1]}{\sup} \;\underset{(j,k)\in E(z)}{\max} \; 
\sqrt{nh} \cdot  \left| \left\{\hat{\bTheta}_j(z)\right\}^T \hat{\bSigma}_j (z)-1 \right|  \cdot \left|\PP_n (w_z)\right|\cdot 
\|\bTheta_j(z)\|_{1}^2 \cdot  \| \hat{\bSigma}(z)-\bSigma(z) \|_{\max} 
\\
&\le 2 \cdot M^2 \cdot \sqrt{nh} \cdot C\cdot \left\{ h^2+\sqrt{\frac{\log (d/h)}{nh}}\right\}\cdot \left[ \bar{f}_Z + \cO \left\{ h^2+\sqrt{\frac{\log (d/h)}{nh}}\right\} \right]\cdot \left\{ h^2+\sqrt{\frac{\log (d/h)}{nh}} \right\}\\
&\le C \cdot \sqrt{nh} \cdot \left\{  h^2+\sqrt{\frac{\log (d/h)}{nh}} \right\}^2,
\end{split}
\end{equation}
with probability greater than $1-4/d$, where the third inequality holds by Theorem~\ref{theorem:estimation error}, (\ref{Eq:proof:lemma:T0andT4}), and (\ref{Eq:proof:lemma:T0andT5}).  \\

\textbf{Upper bound for $I_2$:}
To obtain an upper bound for $I_2$, we first decompose the quantity $\sqrt{nh} \cdot \{ \hat{\bTheta}_j(z)\}^T\{ \hat{\bSigma}(z) \tilde{\bTheta}_k - \mathbf{e}_k\}$ 
into the following
\begin{equation*}
\begin{split}
&\sqrt{nh} \cdot \left\{ \hat{\bTheta}_j(z)\right\}^T  \left\{ \hat{\bSigma}(z) \tilde{\bTheta}_k - \mathbf{e}_k  \right\} \\
=& \underbrace{\sqrt{nh} \cdot \left\{\hat{\bTheta}_j(z)\right\}^T \hat{\bSigma}(z) \left\{ \tilde{\bTheta}_k (z) - \bTheta_k (z) \right\}}_{I_{21}}+
\underbrace{\sqrt{nh} \cdot \left\{\hat{\bTheta}_j(z)\right\}^T \left\{\hat{\bSigma}(z)  \ - \bSigma(z) \right\}\bTheta_k (z)}_{I_{22}}.
 \end{split}
 \end{equation*}
Next, we show that $I_{21}$ converges to zero and that the difference between $I_{22}$ and the term $\sqrt{nh}\cdot \{\bTheta_j (z)\}^T \{ \hat{\bSigma}(z) -\bSigma(z) \} \bTheta_k (z)$ is small.\\

 \textbf{Upper bound for $I_{21}$:} By Holder's inequality and Corollary~\ref{theorem:inverse estimation}, we have 
\begin{equation}
\label{Eq:proof:lemma:T0andT7}
\begin{split}
  |I_{21} |
 &\le  \underset{z\in [0,1]}{\sup} \;\underset{(j,k)\in E(z)}{\max} \; 
\sqrt{nh} \cdot \left\|\left\{\hat{\bTheta}_j(z)\right\}^T \hat{\bSigma}_{-j} (z)  \right\|_{\infty} \cdot \left\|   \hat{\bTheta}_k(z) - \bTheta_{k}(z)  \right\|_1\\
&\le C\cdot \sqrt{nh} \cdot s \cdot \left\{ h^2 + \sqrt{\frac{\log(d/h)}{nh}}\right\}^2 \\ 
&\le C \cdot \left\{ s\cdot \sqrt{nh^9}  + \frac{s \cdot \log (d/h)}{\sqrt{nh}} +  s\cdot h^2\cdot \sqrt{\log (d/h)}\right\},
\end{split}
\end{equation}
with probability at least $1-1/d$. \\

\textbf{Decomposition of  $I_{22}$:} By adding and subtracting terms, we have 
\begin{equation}
\label{Eq:proof:lemma:T0andT8}
I_{22} = \underbrace{\sqrt{nh} \cdot \left\{  \hat{\bTheta}_{j}(z) - \bTheta_j (z)  \right\}^T \left\{ \hat{\bSigma}(z) - \bSigma(z)   \right\} \bTheta_k (z) }_{I_{221}}+\underbrace{\sqrt{nh} \cdot \left\{\bTheta_j (z)  \right\}^T \left\{ \hat{\bSigma}(z) - \bSigma (z)  \right\} \bTheta_k (z) }_{I_{222}}.
\end{equation}
Similar to (\ref{Eq:proof:lemma:T0andT7}), we have
\begin{equation}
\label{Eq:proof:lemma:T0andT9}
\begin{split}
|I_{221}| &\le 
 \underset{z\in [0,1]}{\sup} \;\underset{(j,k)\in E(z)}{\max} \;  \sqrt{nh} \cdot 
\left\|  \hat{\bTheta}_j (z) - \bTheta_j (z) \right\|_1\cdot  \left\| \hat{\bSigma}(z) - \bSigma (z)   \right\|_{\max} \cdot \|\bTheta_k (z)\|_1\\
&\le C\cdot \sqrt{nh} \cdot M \cdot s \cdot \left\{h^2 + \sqrt{\frac{\log(d/h)}{nh}} \right\}^2\\
&\le C \cdot \left\{ s\cdot \sqrt{nh^9}  + \frac{s \cdot \log (d/h)}{\sqrt{nh}} +  s\cdot h^2\cdot \sqrt{\log (d/h)}\right\},
\end{split} 
\end{equation}
where the second inequality holds by Holder's inequality, Corollary~\ref{theorem:inverse estimation}, and the fact that $\bTheta (z) \in \cU_{s,M}$.

Combining the results (\ref{Eq:proof:lemma:T0andT7})-(\ref{Eq:proof:lemma:T0andT9}), we have 
\begin{equation}
\label{Eq:proof:lemma:T0andT10}
\begin{split}
I_2 &= 2 \underset{z\in [0,1]}{\sup} \;\underset{(j,k)\in E(z)}{\max} \; 
\sqrt{nh} \cdot  \left[\left\{ \hat{\bTheta}_j(z)\right\}^T\left\{\hat{\bSigma}(z) \tilde{\bTheta}_k - \mathbf{e}_k  \right\} - 
 \left\{ \bTheta_j (z) \right\}^T
\left\{ \hat{\bSigma}(z)  - \bSigma (z) \right\} \bTheta_k (z)
\right] \cdot \left|\PP_n (w_z)\right|\\
&\le 2 \cdot  \underset{z\in [0,1]}{\sup} \; |\PP_n (w_z)|\cdot \left[ I_{21}+I_{221}\right]\\
&\le 2 \cdot \left[ \bar{f}_Z + \cO\left\{ h^2 +\sqrt{\frac{\log (d/h)}{nh}} \right\} \right] \cdot \left( I_{21}+I_{221}\right)\\
&\le C \cdot \left\{ s\cdot \sqrt{nh^9}  + \frac{s \cdot \log (d/h)}{\sqrt{nh}} +  s\cdot h^2\cdot \sqrt{\log (d/h)}\right\},
\end{split}
\end{equation}
where the  third inequality follows from (\ref{Eq:proof:lemma:T0andT5}).

Combining the upper bounds for $I_1$ in (\ref{Eq:proof:lemma:T0andT6}) and $I_2$ in (\ref{Eq:proof:lemma:T0andT10}), we have 
\begin{equation}
\label{Eq:proof:lemma:T0andT11}
|T-T_0| \le C \cdot \left\{ s\cdot \sqrt{nh^9}  + \frac{s \cdot \log (d/h)}{\sqrt{nh}} +  s\cdot h^2\cdot \sqrt{\log (d/h)}\right\},
\end{equation}
with probability at least $1-1/d$.\\

\textbf{Upper bound for $|T_0-T_{00}|$:}
Recall from  (\ref{proof:theorem:gaussian multiplier bootstrap2}) the 
definition of  $T_{00}$
\begin{equation*}
\begin{split}
T_{00}&= \underset{z\in [0,1]}{\sup} \; \underset{(j,k)\in E(z)}{\max} \;
\sqrt{nh} \cdot \bigg|   \sum_{i\in [n]}  \left\{ \bTheta_j (z) \right\}^T
 K_h (Z_i-z) \left\{ \bX_i \bY_i^T  - \bSigma (z) \right\} \bTheta_k (z)/n\\
 &\qquad \qquad \qquad \qquad -    \left\{ \bTheta_j (z) \right\}^T \bigg[ 
 \EE\{K_h (Z-z)\bX \bY^T\} - \EE\{K_h(Z-z)\} \bSigma (z)     \bigg] \bTheta_k (z) /n\bigg|;
\end{split}
\end{equation*}
Using the triangle inequality $| |x| - |y|| \le |x-y|$, we obtain
\begin{equation}
\label{proof:lemma:T0andT001}
\begin{split}
|T_0 - T_{00}| &\le  \sqrt{nh}\cdot \underset{z\in [0,1]}{\sup} \; \underset{(j,k)\in E(z)}{\max} \;  \left| \left\{\bTheta_j (z)\right\}^T\Big[ \EE\{K_h(Z-z) \bX \bY^T\}  - \EE\{K_h(Z-z)\} \bSigma (z)\Big]  \bTheta_k (z) \right|  \\
&\le \sqrt{nh} \cdot   \underset{z\in [0,1]}{\sup} \; \underset{(j,k)\in E(z)}{\max} \;  \|\bTheta_j (z) \|_1\cdot \|\bTheta_k (z) \|_1\cdot   \left|     \EE\{K_h(Z-z) X_j Y_k\}  - \EE\{K_h(Z-z)\}\cdot  \bSigma_{jk} (z)\right|\\
&\le \sqrt{nh} \cdot M^2 \cdot    \underset{z\in [0,1]}{\sup} \; \underset{(j,k) \in E(z)}{\max} \;    \left|     \EE\{K_h(Z-z) X_j Y_k\}  - \EE\{K_h(Z-z)\}\cdot  \bSigma_{jk} (z)\right|\\
&=\sqrt{nh} \cdot M^2 \cdot \left|   f_Z(z) \cdot \bSigma_{jk}(z) + \cO(h^2)   - f_Z(z) \cdot \bSigma_{jk} (z) + \bSigma_{jk} (z) \cdot \cO(h^2)\right|\\
&\le M^2 \cdot M_{\sigma} \cdot  \sqrt{nh^5},
\end{split}
\end{equation}
where the second inequality follows from an application of Holder's inequality, the third inequality follows from the fact that $\bTheta(z) \in \cU_{s,M}$, the first equality follows by an application of Lemma~\ref{lemma:bias}, and the last inequality follows from Assumption~\ref{ass:covariance} and that $h^2 = o(1)$.

Thus, combining  (\ref{Eq:proof:lemma:T0andT11}) and (\ref{proof:lemma:T0andT001}), there exists a constant $C>0$ such that 
\[
|T-T_{00}| \le C \cdot \left\{\sqrt{nh^5} + s\cdot \sqrt{nh^9} + \frac{s \cdot \log (d/h)}{\sqrt{nh}} +\cdot  s\cdot h^2\cdot \sqrt{\log (d/h)} \right\},
\]
with probability at least $1-1/d$.

\subsection{Proof of Lemma~\ref{lemma:T00andW}}
\label{proof:lemma:T00andW}
Recall from (\ref{proof:theorem:gaussian multiplier bootstrap2})
the definition 
\begin{equation*}
\begin{split}
T_{00}&= \underset{z\in [0,1]}{\sup} \; \underset{(j,k)\in E(z)}{\max} \;
\sqrt{nh} \cdot \bigg|   \sum_{i\in [n]}  \left\{ \bTheta_j (z) \right\}^T
 K_h (Z_i-z) \left\{ \bX_i \bY_i^T  - \bSigma (z) \right\} \bTheta_k (z)/n\\
 &\qquad \qquad \qquad \qquad -    \left\{ \bTheta_j (z) \right\}^T \bigg[ 
\EE\{K_h (Z-z)\bX \bY^T\} - \EE\{K_h(Z-z)\} \bSigma (z)    \bigg] \bTheta_k (z) /n\bigg|.
\end{split}
\end{equation*}
Recall from (\ref{EqA:jzjk}) that $J_{z,jk} (Z_i,\bX_i,\bY_i) = J_{z,jk}^{(1)} (Z_i,\bX_{i},\bY_i)  -J_{z,jk}^{(2)} (Z_i)$, where $ J_{z,jk}^{(1)} (Z_i,\bX_{i},\bY_i)$ and  $J_{z,jk}^{(2)} (Z_i)$ are as defined in (\ref{EqA:jzjk1}) and (\ref{EqA:jzjk2}), respectively.  
Let $\cJ = \left\{ J_{z,jk} \mid z\in [0,1],\; j,k\in [d]\right\}$.
Then the intermediate empirical average $T_{00}$ can be written as  
\[
T_{00} = \underset{z\in [0,1]}{\sup}\; \underset{(j,k)\in E(z)}{\max }\;\left|\frac{1}{\sqrt{n}}  \sum_{i\in [n]}J_{z,jk} (Z_i,\bX_i,\bY_i)\right|.
\]

In this section, we show that there exists a Gaussian process $W$ such that 
\[
|T_{00}-W| \le C\cdot \left\{ \frac{\log^6 s \cdot \log^4 (d/h)}{nh}\right\}^{1/8}
\]
with high probability.
To this end, we apply Theorem~A.1 in \citet{chernozhukov2014anti}, which involves the following quantities
\begin{itemize}
\item upper bound for $\underset{z\in [0,1]}{\sup}\; \underset{j,k\in [d]}{\max}  \;\|J_{z,jk} (Z_i,\bX_i,\bY_i)\|_{\infty}$;
\item  upper bound for $\underset{z\in [0,1]}{\sup}\; \underset{j,k\in [d]}{\max}\; \EE \left\{J_{z,jk}^2 (Z,\bX,\bY)\right\}$;
\item  covering number for the function class $\cJ$. 
\end{itemize}
Let $\cS_j(z)$ and $\cS_k(z)$ to be the support of $\bTheta_{j}(z)$ and $\bTheta_k(z)$, respectively.  Note that the cardinality for both sets are less than $s$. We now obtain the above quantities.\\

\textbf{Upper bound for $\underset{z\in[0,1]}{\sup}\;\underset{j,k\in [d]}{\max}  \;\|J_{z,jk} (Z_i,\bX_i,\bY_i)\|_{\infty}$:} 
We have with probability at least $1-1/(2s)$, 
\begin{equation}
\label{proof:lemma:T00andW1}
\begin{split}
&\underset{z\in[0,1]}{\sup}\;\underset{j,k\in [d]}{\max}  \;\|J_{z,jk} (Z_i,\bX_i,\bY_i)\|_{\infty} \\
&\le \sqrt{h} \cdot \underset{z\in[0,1]}{\sup}\;\underset{j,k\in[d]}{\max}  \; \|\bTheta_j (z)\|_1 \cdot \|\bTheta_k (z)\|_1\cdot \left( \underset{j\in \cS_j(z),k\in \cS_k(z)}{\max} \; \|q_{z,jk}\|_{\infty} +   M_{\sigma}\cdot \|k_z\|_{\infty}\right)\\
&\le \sqrt{h} \cdot M^2 \cdot \left\{  \frac{2}{h} \cdot M_X^2 \cdot \|K\|_{\infty} \cdot \log (2s) +  M_{\sigma} \cdot \frac{2}{h} \cdot \|K\|_{\infty}   \right\}\\
&\le \frac{4}{\sqrt{h}} \cdot M^2 \cdot M_X^2 \cdot M_{\sigma} \cdot \|K\|_{\infty} \cdot \log (2s)\\
&= C_1 \cdot \frac{\log s}{\sqrt{h}},
\end{split}
\end{equation}
where the first inequality follows by Holder's inequality and the definition of $q_{z,jk}$ and $k_z$ and  the second inequality follows from (\ref{proof:lemma:empirical process w2}) and (\ref{proof:lemma:empirical process g2}). Note that since we are only taking max over the set $\cS_j(z)$ and $\cS_k(z)$, instead of a $\log d$ factor from (\ref{proof:lemma:empirical process g2}), we obtain a $\log (2s)$ factor.\\

\textbf{Upper bound for $\underset{z\in [0,1]}{\sup}\; \underset{j,k\in [d]}{\max}\; \EE \{J_{z,jk}^2 (Z,\bX,\bY)\}$:}   
By an application of the inequality $(x-y)^2 \le 2x^2 + 2y^2$, we have 
\begin{equation*}
\begin{split}
\underset{z\in [0,1]}{\sup}\; \underset{j,k\in[d]}{\max}\; \EE \left\{J_{z,jk}^2 (Z,\bX,\bY)\right\} &= \underset{z\in [0,1]}{\sup}\; \underset{j,k\in [d]}{\max}\; \EE \left[\left\{ J_{z,jk}^{(1)} (Z,\bX,\bY)-J_{z,jk}^{(2)}  (Z)\right\}^{2}\right]\\
&\le \underbrace{2 \underset{z\in [0,1]}{\sup}\; \underset{j,k\in[d]}{\max}\; \EE \left[\left\{J_{z,jk}^{(1)} (Z,\bX,\bY)\right\}^2  \right]}_{I_1} + \underbrace{2 \underset{z\in [0,1]}{\sup}\; \underset{j,k\in[d]}{\max}\; \EE \left[\left\{J_{z,jk}^{(2)} (Z)\right\}^2  \right]}_{I_2}. 
\end{split}
\end{equation*}
	
	To obtain an upper bound for $I_1$, we need an upper bound for $\underset{z\in [0,1]}{\sup}\; \underset{j,k\in[d]}{\max}\; \EE \{\underset{j\in \cS_j(z),k\in \cS_k(z)}{\max} \;q_{z,jk}^2\} $. 
	Recall from (\ref{EqA:gjk}) the definition of $g_{z,jk} (Z_i,X_{ij},Y_{ik}) = K_h (Z_i-z) X_{ij}Y_{ik}$ and that $q_{z,jk} (Z_i,X_{ij},Y_{ik}) = g_{z,jk}(Z_i,X_{ij},Y_{ik}) - \EE\{g_{z,jk}(Z,X_{j},Y_{k})\}$.
Thus, we have 
\begin{equation}
\label{proof:lemma:T00andW2}
\begin{split}
\underset{z\in [0,1]}{\sup}\; \underset{j,k\in[d]}{\max}\; \EE \left\{ \underset{j\in \cS_j(z),k\in \cS_k(z)}{\max} \;q_{z,jk}^2\right\} &=\underset{z\in [0,1]}{\sup}\; \underset{j,k\in[d]}{\max}\; \EE \left[ \underset{j\in \cS_j(z),k\in \cS_k(z)\in [d]}{\max} \;\left\{g_{z,jk}- \EE(g_{z,jk})\right\}^2\right] \\
&\le 2 \underset{z\in [0,1]}{\sup}\; \underset{j,k\in[d]}{\max}\; \EE \left\{\underset{j\in \cS_j(z),k\in \cS_k(z)\in [d]}{\max}\;g^2_{z,jk}    \right\} + 2 \underset{z\in [0,1]}{\sup}\; \underset{j,k\in[d]}{\max}\; \EE^2(g_{z,jk}),
\end{split}
\end{equation}
where we apply the fact that $(x-y)^2 \le 2x^2 + 2y^2$ to obtain the last inequality.  
By Lemma~\ref{lemma:bias}, we have 
$2 \underset{z\in [0,1]}{\sup}\; \underset{j,k\in[d]}{\max}\; \EE^2(g_{z,jk}) \le 2 \left\{ \bar{f}_Z \cdot M_{\sigma} + \cO(h^2)\right\}^2$.  
Moreover, we have 
\begin{equation*}
\begin{split}
2 \underset{z\in [0,1]}{\sup}\; \underset{j,k\in[d]}{\max}\; \EE \left\{ \underset{j\in \cS_j(z),k\in \cS_k(z)\in [d]}{\max} \;g^2_{z,jk}    \right\} &= 
2 \underset{z\in [0,1]}{\sup}\; \underset{j,k\in[d]}{\max}\; \EE \left\{ \underset{j\in \cS_j(z),k\in \cS_k(z)\in [d]}{\max} \;K^2_h(Z-z) X_j^2 Y_k^2    \right\}\\
&\le 2\cdot  M_X^4 \cdot \log^2 (2s) \underset{z\in [0,1]}{\sup}\; \underset{j,k\in[d]}{\max}\; \EE \left\{ K^2_h(Z-z)  \right\}\\
&\le 2 \cdot M^4_X \cdot \log^2 (2s)\cdot \left\{  \frac{1}{h} \cdot\bar{f}_Z\cdot \|K\|_2^2 + \cO(1) + \cO(h^2) \right\}\\
&\le 3\cdot \bar{f}_Z \cdot \|K\|_2^2 \cdot M_X^4 \cdot \frac{\log^2 (2s)}{h},
\end{split}
\end{equation*}
with probability at least $1-1/(2s)$, where the second inequality follows from an application of Lemma~\ref{lemma:bias}.

 Thus, by Holder's inequality, we have 
\begin{equation}
\label{proof:lemma:T00andW3}
 \begin{split}
 I_1 &\le 2 \cdot h\cdot \underset{z\in [0,1]}{\sup}\; \underset{j,k\in[d]}{\max}\;
 \EE \left[        \left\{ \|\bTheta_j(z)\|_{1}\cdot \|\bTheta_{k}(z)\|_1\cdot \underset{j\in\cS_j(z),k\in \cS_k(z)}{\max} \;|q_{z,jk}|    \right\}^2\right]\\
 &\le 2 \cdot h \cdot M^4 \cdot \underset{z\in [0,1]}{\sup}\; \underset{j,k\in[d]}{\max}\;
\EE \left\{   \underset{j\in \cS_j(z),k\in \cS_k(z)}{\max}\; q^2_{z,jk}    \right\}\\
&\le 2 \cdot h \cdot M^4 \cdot \left[ 3 \cdot \bar{f}_Z \cdot \|K\|_2^2 \cdot M_X^4 \cdot \frac{\log^2 (2s)}{h} + 2 \left\{\bar{f}_Z \cdot M_{\sigma} + \cO(h^2) \right\}^2   \right]\\
&\le 8 \cdot M^4 \cdot \bar{f}_Z \cdot M_X^4 \cdot \|K\|_2^2 \cdot \log^2 (2s),
 \end{split}
 \end{equation}
where the second inequality holds by the fact that $\bTheta(z) \in \cU_{s,M}$.

Similarly, to obtain an upper bound for $I_2$, we use the fact from (\ref{proof:lemma:empirical process w3}) that 
 \begin{equation}
 \label{proof:lemma:T00andW4}
\underset{z\in[0,1]}{\sup}\; \EE\left\{k_z^2\right\}  \le \frac{3}{h} \cdot \bar{f}_Z\cdot  \|K\|_2^2.
 \end{equation}
By Holder's inequality, we have
\begin{equation}
\label{proof:lemma:T00andW5}
 \begin{split}
 I_2 &\le 2 \cdot h\cdot \underset{z\in [0,1]}{\sup}\; \underset{j,k\in[d]}{\max}\;
 \EE \left[        \left\{ \|\bTheta_j(z)\|_{1}\cdot \|\bTheta_{k}(z)\|_1\cdot \underset{(j,k)\in E(z)}{\max} \;|\bSigma_{jk}(z)| \cdot |k_z|    \right\}^2\right]\\
 &\le 2 \cdot h \cdot M^4 \cdot M_{\sigma}^2   \cdot \underset{z\in [0,1]}{\sup}\;
\EE \left(   k^2_{z}    \right)\\
&\le 6\cdot M_{\sigma}^2 \cdot M^4 \cdot \bar{f}_Z \cdot \| K\|_2^2,
 \end{split}
 \end{equation}
where the second inequality holds by Assumption~\ref{ass:covariance} and by the fact that $\bTheta (z)\in \cU_{s,M}$, and the last inequality holds by (\ref{proof:lemma:T00andW4}).

Combining the upper bounds for $I_1$ (\ref{proof:lemma:T00andW3}) and $I_2$ (\ref{proof:lemma:T00andW5}), we have 
\begin{equation}
\label{proof:lemma:T00andW6}
\underset{z\in [0,1]}{\sup}\; \underset{j,k\in [d]}{\max}\; \EE \left\{J_{z,jk}^2 (Z,\bX,\bY)\right\}
\le 8 \cdot M^4 \cdot \bar{f}_Z \cdot \|K\|_2^2 \cdot \left\{M_{\sigma}^2+   M_X^4 \cdot \log^2 (2s) \right\}\\
\le C \cdot \log^2 s = \sigma^2_J,
\end{equation}
for sufficiently large $C>0$.\\

\textbf{Covering number of the function class $\cJ$:} First, we note that the function class $\cJ$ is generated from the addition of two function classes 
\[
\cJ^{(1)}_{jk} = \left\{ J_{z,jk}^{(1)} \mid z\in [0,1]  \right\} \qquad \mathrm{and}\qquad \cJ^{(2)}_{jk} = \left\{ J_{z,jk}^{(2)} \mid z\in [0,1]  \right\}.
\]
Thus, to obtain the covering number of $\cJ$, we first obtain the covering number for the function classes $\cJ^{(1)}_{jk}$ and $\cJ^{(2)}_{jk}$.  Then, we apply Lemma~\ref{lemma:coveringnumber2} to obtain the covering number of the function class $\cJ$.  
From Lemma~\ref{lemma:coveringnumberjzjk1}, we have with probability at least $1-1/d$, 
\[
N\{\cJ^{(1)}_{jk},L_2(Q),\epsilon\} \le C\cdot \left( \frac{d^{5/4} \cdot \log^{3/2} d }{\sqrt{h} \cdot \epsilon}   \right)^6.
\]
Moreover, from Lemma~\ref{lemma:coveringnumberjzjk2}, we have 
\[
N\{\cJ^{(2)}_{jk},L_2(Q),\epsilon\} \le C \cdot \left(\frac{d^{1/6}}{h^{4/3}\cdot \epsilon} \right)^6.
\]
Applying Lemma~\ref{lemma:coveringnumber2} with $a_1 = d^{5/4} \cdot \log^{3/2} d/h^{1/2}$, $v_1=6$, $a_2 = d^{1/6}/h^{4/3}$, and $v_2=6$,  we have 
\begin{equation}
\label{proof:lemma:T00andW7}
N\{\cJ,L_2(Q),\epsilon\} \le C\cdot d^2\cdot \left( \frac{d^{17/24} \cdot \log^{3/4} d }{h^{11/12} \cdot \epsilon}   \right)^{12},
\end{equation}
where we multiply $d^2$ on the right hand side since the function class $\cJ$ is taken over all $j,k\in [d]$.\\

\textbf{Application of Theorem A.1 in \citet{chernozhukov2014anti}:} Applying Theorem A.1 in \citet{chernozhukov2014anti} with $a= d^{65/24}\cdot \log^{7/4}d/h^{17/12}  $, $b= C\cdot \log s /\sqrt{h}$, $\sigma_J =C\cdot  \log s$, and 
\[
K_n = A \cdot \left\{\log n \vee \log (ab/\sigma_J) \right\} = C\cdot  \log(d/h) ,
\]
for sufficiently large constant $A,C>0$, there exists a random process $W$ such that for any $\gamma \in (0,1)$, 
\[
P\left[|T_{00}-W| \ge C\cdot \left\{ \frac{b K_n}{(\gamma n)^{1/2}}
+   \frac{(b\sigma_J)^{1/2} K_n^{3/4}}{\gamma^{1/2}n^{1/4}} + \frac{b^{1/3} \sigma_J^{2/3}K_n^{2/3}}{\gamma^{1/3}n^{1/6}} \right\}   \right] \le C'\cdot \left( \gamma + \frac{\log n}{n}\right)
\]
for some absolute constant $C'$.
Picking $\gamma = \left\{\log^{6} s \cdot \log^{4}(d/h)/(nh)\right\}^{1/8}$, we have 
\[
P\left[|T_{00}-W| \ge C\cdot   \left\{\frac{\log^{6} s \cdot \log^4 (d/h)}{nh}\right\}^{1/8} \right]  \le C'\cdot  \left\{\frac{\log^{6} s \cdot \log^4 (d/h)}{nh}\right\}^{1/8},
\]
as desired.

\subsection{Proof of Lemma~\ref{lemma:T00BandW}}
\label{proof:lemma:T00BandW}
Recall from the proof of Lemma~\ref{lemma:T00andW} that \[
T_{00} = \underset{z\in [0,1]}{\sup}\; \underset{(j,k)\in E(z)}{\max}\; \left|\frac{1}{\sqrt{n}} \sum_{i\in [n]}J_{z,jk} (Z_i,\bX_i,\bY_i)\right|.
\]
We note that 
\[
T_{00}^B  =\underset{z\in [0,1]}{\sup}\; \underset{(j,k)\in E(z)}{\max}\; \left|\frac{1}{\sqrt{n}} \sum_{i\in [n]}J_{z,jk} (Z_i,\bX_i,\bY_i)\cdot \xi_i\right|,
\]
 where $\xi_i\stackrel{\mathrm{i.i.d.}}{\sim}  N(0,1)$.
To show that the term $|W-T_{00}^B|$ can be controlled, we apply Theorem A.2 in \citet{chernozhukov2014anti}.  

Let 
\[
\psi_n =\sqrt{\frac{\sigma_J^2  K_n}{n}} + \left( \frac{b^2 \sigma_J^2 K_n^3}{n} \right)^{1/4} \qquad \mathrm{and} \qquad \gamma_n (\delta) = \frac{1}{\delta} \left( \frac{b^2 \sigma_J^2 K_n^3}{n} \right)^{1/4} + \frac{1}{n}  ,
\]
as defined in Theorem A.2 in \citet{chernozhukov2014anti}. From the proof of Lemma~\ref{lemma:T00andW}, we have $b= C\cdot \log s/\sqrt{h}$, $K_n = C\cdot \log(d/h)$, and $\sigma_J= C \cdot \log s$. 
Since $b^2 K_n = C \cdot\log^2 s \cdot \log (d/h)/h \le n\cdot \log^2 s$ for sufficiently large $n$, there exists a constant $C''>0$ such that 
\[
P\left[ |T_{00}^B- W| > \psi_n + \delta \; \Big|\;\{ (Z_i,\bX_i,\bY_i)\}_{i\in [n]}  \right] \le C'' \cdot \gamma_n (\delta),
\]
with probability at least $1-3/n$.  Choosing $\delta = \left\{ \log^4 (s) \cdot \log^3 (d/h)/(nh)  \right\}^{1/8}$, we have 
\[
P\left[ |T_{00}^B- W| > C\cdot  \left\{ \frac{\log^4 (s) \cdot \log^3 (d/h)}{nh}  \right\}^{1/8 } \; \;\Big|\; \{(Z_i,\bX_i,\bY_i)\}_{i\in[n]} \right] \le C''\cdot  \left\{ \frac{\log^4 (s) \cdot \log^3 (d/h)}{nh}  \right\}^{1/8},
\]
with probability at least $1-3/n$.

\subsection{Proof of Lemma~\ref{lemma:T00BandTB}}
\label{proof:lemma:T00BandTB}
In this section, we show that $|T^B-T^B_{00}|$ is upper bounded by the quantity 
\[
C\cdot \left\{ s \cdot \sqrt{h^3 \log^3 (d/h)} + s \cdot \sqrt{\frac{\log^4 (d/h)}{nh^2}} + \sqrt{h^5 \log n}\right\}
\]
with high probability for sufficiently large constant $C>0$.  
Throughout the proof of this lemma, we conditioned on the data $\{(Z_i,\bX_i,\bY_i)\}_{i\in [n]}$.  
By the triangle inequality, we have $|T^B - T_{00}^B|  \le |T^B - T_{0}^B|  + |T_0^B - T_{00}^B|  $.  Thus, it suffices to obtain upper bounds for the terms $|T^B - T_{0}^B|$ and  $|T_0^B - T_{00}^B|$.  \\

\textbf{Upper bound for $|T^B-T_{0}^B|$:} 
Recall from (\ref{proof:theorem:gaussian multiplier bootstrap TB}) and  (\ref{proof:theorem:gaussian multiplier bootstrap3})  that 
\begin{equation*}
T^B =  \underset{z\in [0,1]}{\sup} \; \underset{(j,k) \in E(z) }{\max} \;\sqrt{nh}\cdot \left| \frac{ \sum_{i\in [n]} \left\{\hat{\bTheta}_j (z)\right\}^T K_h (Z_i-z) \left\{ \bX_i \bY_i^T\hat{\bTheta}_k (z)  - \mathbf{e}_k\right\} \xi_i/n}{ \left\{ \hat{\bTheta}_j(z)\right\}^T \hat{\bSigma}_j(z)}\right|,
\end{equation*}
and that 
\begin{equation*}
T_0^B = \underset{z\in [0,1]}{\sup} \; \underset{(j,k)\in E(z)}{\max} \;
\sqrt{nh} \cdot \left| \sum_{i\in [n]} \left[ \left\{ \bTheta_j (z) \right\}^T
 K_h (Z_i-z) \left\{ \bX_i \bY_i^T  - \bSigma (z) \right\} \bTheta_k (z) \right]\xi _i/n\right|,
\end{equation*}
respectively.
Using the triangle inequality, we have 
\begin{equation}
\label{Eq:proof:lemma:T00BandT0B1}
\small
\begin{split}
|T^B-T_0^B| &\le  \sqrt{nh} \cdot  \Bigg|  \underset{z\in [0,1]}{\sup} \; \underset{(j,k)\in E(z)}{\max} \;
 \Bigg[ \frac{1}{n}\sum_{i\in [n]}  \left\{\hat{\bTheta}_j (z)\right\}^T K_h (Z_i-z) \left\{ \bX_i \bY_i^T\hat{\bTheta}_k (z)  - \mathbf{e}_k\right\}/ \left\{\hat{\bTheta}_j(z)\right\}^T \hat{\bSigma}_j(z)  \\
&\qquad\quad \qquad \qquad \qquad \qquad- \frac{1}{n} \sum_{i\in [n]}  \left\{ \bTheta_j (z) \right\}^T
 K_h (Z_i-z) \left\{ \bX_i \bY_i^T  - \bSigma (z) \right\} \bTheta_k (z) \Bigg]\xi _i\Bigg|\\
 &\le 2\underbrace{ \sqrt{nh} \cdot     \left|   \underset{z\in [0,1]}{\sup} \; \underset{(j,k)\in E(z)}{\max} \;\frac{1}{n} \sum_{i\in [n]} \left\{ \hat{\bTheta}_j (z)-\bTheta_j (z) \right\}^T K_h (Z_i-z) \left\{\bX_i\bY_i^T -\bSigma (z)  \right\} \bTheta_k (z) \xi_i     \right|}_{I_1} \\
&+  2\underbrace{ \sqrt{nh}  \cdot  \left|  \underset{z\in [0,1]}{\sup} \; \underset{(j,k)\in E(z)}{\max} \; \frac{1}{n} \sum_{i\in [n]} \left\{\bTheta_j (z) \right\}^T K_h (Z_i-z) \bX_i\bY_i^T  \left( \hat{\bTheta}_k (z)-   \bTheta_k (z)\right) \xi_i     \right|}_{I_2}\\
&+  2\underbrace{ \sqrt{nh}  \cdot  \left|  \underset{z\in [0,1]}{\sup} \; \underset{(j,k)\in E(z)}{\max} \; \frac{1}{n} \sum_{i\in [n]}  \left\{ \bTheta_j (z) \right\}^T
 K_h (Z_i-z) \left\{ \bX_i \bY_i^T  - \bSigma (z) \right\} \bTheta_k (z) \xi_i     \right| \cdot \left|\left\{\hat{\bTheta}_j(z) \right\}^T \hat{\bSigma}_j(z)-1\right|}_{I_3}, 
\end{split}
\end{equation}
where the  second inequality holds by another application of the triangle inequality and inequality in \eqref{Eq:proof:lemma:T0andT2}.  We now obtain upper bounds for $I_1$, $I_2$, and $I_3$.\\

\textbf{Upper bound for $I_1$:} 
By an application of Holder's inequality, we have 
\begin{equation}
\label{Eq:proof:lemma:T00BandT0B2}
\small
\begin{split}
I_1 &\le \underset{z\in[0,1]}{\sup}\;  \underset{j,k\in[d]}{\max} \;
 \left\| \hat{\bTheta}_j(z) - \bTheta_j (z)  \right\|_1\cdot \|\bTheta_k (z) \|_1 
\cdot \sqrt{nh}\cdot  \left| \underset{z\in [0,1]}{\sup} \; \underset{j,k \in[d]}{\max} \; \frac{1}{n}  \sum_{i\in[n]} \left\{ K_h (Z_i-z)X_{ij}Y_{ik}- K_h(Z_i-z)\bSigma_{jk}(z)   \right\}\xi_i   \right|\\
&\le M \cdot C\cdot s\cdot \left\{ h^2 + \sqrt{\frac{\log (d/h)}{nh}}\right\}\cdot \sqrt{nh}\cdot   \left| \underset{z\in [0,1]}{\sup} \; \underset{j,k \in [d]}{\max} \;\frac{1}{n}  \sum_{i\in[n]} \left\{ K_h (Z_i-z)X_{ij}Y_{ik}- K_h(Z_i-z)\bSigma_{jk}(z)   \right\}\xi_i   \right|,
\end{split}
\end{equation}
where the last inequality follows from the fact that $\bTheta(z)\in \cU_{s,M}$ and by an application of Corollary~\ref{theorem:inverse estimation}.  For notational convenience, we use the notation as defined in (\ref{EqA:Wzjk})
\begin{equation}
\label{Eq:proof:lemma:T00BandT0B3}
W_{z,jk} (Z_i,X_{ij},Y_{ik}) = \sqrt{h} \cdot \left\{K_h (Z_i-z)X_{ij}Y_{ik}- K_h(Z_i-z)\bSigma_{jk}(z)   \right\}.
\end{equation}
Then, we have 
\begin{equation*}
\begin{split}
 \sqrt{\frac{h}{n}}  \sum_{i\in[n]} \left\{ K_h (Z_i-z)X_{ij}Y_{ik}- K_h(Z_i-z)\bSigma_{jk}(z)   \right\}\xi_i = \frac{1}{\sqrt{n}} \sum_{i\in[n]} W_{z,jk}  (Z_i,X_{ij},Y_{ik}) \cdot \xi_i.
 \end{split}
\end{equation*}
We note that conditioned on the data $\{(Z_i,\bX_i,\bY_i)\}_{i\in[n]}$, the above expression is a Gaussian process.
It remains to bound the supreme of 
the Gaussian process 
\[
\frac{1}{\sqrt{n}} \sum_{i\in [n]} W_{z,jk}  (Z_i,X_{ij},Y_{ik}) \cdot \xi_i
\sim N\left\{0,\frac{1}{n}\sum_{i\in[n]} W_{z,jk}^2 (Z_i,X_{ij},Y_{ik})\right\}
\]
in probability.

To this end, we apply the Dudley's inequality (see, e.g., Corollary 2.2.8 in \citealp{van1996weak}) and the Borell's inequality (see, e.g., Proposition A.2.1 in \citealp{van1996weak}), which involves  the following quantities:  

\begin{itemize}
\item upper bound on the conditional variance  $\sum_{i\in[n]}W_{z,jk}^2 (Z_i,X_{ij},Y_{ik})/n;$
\item the covering number of the function class 
\[
\cW = \left\{ W_{z,jk} (\cdot) \mid z\in [0,1],\; j,k\in [d]  \right\}
\]
under the $L_2$ norm on the empirical measure.
\end{itemize}

\textbf{Upper bound for the conditional variance $\sum_{i=1}^n W_{z,jk}^2 (Z_i,X_{ij},Y_{ik})/n:$
}
By the definition of $W_{z,jk} (Z_i,X_{ij},Y_{ik})$ in (\ref{Eq:proof:lemma:T00BandT0B3}), we have 
\begin{equation}
\label{Eq:proof:lemma:T00BandT0B4}
\begin{split}
\frac{1}{n}\sum_{i=1}^n W_{z,jk}^2 (Z_i,X_{ij},Y_{ik}) 
&= \frac{h}{n} \cdot \sum_{i\in [n]} \left\{ K_h(Z_i-z) X_{ij}Y_{ik} - K_h(Z_i-z) \bSigma_{jk} (z)  \right\}^2\\
&\le h \cdot \underset{i\in[n]}{\max} \; \left\{ K_h(Z_i-z) X_{ij}Y_{ik} - K_h(Z_i-z) \bSigma_{jk} (z)  \right\}^2\\
&\le 2h \cdot \underset{i\in[n]}{\max} \; \left\{K_h^2(Z_i-z) X_{ij}^2 Y_{ik}^2+ K_h^2 (Z_i-z)\bSigma^2_{jk}(z)   \right\}\\
&\le 2h\cdot \left(\frac{1}{h^2} \cdot \|K\|_{\infty}^2 \cdot M_X^4 \cdot\log^2 d  + \frac{1}{h^2} \cdot \|K\|_{\infty}^2 \cdot  M_{\sigma}^2  \right)\\
&\le C\cdot \frac{\log^2 d}{h},
\end{split}
\end{equation}
with probability at least $1-1/d$.  Note that the second inequality holds by the fact that $(x-y)^2 \le 2x^2 + 2y^2$, and the third inequality holds by \eqref{prop:kernel} and Assumption~\ref{ass:covariance}, and the fact that $\max (X_{ij},Y_{ij}) \le M_X \cdot \sqrt{\log d}$ with probability at least $1-1/d$.\\

\textbf{Covering number of the function class $\cW$:} To obtain the covering number of the function class $\cW$ under the $L_2$ norm on the empirical measure, it suffices to obtain the covering number $\underset{Q}{\sup} \; N\{\cW,L_2(Q),\epsilon\} $.
First, we note that $W_{z,jk}= \sqrt{h} \cdot \left\{ g_{z,jk} - w_z \cdot \bSigma_{jk} (z)\right\}$.  
From Lemma~\ref{lemma:coveringnumberkz}, we have  $\cK_1 = \{w_z(\cdot) \mid z\in [0,1]\}$ and that 
\[
\underset{Q}{\sup} \; \cN\{\cK_1,L_2(Q),\epsilon\} \le \left( \frac{2\cdot C_K\cdot \|K\|_{\mathrm{TV}}}{h\epsilon}\right)^4.
\]
Also, From Lemma~\ref{lemma:coveringnumbergzjk}, we have  $\cG_{1,jk} = \{g_{z,jk} (\cdot) \mid z\in [0,1] \}$ and that 
\[
\underset{Q}{\sup} \;  N\{\cG_{1,jk}, L_2(Q),   \epsilon\} \le   \left( \frac{2\cdot M_X^2 \cdot \log d\cdot  C_K \cdot \|K\|_{\mathrm{TV}} }{h\epsilon}\right)^4.
\]
Moreover, by Assumption~\ref{ass:covariance}, $\bSigma_{jk}(z)$ is $M_{\sigma}$-Lipschitz. 
Thus, applying Lemmas~\ref{lemma:coveringnumber1} and \ref{lemma:coveringnumber2}, we obtain
\begin{equation}
\label{Eq:proof:lemma:T00BandT0B5}
\underset{Q}{\sup}\; N\{\cW,L_2 (Q),\epsilon\} \le  2^{22} \cdot M_{\sigma}  \cdot M_{X}^8\cdot C_K^8 \cdot \|K\|_{\mathrm{TV}}^8 \cdot \|K\|_{\infty}^5 \cdot d^2\cdot \left( \frac{\log^{4/9} d}{h^{17/18} \epsilon}\right)^{9} = C\cdot d^2 \cdot  \left( \frac{\log^{4/9}d }{h^{17/18} \epsilon}\right)^{9},
\end{equation}
where the term $d^2$ appear on the right hand side because the function class $\cW$ is over  $j,k\in[d]$.\\

\textbf{Applying Dudley's inequality and Borell's inequality:}
Applying Dudley's inequality (see Corollary 2.2.8 in \citealp{van1996weak}) with (\ref{Eq:proof:lemma:T00BandT0B4}) and (\ref{Eq:proof:lemma:T00BandT0B5}), we have 
\[
\EE \left\{\underset{z\in [0,1]}{\sup}\; \underset{j,k\in[d] }{\max}\;    \frac{1}{\sqrt{n}} \sum_{i\in [n]} W_{z,jk}  (Z_i,X_{ij},Y_{ik}) \cdot \xi_i\right\} \le  C \cdot \int_{0}^{C\cdot \sqrt{ \frac{\log^2 d}{h}}}
\sqrt{ \log \left(  \frac{d^{2/9}\cdot \log^{4/9}d}{h^{17/18} \epsilon}  \right)      }   d\epsilon.  
\]
Applying (\ref{EqA:covering number dudley}) with $b_1 = C\cdot\sqrt{ \log^2d/h}$ and $b_2 = d^{2/9} \cdot \log^{4/9} d/h^{17/18}$, we have 
\begin{equation}
\label{Eq:proof:lemma:T00BandT0B6}
\EE \left\{\underset{z\in [0,1]}{\sup}\; \underset{j,k\in [d]}{\max}\;    \frac{1}{\sqrt{n}} \sum_{i\in [n]} W_{z,jk}  (Z_i,X_{ij},Y_{ik}) \cdot \xi_i\right\} \le  C \cdot \sqrt{\frac{\log^3 (d/h)}{h}},
\end{equation}
for some sufficiently large $C>0$.

By Borell's inequality (see Proposition A.2.1 in \citealp{van1996weak}), for $\lambda>0$, we have 
\begin{equation*}
\small
\begin{split}
&P\left[ \left| \underset{z\in [0,1]}{\sup}\; \underset{j,k\in E(z)}{\max}\;    \frac{1}{\sqrt{n}} \sum_{i\in [n]} W_{z,jk}  (Z_i,X_{ij},Y_{ik}) \cdot \xi_i    
\right|  \ge C\cdot \sqrt{\frac{\log^3(d/h)}{h}} +\lambda    \;\Bigg|\; \{(Z_i,\bX_i,\bY_{i})\}_{i\in [n]}\right] \\
&\le 2\cdot \exp \left(  -\frac{\lambda^2}{2 \sigma_X^2} \right),
\end{split}
\end{equation*}
where $\sigma^2_X$ is the upper bound on the conditional variance.  
Picking $\lambda = C \cdot \sqrt{\frac{\log^3 (d/h)}{h}}$, we have 
\begin{equation}
\label{Eq:proof:lemma:T00BandT0B7}
P\left[ \left| \underset{z\in [0,1]}{\sup}\; \underset{j,k\in [d]}{\max}\;    \frac{1}{\sqrt{n}} \sum_{i\in [n]} W_{z,jk}  (Z_i,X_{ij},Y_{ik}) \cdot \xi_i    
\right|  \ge C\cdot \sqrt{\frac{\log^3(d/h)}{h}}    \;\Bigg|\; \{(Z_i,\bX_i,\bY_i)\}_{i\in [n]}\right]\le \frac{1}{d}.
\end{equation}
Thus, substituting (\ref{Eq:proof:lemma:T00BandT0B7}) into (\ref{Eq:proof:lemma:T00BandT0B2}), we have 
\begin{equation}
\label{Eq:proof:lemma:T00BandT0B7-5}
\begin{split}
I_1 &\le C\cdot M \cdot s \cdot \left\{ h^2 + \sqrt{\frac{\log (d/h)}{nh}}\right\} \cdot \sqrt{\frac{\log^3(d/h)}{h}}\\
&\le C\cdot s \cdot \sqrt{h^3 \log^3 (d/h)} + C \cdot s \cdot \sqrt{\frac{\log^4 (d/h)}{nh^2}},
\end{split}
\end{equation}
with probability $1-1/d$.\\

\textbf{Upper bound for $I_2$:} 
By an application of Holder's inequality, we have 
\begin{equation}
\label{Eq:proof:lemma:T00BandT0B8}
\small
\begin{split}
I_2 &\le\sqrt{nh}\cdot   \underset{z\in[0,1]}{\sup} \; \underset{j,k\in[d]}{\max} \;
 \left\| \hat{\bTheta}_j(z) - \bTheta_j (z)  \right\|_1\cdot \left\|\hat{\bTheta}_k (z) \right\|_1 
\cdot  \left| \underset{z\in [0,1]}{\sup} \; \underset{j,k \in [d]}{\max} \; \frac{1}{n}  \sum_{i\in[n]} \left\{ K_h (Z_i-z)X_{ij}Y_{ik}  \right\}\xi_i   \right|\\
&\le  \underset{z\in [0,1]}{\sup} \;  \underset{j,k \in [d]}{\max} \;  \left[ \left\| \hat{\bTheta}_k(z) - \bTheta_k (z)  \right\|_1 + \|\bTheta_k (z)\|_1  \right] \cdot C\cdot s\cdot \left\{ h^2 + \sqrt{\frac{\log (d/h)}{nh}}\right\} \\
& \quad \times  \sqrt{nh}\cdot  \left|\underset{z\in [0,1]}{\sup} \;  \underset{j,k \in [d]}{\max} \; \frac{1}{n}  \sum_{i\in[n]} \left\{ K_h (Z_i-z)X_{ij}Y_{ik}  \right\}\xi_i   \right|\\
&\le C \cdot M \cdot s\cdot \left\{ h^2 + \sqrt{\frac{\log (d/h)}{nh}}\right\} \cdot \sqrt{nh}\cdot  \left| \underset{z\in [0,1]}{\sup} \; \underset{j,k \in [d]}{\max} \;\frac{1}{n}  \sum_{i\in[n]} \left\{ K_h (Z_i-z)X_{ij}Y_{ik}  \right\}\xi_i   \right|,
\end{split}
\end{equation}
where the second inequality holds by triangle inequality and Corollary~\ref{theorem:inverse estimation}, and the last inequality holds by another application of Corollary~\ref{theorem:inverse estimation} and the assumption that $h^2 + \sqrt{\log(d/h)/(nh)}= o(1)$.

Recall the definition of $g_{z,jk}(Z_i,X_{ij},Y_{ik}) = K_h (Z_i-z)X_{ij}Y_{ik} $.  Conditioned on the data $\{(Z_i,\bX_i,\bY_i)   \}_{i\in [n]}$, we note that 
\[
\sqrt{\frac{h}{n}}  \sum_{i\in[n]} \left\{ K_h (Z_i-z)X_{ij}Y_{ik}  \right\}\xi_i  = \frac{1}{\sqrt{n}} \sum_{i\in [n]} \sqrt{h} \cdot g_{z,jk}(Z_i,X_{ij},Y_{ik})\cdot  \xi_i \sim N\left\{0, \frac{h}{n} \sum_{i\in [n]} g_{z,jk}^2 (Z_i,X_{ij},Y_{ik}) \right\}.
\]  
Similar to the upper bound for $I_1$, we apply Dudley's inequality and Borell's inequality to bound the supreme of the Gaussian process in the last expression.  

To this end, 
we need to obtain an upper bound for the conditional covariance. By (\ref{proof:lemma:empirical process g2}), we have 
\begin{equation}
\label{Eq:proof:lemma:T00BandT0B9}
\frac{h}{n}\sum_{i\in [n]} g^2_{z,jk}(Z_i,X_{ij},Y_{ik}) \le \frac{1}{h}
 \cdot M_X^4 \cdot \|K\|_{\infty}^4 \cdot \log^2 d,
 \end{equation}
with probability at least $1-1/d$.
In addition, by an application of Lemma~\ref{lemma:coveringnumbergzjk},  the covering number for the class of function $ \{\sqrt{h} \cdot g_{z,jk} (\cdot) \; | \; z\in [0,1],\; j,k\in [d] \}$ is 
\begin{equation}
\label{Eq:proof:lemma:T00BandT0B10}
\underset{Q}{\sup} \;  N\left[\left\{\sqrt{h} \cdot g_{z,jk} (\cdot) \; | \; z\in [0,1],\; j,k\in [d] \right\}, L_2(Q),  \epsilon\right] \le  d^2\cdot  \left( \frac{2\cdot M_X^2 \cdot \log d\cdot  C_K \cdot \|K\|_{\mathrm{TV}} }{\sqrt{h}\epsilon}\right)^4.
\end{equation}

By an application of Dudley's inequality, we have 
\[
\EE \left\{\underset{z\in [0,1]}{\sup}\; \underset{j,k\in [d]}{\max}\;    \frac{1}{\sqrt{n}} \sum_{i\in [n]} \sqrt{h} \cdot g_{z,jk}  (Z_i,X_{ij},Y_{ik}) \cdot \xi_i\right\} \le  C \cdot \int_{0}^{\sqrt{M_X^4 \cdot \|K\|_{\infty}^4 \cdot \frac{\log^2 d}{h}}}
\sqrt{ \log \left(  \frac{d^{1/2}\cdot \log d}{h^{1/2} \epsilon}  \right)      }   d\epsilon.  
\]
Applying (\ref{EqA:covering number dudley}) with $b_1 = \sqrt{M_X^4 \cdot \|K\|_{\infty}^4 \cdot \log^2d/h}$ and $b_2 = d^{1/2} \cdot \log d/h^{1/2} $, we have 
\begin{equation}
\label{Eq:proof:lemma:T00BandT0B10}
\EE \left\{\underset{z\in [0,1]}{\sup}\; \underset{j,k\in [d]}{\max}\;    \frac{1}{\sqrt{n}} \sum_{i\in [n]} \sqrt{h} \cdot g_{z,jk}  (Z_i,X_{ij},Y_{ik}) \cdot \xi_i\right\} \le  C \cdot \sqrt{\frac{\log^3 (d/h)}{h}}.
\end{equation}

By Borell's inequality (see Proposition A.2.1 in \citealp{van1996weak}), we have 
\begin{equation*}
\small
\begin{split}
&P\left[ \left| \underset{z\in [0,1]}{\sup}\; \underset{j,k\in [d]}{\max}\;    \frac{1}{\sqrt{n}} \sum_{i\in [n]} \sqrt{h}\cdot g_{z,jk}  (Z_i,X_{ij},Y_{ik}) \cdot \xi_i    
\right|  \ge C\cdot \sqrt{\frac{\log^3(d/h)}{h}} +\lambda    \;\Bigg|\; \{(Z_i,\bX_i,\bY_i)\}_{i\in [n]}\right] \\
&\le 2\cdot \exp \left(  -\frac{\lambda^2}{2 \sigma_X^2} \right).
\end{split}
\end{equation*}
Picking $\lambda = C \cdot \sqrt{\frac{\log^3 (d/h)}{h}}$, we have 
\begin{equation}
\label{Eq:proof:lemma:T00BandT0B11}
P\left[ \left| \underset{z\in [0,1]}{\sup}\; \underset{j,k\in [d]}{\max}\;    \frac{1}{\sqrt{n}} \sum_{i\in [n]}\sqrt{h}\cdot g_{z,jk}  (Z_i,X_{ij},Y_{ik}) \cdot \xi_i    
\right|  \ge C\cdot \sqrt{\frac{\log^3(d/h)}{h}}    \;\Bigg|\; \{(Z_i,\bX_i,\bY_i)\}_{i\in [n]}\right]\le \frac{1}{d}. 
\end{equation}

Thus, by (\ref{Eq:proof:lemma:T00BandT0B8}) and (\ref{Eq:proof:lemma:T00BandT0B11}), we have 
\begin{equation}
\label{Eq:proof:lemma:T00BandT0B12}
\begin{split}
I_2 &\le C \cdot M \cdot s \cdot \left\{ h^2 + \sqrt{\frac{\log (d/h)}{nh}}\right\}\cdot \sqrt{\frac{\log^3 (d/h)}{h}}\\
&\le C\cdot s \cdot \sqrt{h^3 \log^3 (d/h)} + C \cdot s \cdot \sqrt{\frac{\log^4 (d/h)}{nh^2}},
\end{split}
\end{equation}
with probability at least $1-1/d$.

\textbf{Upper bound for $I_3$:} By an application of Holder's inequality, we have
\begin{equation}
\small
\label{i3holderadd}
\begin{split}
I_3 &\le   \underset{z\in[0,1]}{\sup} \; \underset{j\in[d]}{\max} \;
\left\|   \left\{\hat{\bTheta}_j(z)\right\}^T \hat{\bSigma}(z)-\mathbf{e}_j       \right\|_{\infty} \left\|\hat{\bTheta}_k (z) \right\|_1^2 
\cdot \sqrt{nh}\cdot  \left| \underset{z\in [0,1]}{\sup} \; \underset{j,k \in[d]}{\max} \; \frac{1}{n}  \sum_{i\in[n]} \left\{ K_h (Z_i-z)X_{ij}Y_{ik}- K_h(Z_i-z)\bSigma_{jk}(z)   \right\}\xi_i   \right|\\\\
&\le M^3 \cdot C \cdot s \cdot \left\{h^2 + \sqrt{\frac{\log(d/h)}{nh}}\right\}  \sqrt{nh}\cdot  \left| \underset{z\in [0,1]}{\sup} \; \underset{j,k \in[d]}{\max} \; \frac{1}{n}  \sum_{i\in[n]} \left\{ K_h (Z_i-z)X_{ij}Y_{ik}- K_h(Z_i-z)\bSigma_{jk}(z)   \right\}\xi_i   \right|\\\\
&\le C\cdot M^3 \cdot s \cdot \left\{ h^2 +\sqrt{\frac{\log(d/h)}{nh}}\right\}\cdot \sqrt{\frac{\log^3 (d/h)}{h}}\\
&\le C\cdot s \cdot \sqrt{h^3 \log^3 (d/h)} + C\cdot s\cdot \sqrt{\frac{\log^4 (d/h)}{nh^2}},
\end{split}
\end{equation}
where the second inequality holds by the fact that $\bTheta(z)\in \cU_{s,M}$ and by an application of Corollary~\ref{theorem:inverse estimation}, and the third inequality holds by \eqref{Eq:proof:lemma:T00BandT0B7}.

Thus, combining (\ref{Eq:proof:lemma:T00BandT0B7-5}), (\ref{Eq:proof:lemma:T00BandT0B12}), and \eqref{i3holderadd}, we have 
\begin{equation}
\label{Eq:proof:lemma:T00BandT0B12-5}
|T^B-T_0^B| \le C\cdot s \cdot \sqrt{h^3 \log^3 (d/h)} + C \cdot s \cdot \sqrt{\frac{\log^4 (d/h)}{nh^2}}
\end{equation}
with probability at least $1-3/d$.\\

\textbf{Upper bound for $|T_0^B-T_{00}^B|$:} 
Recall from  (\ref{proof:theorem:gaussian multiplier bootstrap4}) that 
\begin{equation*}
\begin{split}
T_{00}^B &=   \underset{z\in [0,1]}{\sup} \; \underset{(j,k)\in E(z)}{\max} \;
\sqrt{nh} \cdot \bigg|   \sum_{i\in [n]}  \Big(\left\{ \bTheta_j (z) \right\}^T
 K_h (Z_i-z) \left\{ \bX_i \bY_i^T  - \bSigma (z) \right\} \bTheta_k (z)/n\\
 &\qquad \qquad \qquad \qquad -    \left\{ \bTheta_j (z) \right\}^T \Big[ 
 \EE\{K_h (Z-z)\bX \bY^T\} - \EE\{K_h(Z-z)\} \bSigma (z)  \Big] \bTheta_k (z)  \Big) \cdot \xi_i/n\bigg|.
\end{split}
\end{equation*}
By the triangle inequality, we have 
\begin{equation}
\label{Eq:T0BandT00B1}
\small
\begin{split}
|T_0^B-T_{00}^B| &\le \sqrt{nh} \cdot \underset{z\in [0,1]}{\sup} \; \underset{(j,k)\in E(z)}{\max}\; 
\left| \frac{1}{n}\sum_{i\in [n]} \left( \left\{ \bTheta_j(z)\right\}^T \left[  
\EE\left\{K_h(Z_i-z)\bX_i\bY_i^T\right\} - \EE\{K_h(Z_i-z)\} \bSigma(z)
 \right] \bTheta_k(z)         \right) \cdot \xi_i\right|\\
 &\le  \sqrt{nh} \cdot \underset{z\in [0,1]}{\sup} \; \underset{(j,k)\in E(z)}{\max}\; \left| \left\{\bTheta_j (z)\right\}^T \left[     
\EE\left\{K_h(Z-z)\bX\bY^T\right\} - \EE\{K_h(Z-z)\} \bSigma(z)
    \right] \bTheta_k (z)\right| \cdot \left| \frac{1}{n}\sum_{i\in [n]} \xi_i  \right|\\
    &\le \sqrt{nh} \cdot M^2 \cdot C \cdot h^2 \cdot \left| \frac{1}{n}\sum_{i\in [n]} \xi_i  \right|,
\end{split}
\end{equation}
where the last inequality holds by applying Holder's inequality and Lemma~\ref{lemma:bias}.
Since $\xi_i \stackrel{\mathrm{i.i.d.}}{\sim} N(0,1)$, by the Gaussian tail inequality, we have
\[
P\left(\left| \frac{1}{n}\sum_{i\in [n]} \xi_i\right| > \sqrt{\frac{2\log n}{n}}    \right) \le  \frac{1}{n}. 
\]
Thus, substituting the above expression into (\ref{Eq:T0BandT00B1}), we obtain 
\begin{equation}
\label{Eq:T0BandT00B2}
|T_0^B-T_{00}^B| \le \sqrt{nh} \cdot M^2 \cdot C\cdot h^2 \cdot \sqrt{\frac{2\log n}{n}}  \le C \cdot \sqrt{h^5 \log n},
\end{equation}
with probability at least $1-1/n$.\\

\textbf{Combining the upper bounds:} Combining the upper bounds (\ref{Eq:proof:lemma:T00BandT0B12-5}) and (\ref{Eq:T0BandT00B2}), and applying the union bound, we have 
\begin{equation*}
\small
\begin{split}
&P\left[\left| T^B-T_{00}^B\right| \ge C\cdot s \cdot \sqrt{h^3 \log^3 (d/h)} + C \cdot s \cdot \sqrt{\frac{\log^4 (d/h)}{nh^2}}  + C \cdot \sqrt{h^5 \log n}      \;\Bigg|\; \{(Z_i,\bX_i,\bY_i)\}_{i\in [n]} \right]\\
&\le P\left[\left| T^B-T_{0}^B\right| + \left|T_0^B-T_{00}^B \right| \ge C\cdot s \cdot \sqrt{h^3 \log^3 (d/h)} + C \cdot s \cdot \sqrt{\frac{\log^4 (d/h)}{nh^2}}  + C \cdot \sqrt{h^5 \log n}      \;\Bigg|\; \{(Z_i,\bX_i,\bY_i)\}_{i\in [n]} \right]\\
&\le P\left[\left| T^B-T_{0}^B\right| \ge C\cdot s \cdot \sqrt{h^3 \log^3 (d/h)} + C \cdot s \cdot \sqrt{\frac{\log^4 (d/h)}{nh^2}}     \;\Bigg|\; \{(Z_i,\bX_i,\bY_i)\}_{i\in [n]} \right]\\
&\quad+ P\left[\left|T_0^B-T_{00}^B \right| \ge C \cdot \sqrt{h^5 \log n}      \;\Bigg|\; \{(Z_i,\bX_i,\bY_i)\}_{i\in [n]} \right]\\
&\le 2/d + 1/n,
\end{split}
\end{equation*}
as desired.

\subsection{Lower Bound of the Variance}

We aim to show that the variance of $J_{z,jk}$ defined in \eqref{EqA:jzjk} is bounded from below. 

\begin{lemma}\label{lm:var-lower-bd}
  Under the same conditions of Theorem~\ref{theorem:gaussian multiplier bootstrap}, there exists a constant $c>0$ such that 
  $\inf_z \min_{j,k} \Var(J_{z,jk}) \ge c>0$.
\end{lemma}

\begin{proof}
In this proof, we will apply Isserlis' theorem \citep{isserlis1918formula}.  Given $\Tb \sim N(0, \bSigma)$, Isserlis' theorem implies that  for any vectors $\ub, \vb \in \RR^d$,
\begin{align}
 \nonumber \EE\{(\ub^T \Tb \Tb^T \vb)^2\}  &= \EE\{(\ub^T \Tb)^2\}\EE\{(\vb^T\Tb)^2\} + 2\{\EE(\ub^T \Tb \vb^T\Tb)\}^2 \\
&= (\ub^T \bSigma \ub)(\vb^T \bSigma \vb) + 2(\ub^T \bSigma \vb)^2 \label{eq:low-iss}
\end{align}
 According to the definition of $J_{z,jk}$ in \eqref{EqA:jzjk}, it can be decomposed into $J_{z,jk} (Z_i,\bX_i,\bY_i) = J_{z,jk}^{(1)} (Z,\bX_i,\bY_i)  -J_{z,jk}^{(2)} (Z_i)$. 
 Recall that 
 \[
 J_{z,jk}^{(1)} (Z_i,\bX_{i},\bY_{i}) = \sqrt{h} \cdot \left\{\bTheta_j (z)\right\}^T  \cdot \left[ K_h (Z_i-z) \bX_{i}\bY_{i}^T -   \EE\left\{K_h (Z-z) \bX\bY^T  \right\}    \right] \cdot \bTheta_k(z),
\]
and
\[
J_{z,jk}^{(2)} (Z_i) = \sqrt{h} \cdot \left\{\bTheta_j (z)\right\}^T  \cdot \left[ K_h (Z_i-z) -   \EE\left\{K_h (Z-z)  \right\}    \right] \cdot \bSigma (z)\cdot \bTheta_k(z).
\]
 We will calculate $\Var\{J_{z,jk}^{(2)} (Z)\}$, $\Var\{J_{z,jk}^{(1)} (Z,\bX,\bY)\}$, and $\Cov\{J_{z,jk}^{(1)} (Z,\bX,\bY), J_{z,jk}^{(2)} (Z)\}$ separately.
 
We first calculate $\Var\{J_{z,jk}^{(2)} (Z)\}$. Following a similar method as the proof of Lemma \ref{lemma:bias}, we have
$
   \EE\{K_h(Z-z)\} = f_Z(z) + \cO(h^2)$ and  $\EE\{K^2_h(Z-z)\} =  h^{-1}f_Z(z) \int K^2(u) du + \cO(1).
$ 
This implies that 
\begin{equation}\label{eq:Var-J2}
   \Var\{J_{z,jk}^{(2)} (Z)\} = \bTheta^2_{jk}(z) \cdot f_Z(z) \int K^2(u) du + \cO(h).
\end{equation}

Next, we proceed to calculate the variance of $J_{z,jk}^{(1)} (Z)$.
By a change of variable and Taylor's expansion, we obtain
\begin{equation}
\label{fix1}
\begin{split}
&\bTheta_j(z)^T  \EE\big\{K_h(Z-z) \bSigma(Z)\}  \bTheta_k(z)\\
&=\bTheta_j(z)^T \left\{ \int K(u) \bSigma(z+uh) f_Z(z+uh) du\right\}  \bTheta_k(z)\\
&=\bTheta_j(z)^T \left[ \int K(u) \{\bSigma(z)+ uh \dot{\bSigma}(z) +u^2h^2 \ddot{\bSigma}(z')\} \{f_Z(z) + uh \dot{f}_Z(z) + u^2h^2 \ddot{f}_Z(z)\} du\right]  \bTheta_k(z).
\end{split}
\end{equation}
Note that each term in the integrant that involves $\int u K(u)du$ is equal to zero since $\int u K(u)du=0$ by assumption.
For terms with $\bSigma(z)$, we have 
\begin{equation*}
\begin{split}
&\bTheta_j (z)^T \bSigma(z) \bTheta_k(z)  \int K(u) \{f_Z(z)+uh \dot{f}_Z(z)+u^2h^2 \ddot{f}_Z (z)\} du \\
&=\bTheta_{jk}(z)\{f_Z(z)  + \cO(h^2)\}.
\end{split}
\end{equation*}
For terms that involve $\dot{\bSigma}(z)$ and $\ddot{\bSigma}(z')$, we have 
\[
\bTheta_j(z)^T \dot{\bSigma}(z) \bTheta_k(z) \le M_\sigma \|\bTheta_j(z)\|_2 \|\bTheta_k(z)\|_2 \le \rho^2 M_\sigma= \cO(1),
\]
since the maximum eigenvalue of $\bTheta (z)$ is bounded by $\rho$ by assumption.
Thus, combining the above into \eqref{fix1}, we have 
\begin{equation}
\label{fix2}
\begin{split}
\bTheta_j(z)^T  \EE\big\{K_h(Z-z) \bSigma(Z)\}  \bTheta_k(z)= \bTheta_{jk}(z) f_Z(z)+\cO(h^2).
\end{split}
\end{equation}

Next, we bound the second moment.  By the Isserlis' theorem in \eqref{eq:low-iss}, and by taking the conditional expectation, we have
\begin{equation}
\begin{split}
\label{eq:J1-exp2}
& \EE\big[  K^2_h(Z-z) \{\bTheta_j(z)^T\bX \bY^T\bTheta_k(z)\}^2 \big] \\
 &= \EE\big(K^2_h(Z-z) [ \{\bTheta_j(z)^T \bSigma(Z)  \bTheta_j(z)\}\{\bTheta_k(z)^T \bSigma(Z) \bTheta_k(z)\} + 2\{\bTheta_j(z)^T \bSigma(Z)  \bTheta_k(z)\}^2]\big).
\end{split}
\end{equation}
Following a similar argument as in \eqref{fix2}, we can derive
\begin{equation}
\label{fit3}
 \EE\big[  K^2_h(Z-z) \{\bTheta_j(z)^T\bX \bY^T\bTheta_k(z)\}^2 \big] =  \{\bTheta_{jj}(z)\bTheta_{kk}(z)  + 2\bTheta_{jk}^2(z) \} f_Z(z) h^{-1}\int K^2(u) du  + \cO(1)
\end{equation}
Thus, we have
\begin{equation}\label{eq:Var-J1}
 \Var\left\{J_{z,jk}^{(1)} (Z)\right\} =  \{\bTheta_{jj}(z)\bTheta_{kk}(z)  + 2\bTheta_{jk}^2(z) \} f_Z(z) \int K^2(u) du + \cO(h).
\end{equation}

Now we begin to bound the $\Cov\{J_{z,jk}^{(1)} (Z), J_{z,jk}^{(2)} (Z)\}$. By using a similar argument as \eqref{fix2}, we have
\begin{align}
\label{eq:J1-exp3}
 \EE\big[\bTheta_{jk}(z)  K^2_h(Z-z) \{\bTheta_j(z)^T\bX \bY^T\bTheta_k(z)\} \big]=\bTheta_{jk}^2 (z) \cdot h^{-1}f_Z(z) \int K^2(u) du+ \cO(1),
\end{align}

Combining with \eqref{eq:J1-exp3} and \eqref{fix2}, and using the covariance formula, we have that 
\begin{equation}\label{eq:J1-J2-cov}
\Cov\big\{J_{z,jk}^{(1)} (Z), J_{z,jk}^{(2)} (Z)\big\} = \bTheta^2_{jk}(z) f_Z(z) \int K^2(u) du+ \cO(h).
\end{equation}
Using \eqref{eq:Var-J2}, \eqref{eq:Var-J1} and \eqref{eq:J1-J2-cov}, we have
\begin{align*}
   \Var\{J_{z,jk} (Z)\} &= \Var(J_{z,jk}^{(1)} (Z)) + \Var\{J_{z,jk}^{(2)} (Z)\} - 2\Cov\big\{J_{z,jk}^{(1)} (Z), J_{z,jk}^{(2)} (Z)\big\}\\
   &  = \{\bTheta_{jj}(z)\bTheta_{kk}(z)  + \bTheta_{jk}^2(z) \} f_Z(z) \int K^2(u) du + \cO(h) \ge \rho^2\underline{f}_Z,
\end{align*}
where the last inequality is because $\rho$ is smaller than the
minimum eigenvalue of $\bSigma(z)$ for any $z \in [0,1]$ and $\inf_{z\in [0,1]} \; f_Z (z)\ge \underline{f}_Z>0$ by Assumption \ref{ass:marginal}. Since the lower bound above is uniformly true over $z, j, k$, the lemma is proven.

\end{proof}

\section{Proof of Theorem~\ref{theorem:algorithm2}}
\label{appendix:proof of algorithm2}
In this section, we show that the proposed procedure in Algorithm~\ref{al:stepdown} 
is able to control the type I error below a pre-specified level  $\alpha$. 
We first define some notation that will be used throughout the proof of Theorem~\ref{theorem:algorithm}.
Let $E^*(z)$ be the true edge set at $Z=z$. That is, $E^*(z)$ is the set of edges induced by the true inverse covariance matrix  $\bTheta (z)$.  
Recall from Definition~\ref{def:critical} that the critical edge set  is defined as 
\begin{equation}
\label{appendix:proof of algorithm crit}
  \cC\{E(z), \cP\} = \{e \mid e \not\in E(z), \;   \mathrm{there\; exists\;} E'(z) \supseteq E(z) \mathrm{ \;such\; that\; } E' (z) \in \mathscr{P} \text{ and }   E'(z) \backslash \{e\}  \notin \mathscr{P} \},
\end{equation}
where $\mathscr{P}= \{      E\subseteq V\times V \mid \cP(G)=1\}$ is the class of edge sets satisfying the graph property $\cP$.

Suppose that Algorithm~\ref{al:stepdown} rejects the null hypothesis at the $T$th iteration.  That is, there exists  $z_0 \in [0,1]$ such that $E_T (z_0) \in \mathscr{P}$ but $E_{T-1} (z_0) \notin \mathscr{P}$.
To prove Theorem~\ref{theorem:algorithm}, we state the following two lemmas on the properties of critical edge set.

\begin{lemma}
\label{lemma:algorithm1}
Let $E_T(z_0) \in \mathscr{P}$ for some $z_0 \in[0,1]$. 
Then, at least one rejected edge  in $E_T(z_0)$ is in the critical edge set $\cC\{E^* (z_0),\cP\}$.
\end{lemma}

\begin{lemma}
\label{lemma:algorithm2}
Let $\bar{e} \in \cC\{E^*(z_0),\cP\}$ be the first rejected edge in the critical edge set $\cC\{E^*(z_0),\cP\}$. 
Suppose that $\bar{e}$ is rejected at the $l$th step of Algorithm~\ref{al:stepdown}.
  Then, $ \cC\{E^*(z), \cP\}\subseteq \cC \{E_{l-1}(z),\cP\} $ for all $z\in[0,1]$.
\end{lemma}

The proofs of Lemmas~\ref{lemma:algorithm1} and~\ref{lemma:algorithm2} are deferred to Sections~\ref{proof:lemma:algorithm1} and~\ref{proof:lemma:algorithm2}, respectively.  We now provide the proof of Theorem~\ref{theorem:algorithm}.

\subsection{Proof of Theorem~\ref{theorem:algorithm}}
Suppose that Algorithm~\ref{al:stepdown} rejects the null hypothesis at the $T$th iteration.
That is, $E_{T} (z_0)\in \mathscr{P}$ and $E_{T-1} (z_0) \notin \mathscr{P}$. 
By Lemma~\ref{lemma:algorithm1}, there is at least one edge in $E_T(z_0)$ that is also in the critical edge set $\cC\{E^*(z_0),\cP\}$.  We denote the first rejected edge in the critical edge set as $\bar{e}$, i.e., $\bar{e} \in 
\cC\{E^*(z_0),\cP\}$ and suppose that $\bar{e}$ is rejected at the $l$th iteration of Algorithm~\ref{al:stepdown}.  We note that $l$ is not necessarily $T$.
Thus, we have 
\begin{equation*}
\begin{split}
 \underset{z\in [0,1]}{\sup} \;\underset{e\in \cC\{E^*(z),\cP\}}{\max} \; 
\sqrt{nh}\cdot  \hat{\bTheta}_{e}^\mathrm{de} (z) \cdot \sum_{i \in [n]} K_h(Z_i-z)/n  &\ge\sqrt{nh}\cdot  \hat{\bTheta}_{\bar{e}}^\mathrm{de} (z_0) \cdot \sum_{i \in [n]} K_h(Z_i-z_0)/n\\
&\ge c\{1-\alpha , \cC(E_{l-1},\cP)\}\\
&\ge c\{1-\alpha , \cC(E^*,\cP)\},
\end{split}
\end{equation*}
where the first inequality follows by Lemma~\ref{lemma:algorithm1}, the second inequality follows from the $l$th step of Algorithm~\ref{al:stepdown}, and the last inequality follows directly from Lemma~\ref{lemma:algorithm2}.

Under the null hypothesis, $\bTheta_{e}(z) = 0$ for any $e\in \cC\{E^*(z) ,\cP\}$.  By Theorem~\ref{theorem:gaussian multiplier bootstrap}, we have 
\begin{equation*}
\begin{split}
&\underset{n\rightarrow \infty}{\lim} \; \underset{\bTheta (\cdot) \in \cG_0}{\sup}\; P_{\bTheta (\cdot)} (\psi_{\alpha}=1) \\
&\le 
\underset{n\rightarrow \infty}{\lim} \; \underset{\bTheta (\cdot) \in  \cG_0  }{\sup}
\; P \left[ \underset{z\in [0,1]}{\sup}\;  \underset{e\in \cC\{E^*(z),\cP\}}{\max} \; \sqrt{nh} \cdot |\hat{\bTheta}^\mathrm{de}_{jk} (z)|\cdot \sum_{i\in [n]}K_h(Z_i-z)/n \ge c\{1-\alpha,\cC(E^*,\cP)  \}    \right] \\
&\le \alpha,
\end{split}
\end{equation*}
as desired.

\subsection{Proof of Lemma~\ref{lemma:algorithm1}}
\label{proof:lemma:algorithm1}

To prove Lemma~\ref{lemma:algorithm1}, it suffices to show that the intersection between the two sets $E_T(z_0)$ and $\cC\{E^* (z_0),\cP\}$ is not an empty set, i.e., $E_T(z_0) \cap \cC\{E^* (z_0),\cP\} \ne \emptyset$.  
To this end, we let  $F= E_T(z_0)\cup E^*(z_0)$ and let  $E_T(z_0) \setminus E^*(z_0) = \{e_1,e_2,\ldots,e_k\}$.  
We note that the set $E_T(z_0) \setminus E^*(z_0) $ is not an empty set since $E_T(z_0) \in \mathscr{P}$ but $ E^*(z_0) \notin \mathscr{P}$.

Using the fact that $\cP$ is monotone and that $E_T(z_0) \in \mathscr{P}$, we have $F \in \mathscr{P}$ since adding additional edges to $E_T(z_0)$ does not change the graph property of $E_T(z_0)$.
Then, we have
\[
E^*(z_0) \subseteq E^*(z_0) \cup \{e_1\} \subseteq E^*(z_0) \cup \{e_1,e_2\} \subseteq \cdots  \subseteq E^*(z_0) \cup \{e_1,\ldots,e_k\}   = F.
\]
Since $E^*(z_0) \notin \mathscr{P}$ and $F\in \mathscr{P}$, there must exists an edge set $\{e_1,\ldots,e_{k_0}\}$ for $k_0 \le k$ that changes the graph property of $E^*(z_0)$ from $E^*(z_0)\notin \mathscr{P}$ to $E^*(z_0) \cup \{e_1,\ldots,e_{k_0}\} \in \mathscr{P}$.

Thus, there must exists at least an edge $\bar{e} \in \{e_1,\ldots,e_{k_0}\}$ such that $\bar{e} \in \cC\{E^*(z_0),\cP\}$
since adding the set of edges $\{e_1,\ldots,e_{k_0}\}$ changes the graph property of $E^*(z_0)$.  Also, $\bar{e} \in E_T(z_0)$ by construction.  Thus, we conclude that $E_T(z_0) \cap \cC\{E^* (z_0),\cP\} \ne \emptyset$.

\subsection{Proof of Lemma~\ref{lemma:algorithm2}}
\label{proof:lemma:algorithm2}
Let $\bar{e} \in \cC\{E^*(z_0),\cP\}$ be the first rejected edge 
in the critical edge set $\cC\{E^*(z_0),\cP\}$ for some $z_0\in [0,1]$.  
Suppose that $\bar{e}$ is rejected at the $l$th step of  Algorithm~\ref{al:stepdown}.  
We want to show that $\cC\{E^*(z),\cP\} \subseteq \cC\{E_{l-1} (z),\cP\}$
for all $z\in [0,1]$.
It suffices to show that $\cC\{E^*(z_0),\cP\} \subseteq \cC\{E_{l-1} (z_0),\cP\}$.
In other words, we want to prove that for any $e' \in \cC\{E^*(z_0),\cP\}$, $e' \in \cC\{E_{l-1} (z_0),\cP\}$.
We first note the following fact
\begin{equation}
\label{Eq:proof:lemma:algorithm21}
E_{l-1}(z_0) \cap \cC\{E^*(z_0),\cP\} = \emptyset \qquad \mathrm{and} \qquad 
E_{l-1} (z_0) \notin \mathscr{P}.
\end{equation}

By the definition of the critical edge set (\ref{appendix:proof of algorithm crit}), we construct a set $E'$ such that $E^*(z_0)\supseteq E'$, $E' \in \mathscr{P}$, and $E' \setminus \{e'\} \notin \mathscr{P}$, for any $e' \in \cC\{E^*(z_0),\cP\}$. 
By the definition of monotone property, we have $E' \cup E_{l-1}(z_0) \in \mathscr{P}$.  Since $\cC\{E' \cup E_{l-1}(z_0),\cP\} \subseteq \cC\{ E_{l-1}(z_0),\cP\} $, to show that $e' \in  \cC\{E_{l-1}(z_0),\cP\}$, it is equivalent to showing $e' \in \cC\{E'\cup E_{l-1}(z_0),\cP\}$.
That is, we want to show 
\[
\{E' \cup E_{l-1}(z_0)\}\setminus \{e'\} \notin \mathscr{P}.
\]
This is equivalent to showing 
\begin{equation}
\label{Eq:proof:lemma:algorithm22}
\{E' \setminus e'\} \cup \{E_{l-1}(z_0)\setminus E'\} \notin \mathscr{P}.
\end{equation}
There are two cases: (1) $E_{l-1} (z_0)\setminus E'  = \emptyset$ and (2) $E_{l-1} (z_0)\setminus E'  \ne \emptyset$.  For the first case, (\ref{Eq:proof:lemma:algorithm22}) is true by the construction of $E'$.
For the second case, we prove by contradiction.

Suppose that $(E' \setminus e') \cup \{E_{l-1}(z_0)\setminus E'\} \in \mathscr{P}$.
Let $E_{l-1}(z_0)\setminus E' = \{e'_1,\ldots,'_k\}$.  By the definition of monotone property, we have 
\[
E' \setminus \{e'\} \subseteq (E' \setminus \{e'\} ) \cup \{e_1'\}
\subseteq \cdots 
\subseteq (E' \setminus \{e'\} ) \cup \{e_1',e_2',\ldots,e_{k}'\} 
=(E' \setminus \{e'\}) \cup (E_{l-1}(z_0)\setminus E').
\]
Since $E' \setminus \{e'\} \notin \mathscr{P}$  by construction, and that  
$(E' \setminus \{e'\}) \cup (E_{l-1}(z_0)\setminus E') \in \mathscr{P}$, there must exists an edge set $\{e_1,\ldots,e_{k_0}\}$ for $k_0 \le k$ that changes the graph property of $E' \setminus \{e'\}\notin \mathscr{P}$ 
to $(E' \setminus \{e'\}) \cup \{e_1',\ldots,e_{k_0}'\} \in \mathscr{P}$.

Since $e'_{k_0} \in E_{l-1}(z_0) \setminus E'$ and that $E^*(z_0) \subseteq E'$ by construction, we have $e'_{k_0} \notin E^*(z_0)$.  Thus, $e'_{k_0} \in \cC\{E^*(z_0),\cP\}$.  This contradicts the fact that 
\[
E_{l-1}(z_0) \cap \cC\{E^*(z_0),\cP\} = \emptyset.
\]

\section{Proof of Theorem \ref{theorem:algorithm-power2}}
\label{proof:theorem:algorithm-power2}
By the definition in \eqref{eq:G1}, if $\bTheta(\cdot) \in  \cG_1(\theta; \cP)$, there exists an edge set $E'_0$ and $z_0 \in [0,1]$ satisfying 
\begin{equation}\label{eq:power-e0}
  E_0' \subseteq E\{\bTheta(z_0)\}, \cP(E_0') = 1 \text{ and } \min_{e \in E_0'} |\bTheta_e(z_0)| > C \sqrt{\log (d/h)/nh},
\end{equation}
and we will determine the magnitude tf constant $C$ later. We aim to show that  $\cP\{E_0' \cap \cC(\emptyset, \cP)\} = \cP(E_0') = 1$. First, there exists a subgraph $E_0'' \subset E_0'$ such that $\cP(E_0'' ) =  \cP(E_0')  = 1$ and for any $\tilde E \subset E_0''$,  $\cP(\tilde E ) =0$.  We can construct such $E_0'' $   by  deleting edges from $E_0'$ until it is impossible to further deleting any edge such that the property $\cP$ is still true. By Definition \ref{def:critical}, $E_0'' \subseteq \cC(\emptyset, \cP)$ and therefore $E_0'' \subseteq E_0' \cap \cC(\emptyset, \cP)$.  By monotone property, we have  $\cP\{E_0' \cap \cC(\emptyset, \cP)\} = \cP(E_0') =1$ since  $\cP(E''_0 ) = \cP(E'_0 ) = 1$.
Consider the following event
  \[
     \cE_1 = \Big[\min_{e \in E_0' \cap \cC(\emptyset, \cP)} \sqrt{nh}|\hat \bTheta^\mathrm{de}_e(z_0)| \cdot \sum_{i\in [n]} K_h(Z_i-z_0)/n > c\{ 1- \alpha, \cC(\emptyset, \cP)\} \Big].
  \]
According to Algorithm \ref{al:stepdown}, the rejected set in the first iteration at $z_0$ is
\[
   E_1(z_0) = \Big[e \in \cC(\emptyset, \cP) : \sqrt{nh} |\hat{\bTheta}^\mathrm{de}_e(z_0)|  \cdot \sum_{i\in [n]} K_h(Z_i-z_0)/n> c\{1-\alpha, \cC(\emptyset, \cP)\}  \Big].
\]
Under the event  $\cE_1$, we have $E_0' \cap \cC(\emptyset, \cP) \subseteq E_1(z_0)$ and since $\cP\{E_0' \cap \cC(\emptyset, \cP)\} = \cP(E_0')$, we have $\cP\{E_1(z_0)\} = \cP(E_0') = 1$. Therefore,
\begin{equation}\label{eq:power-e1}
    \PP(\psi_{\alpha} = 1) \ge \PP(\cE_1).
\end{equation}
It suffices to bound $\PP(\cE_1)$ then. We consider two events
\begin{align*}
 \cE_2 &= \Big[\min_{e \in E_0'} \sqrt{nh} |\bTheta_e(z_0)| \cdot \sum_{i\in [n]} K_h(Z_i-z_0)/n > 2 c\{1-\alpha, \cC(\emptyset, \cP)\} \Big]; \\
  \cE_3 &= \Big[\max_{e \in V \times V} \sqrt{nh} |\hat\bTheta^\mathrm{de}_e(z_0) - \bTheta_e(z_0)|\cdot \sum_{i\in [n]}  K_h(Z_i-z_0)/n \le c\{1-\alpha, \cC(\emptyset, \cP)\} \Big].
\end{align*}
We have
$
   \PP(\cE_1)  \ge \PP(\cE_2 \cup \cE_3).
$
By Lemmas \ref{lemma:bias} and \ref{lemma:variance}, we have
\begin{equation}
\textstyle  \PP\Big\{\sup_z\Big|\sum_{i\in [n]} K_h(Z_i-z)/n - f_Z(z)\Big| >    \sqrt{\log (d/h)/nh}\Big\} < 3/d.
\end{equation} 
Combining with \eqref{Eq:inverse1} in Corollary \ref{theorem:inverse estimation}, we have with probability at least $1-6/d$,
\[
\sup_z \max_{e \in V \times V} \sqrt{nh} |\hat\bTheta^\mathrm{de}_e(z) - \bTheta_e(z)|\cdot \sum_{i\in [n]}  K_h(Z_i-z)/n  \le C \sqrt{\log (d/h)/nh}\cdot \sqrt{nh}.
\]
For any fixed $\alpha \in (0,1)$ and sufficiently large $d, n$, as $\cC(\emptyset, \cP) \subseteq V \times V$, we have
\[
c\{1-\alpha, \cC(\emptyset, \cP)\} \le c(1-\alpha, V\times V) \le C \sqrt{\log (d/h)/nh}\cdot \sqrt{nh}.
\]
Thus $\PP(\cE_3) > 1-6/d$. Similarly, we also have $\PP(\cE_2) > 1-3/d$. By \eqref{eq:power-e1}, we have
\[
\PP(\psi_{\alpha} = 1) \ge \PP(\cE_1) \ge \PP(\cE_2 \cup \cE_3) \ge 1- 9/d.
\]
Therefore, we complete the proof of the theorem.

\section{Technical Lemmas on Covering Number}
\label{appendix:covering number}
In this section, we present some technical lemmas on the covering number of some function classes.  
Lemma~\ref{lemma:coveringnumber3} provides an upper bound on the covering number for the class of function of bounded variation.  Lemma~\ref{lemma:coveringnumber1} provides an upper bound on the covering number of a class of Lipschitz function.  Lemma~\ref{lemma:coveringnumber2}
provides an upper bound on the covering numbers for function classes generated from the product and addition of two function classes.  

\begin{lemma}
\label{lemma:coveringnumber3}
(Lemma 3 in \citealp{gine2009exponential})
Let $K:\RR\rightarrow\RR$ be a function of bounded variation.
Define the function class $\cF_h = \left[  K\left\{(t-\cdot)/h)\right\}\mid t\in \RR  \right]$.
Then, there exists $C_K < \infty$ independent of $h$ and of $K$ such that for all $0<\epsilon<1$,
\[
\underset{Q}{\sup}\; N\{\cF_h, L_2 (Q),\epsilon\} \le \left(\frac{2\cdot C_K\cdot \|K\|_{\mathrm{TV}}}{\epsilon}\right)^4,
\]
where $\|K\|_{\mathrm{TV}} $ is the total variation norm of the function $K$.
\end{lemma}

\begin{lemma}
\label{lemma:coveringnumber1}
Let $f(l)$ be a Lipschitz function defined on $[0,1]$ such that $|f(l)- f(l')| \le L_f \cdot |l-l'|$ for any $l,l'\in [0,1]$. We define the constant function class $\cF = \{g_l := f(l) \mid l\in [0,1] \} $.
For any probability measure $Q$, the covering number of the function class $\cF$  satisfies 
\[
N\{\cF ,L_2(Q),\epsilon\} \le \frac{L_f}{\epsilon},
\]
where $\epsilon \in (0,1)$.
\end{lemma}
\begin{proof}
Let $\cN = \left\{\frac{i\epsilon}{L_f} \; | \;i=1,\ldots,\frac{L_f}{\epsilon}\right\}$. By definition of $\cN$, for any $l\in [0,1]$, there exists an $l'\in \cN$ such that $|l-l'|\le \epsilon /L_f$.
Thus, we have 
\[
|f(l) - f(l')| \le L_f \cdot |l-l'| 	\le \epsilon.
\]
This implies that $\{g_l \; | \; l\in \cN\}$ is an $\epsilon$-cover of the function class $\cF$. To complete the proof, we note that the cardinality of the set  $|\cN| \le L_f/\epsilon$.  
\end{proof}

\begin{lemma}
\label{lemma:coveringnumber2}
Let $\cF_1$ and $\cF_2$ be two function classes satisfying 
\[
N\{\cF_1 ,L_2(Q),a_1 \epsilon\} \le C_1 \epsilon^{-v_1} \quad \mathrm{and}\quad
N\{\cF_2 ,L_2(Q),a_2 \epsilon\} \le C_2 \epsilon^{-v_2}
\]
for some $C_1,C_2,a_1,a_2,v_1,v_2 >0$ and any $0<\epsilon<1$.  
Define $\|F_\ell\|_\infty = \underset{f\in \cF_\ell}{\sup} \; \|f\|_{\infty}$ for $\ell=1,2$ and $U = \|\cF_1\|_\infty \vee  \|\cF_2\|_\infty$.
For the function classes $\cF_{\times} = \{ f_1 f_2 \;|\; f_1\in \cF_1, \;f_2\in \cF_2\}$ and $\cF_{+} = \{f_1+f_2 \;|\; f_1\in \cF_1, \;f_2 \in \cF_2 \} $, we have for any $\epsilon \in (0,1)$, 
\[
N\{\cF_{\times} ,L_2(Q), \epsilon\} \le C_1\cdot C_2\cdot  \left(\frac{2a_1 U}{\epsilon} \right)^{v_1} \cdot \left(\frac{2a_2 U}{\epsilon} \right)^{v_2}
\]  
and 
\[
N\{\cF_{+} ,L_2(Q), \epsilon\} \le C_1\cdot C_2\cdot \left(\frac{2a_1 }{\epsilon} \right)^{v_1}\cdot \left(\frac{2a_2 }{\epsilon} \right)^{v_2}.
\]  
\end{lemma}

\begin{lemma}
\label{lemma:coveringnumberkz}
Let $w_z (u) = K_h (u-z)$.  We define the function classes
\[
\cK_1 = \left\{ w_z (\cdot) \; | \; z\in [0,1]    \right\} \qquad \mathrm{and}\qquad 
\cK_2 = \left[ \EE\{w_z (Z)\} \;|\; z\in [0,1]       \right].
\]
 Given Assumptions~\ref{ass:marginal}-\ref{ass:covariance}, we have 
for any $\epsilon \in (0,1)$,
\[
\underset{Q}{\sup}\;N\{\cK_1,L_2(Q),\epsilon\} \le \left(   \frac{2 \cdot C_K \cdot \|K\|_{\mathrm{TV}}}{h\epsilon}    \right)^{4} 
\]
and 
 \[
\underset{Q}{\sup}\;N\{\cK_2,L_2(Q),\epsilon\} \le \frac{2}{h\epsilon} \cdot \|K\|_{\mathrm{TV}} \cdot \bar{f}_Z.
\]
Moreover, let $k_z (u) = w_z(u) - \EE\{w_z (Z)\}$ and let $\cK = \{k_z (\cdot) \mid  z\in [0,1] \}$. We have 
\[
\underset{Q}{\sup}\;N\{\cK,L_2(Q),\epsilon\} \le \left(\frac{4 \cdot \|K\|_{\mathrm{TV}} \cdot C_K^{4/5} \cdot \bar{f}_Z^{1/5}}{h\epsilon}\right)^{5}.
\]
\end{lemma}
\begin{proof}
The covering number for the function class $\cK_1$ is obtained by an application of Lemma~\ref{lemma:coveringnumber3}.  To obtain the covering number for $\cK_2$, we show that the constant function $\EE\{w_z (Z)\}$ is Lipschitz.  The covering number is obtained by applying Lemma~\ref{lemma:coveringnumber1}.  Finally, we note that the function class $\cK$ is generated from the addition of the two function classes $\cK_1$ and $\cK_2$.  The covering number can be obtained by an application of Lemma~\ref{lemma:coveringnumber2}. The details are deferred to \ref{proof:lemma:coveringnumberkz}. 
\end{proof}
\begin{lemma}
\label{lemma:coveringnumbergzjk}
Let $g_{z,jk} (u,X_{ij},Y_{ik}) = K_h (u-z) X_{ij}Y_{ik}$. We define the function classes
\[
\cG_{1,jk} = \left\{ g_{z,jk} (\cdot) \mid z\in [0,1]  \right\}
\qquad
\mathrm{and}
\qquad
\cG_{2,jk}= \left[ \EE\{g_{z,jk} (Z,X_j,Y_k)\} \mid z\in [0,1]  \right].
\]
 Given Assumptions~\ref{ass:marginal}-\ref{ass:covariance},  for all $\epsilon \in (0,1)$,
  \[
\underset{Q}{\sup}\; N\{\cG_{1,jk}, L_2(Q),   \epsilon\} \le   \left( \frac{2\cdot M_X^2 \cdot \log d\cdot  C_K \cdot \|K\|_{\mathrm{TV}} }{h\epsilon}\right)^4
\]
and 
 \[
\underset{Q}{\sup}\; N\{\cG_{2,jk},L_2(Q),\epsilon\} \le \frac{2}{h \epsilon}\cdot  \|K\|_{\mathrm{TV}} \cdot \bar{f}_Z\cdot M_\sigma,
\]
with probability at least $1-1/d$. 
Moreover, let $q_{z,jk} (u,X_{ij},Y_{ik}) = g_{z,jk}(u,X_{ij},Y_{ik}) - \EE\{g_{z,jk}(Z,X_{j},Y_{k})\}$ and let $\cG_{jk} = \left\{ q_{z,jk}(\cdot) \mid z\in [0,1] \right\}$. 
We have
\[
\underset{Q}{\sup}\; N\{\cG_{jk},L_2(Q),\epsilon\} \le \left(\frac{4 \cdot \|K\|_{\mathrm{TV}} \cdot C_K^{4/5} \cdot \bar{f}_Z^{1/5}\cdot M_\sigma^{1/5}  \cdot M^{8/5}_X\cdot  \log^{4/5} d}{h\epsilon}\right)^{5}.
\]
with probability at least $1-1/d$.
\end{lemma}
\begin{proof}
The proof uses the same set of argument as in the proof of Lemma~\ref{lemma:coveringnumberkz}. 
The probability statement comes from the fact that we upper bound the random variable $X_j$ by $M_X \cdot \sqrt{\log d}$ for some constant $M_X>0$.
The details are deferred to \ref{proof:lemma:coveringnumbergzjk}.
\end{proof}

\begin{lemma}
\label{lemma:coveringnumberjzjk1}
Let $J_{z,jk}^{(1)} (u,\bX_{i},\bY_{i}) = \sqrt{h} \cdot \left\{\bTheta_j (z)\right\}^T  \cdot \left[ K_h (u-z) \bX_{i}\bY_{i}^T -   \EE\left\{K_h (Z-z) \bX\bY^T  \right\}    \right] \cdot \bTheta_k(z)$ and let 
$\cJ^{(1)}_{jk} =\{ J_{z,jk}^{(1)} \;|\; z\in [0,1]  \}$.
 Given Assumptions~\ref{ass:marginal}-\ref{ass:covariance},  for all $\epsilon\in (0,1)$
\[
\underset{Q}{\sup}\; N\{\cJ^{(1)}_{jk},L_2(Q),\epsilon\} \le C\cdot \left( \frac{d^{5/4} \cdot \log^{3/2} d }{\sqrt{h} \cdot \epsilon}   \right)^6,
\]
with probability at least $1-1/d$, where $C>0$ is a generic constant that does not depend on $d$, $h$, and $n$.
\end{lemma}
\begin{proof}
The proof is deferred to \ref{proof:lemma:coveringnumberjzjk1}.
\end{proof}
\begin{lemma}
\label{lemma:coveringnumberjzjk2}
Let $J_{z,jk}^{(2)} (u) = \sqrt{h} \cdot \left\{\bTheta_j (z)\right\}^T  \cdot \left[ K_h (u-z)  -   \EE\left\{K_h (Z-z)   \right\}    \right]\cdot  \bSigma (z)\cdot \bTheta_k(z) $ and let 
$
\cJ^{(2)}_{jk} =\{ J_{z,jk}^{(2)} \mid z\in [0,1] \}.
$
 Given Assumptions~\ref{ass:marginal}-\ref{ass:covariance},  for all probability measures $Q$ on $\RR$ and all $0<\epsilon<1$, 
\[
N\{\cJ^{(2)}_{jk},L_2(Q),\epsilon\} \le C \cdot \left(\frac{d^{1/6}}{h^{4/3}\cdot \epsilon} \right)^6,
\]
where $C>0$ is a generic constant that does not depend on $d$, $h$, and $n$.
\end{lemma}
\begin{proof}
We first note that  $\cJ^{(2)}_{jk}$ is a function class generated from the product of two function classes $\cK$ as in Lemma~\ref{lemma:coveringnumberkz} and $\varTheta_{jk} = \{\bTheta_{jk}(z) \;|\;z\in [0,1]\}$. To obtain the covering number of $\varTheta_{jk}$, we show that the constant function $\bTheta_{jk}(z)$ is Lipschitz and apply Lemma~\ref{lemma:coveringnumber1}.  We then apply Lemma~\ref{lemma:coveringnumber2} to obtain the covering number of $\cJ^{(2)}_{jk}$. The details are deferred to \ref{proof:lemma:coveringnumberjzjk2}.
\end{proof}

\subsection{Proof of Lemma~\ref{lemma:coveringnumberkz}}
\label{proof:lemma:coveringnumberkz}

Let $w_z (u) = K_h (u-z)$ and that $k_z (u) = w_z(u) - \EE\{w_z (Z)\}$. We first obtain the covering number for the function classes
\[
\cK_1 = \{ w_z(\cdot )\mid  z\in [0,1]\}
\qquad \mathrm{and} \qquad  
\cK_2 = [ \EE\{w_z (Z)\}\mid z\in [0,1]].
\]
Then, we apply Lemma~\ref{lemma:coveringnumber2} to obtain the covering number of the function class  
\[
\cK = \{k_z (\cdot) \mid  z\in [0,1] \}.
\]

\textbf{Covering number for $\cK_1$:}
By an application of Lemma~\ref{lemma:coveringnumber3}, the covering number for $\cK_1$ is 
\begin{equation}
\label{proof:lemma:coveringnumberkz1}
\underset{Q}{\sup}\;N\{\cK_1, L_2(Q), \epsilon\} \le \left( \frac{2\cdot C_K \cdot \|K\|_{\mathrm{TV}} }{h\epsilon}\right)^4.
\end{equation}

\textbf{Covering number for $\cK_2$:}
First, note that $\EE\{w_z(Z)\} = \int  K_h (z-Z) f_Z(Z)dZ   = ( K_h * f_Z) (z) $ is a function of $z$ generated by the convolution $(K_h * f_Z)(z)$.
By the property of the derivative of a convolution as in (\ref{EqA:convolution derivative}),  we have 
\begin{equation}
\label{proof:lemma:coveringnumberkz2}
\begin{split}
\underset{z_0\in [0,1]}{\sup} \;\left|  \frac{\partial }{\partial z}  \EE \{w_z (Z)\} \Big|_{z=z_0}   \right|  &=\underset{z_0\in [0,1]}{\sup}  \;  \left| \dot{K}_h * f_Z (z_0)    \right|
= \left\|  (\dot{K}_h * f_Z)(z) \right\|_{\infty}
\le \left\|  \dot{K}_h  \right\|_1  \cdot \|f_Z\|_{\infty},
\end{split}
\end{equation}
where the last expression is obtained by an application of Young's inequality. 
The expression in (\ref{proof:lemma:coveringnumberkz2}) depends on the quantity $ \| \dot{K}_h  \|_1$, which is equal to the following expression
\begin{equation}
\label{proof:lemma:coveringnumberkz3}
 \left\|  \dot{K}_h  \right\|_1 = \int \frac{1}{h^2}   \left| \dot{K}\left(\frac{Z-z}{h}\right) \right| dZ = \frac{1}{h} \int \left|\dot{K}(u)\right|du = \frac{1}{h} \cdot \|K\|_{\mathrm{TV}},
\end{equation}
where the second inequality holds by a change of variable, and $\|K\|_{\mathrm{TV}}$ is the total variation of the function $K(\cdot)$. 
Substituting (\ref{proof:lemma:coveringnumberkz3}) into (\ref{proof:lemma:coveringnumberkz2}) and by Assumption~\ref{ass:marginal}, we have 
\begin{equation}
\label{proof:lemma:coveringnumberkz4}
\underset{z_0\in [0,1]}{\sup} \;\left|  \frac{\partial }{\partial z}  \EE \{w_z (Z)\} \Big|_{z=z_0}   \right|\le   \frac{1}{h}\cdot  \|K\|_{\mathrm{TV}} \cdot \bar{f}_Z.
\end{equation}

Thus, for any $z_1,z_2 \in [0,1]$, we have 
\[
\left|   \EE \{w_{z_1}(Z)\}-   \EE \{w_{z_2}(Z)\}    \right| \le \frac{1}{h}\cdot  \|K\|_{\mathrm{TV}} \cdot \bar{f}_Z \cdot |z_1-z_2|,
\]
implying that $\EE\{w_z(Z)\}$ is a Lipschitz continuous function with Lipschitz constant $h^{-1}\cdot \|K\|_{\mathrm{TV}}\cdot \bar{f}_Z$.
By Lemma~\ref{lemma:coveringnumber1}, an upper bound for the covering number of $\cK_2$ is
\begin{equation}
\label{proof:lemma:empirical process w10}
\underset{Q}{\sup}\; N\{\cK_2,L_2(Q),\epsilon\} \le \frac{2}{h \epsilon} \cdot \|K\|_{\mathrm{TV}} \cdot \bar{f}_Z.
\end{equation}

\textbf{Covering number of the function class $\cK$:}
The  function class $\cK$ can be written as 
\begin{equation*}
\cK = \{f_1 - f_2 \mid  f_1 \in \cK_1, \;f_2 \in \cK_2 \}.
\end{equation*}
By an application of Lemma~\ref{lemma:coveringnumber2} with $C_1 = (2\cdot C_K\cdot \|K\|_{\mathrm{TV}})^4$, $C_2 = 2 \cdot \bar{f}_Z \cdot \|K\|_{\mathrm{TV}}$, $a_1 = a_2 = h^{-1}$, $v_1 =4$, and $v_2=1$, along with (\ref{proof:lemma:coveringnumberkz1}) and (\ref{proof:lemma:empirical process w10}), we obtain 
\begin{equation*}
\underset{Q}{\sup} \;N\{\cK,L_2(Q),\epsilon\} \le \left(\frac{4 \cdot \|K\|_{\mathrm{TV}} \cdot C_K^{4/5} \cdot \bar{f}_Z^{1/5}}{h\epsilon}\right)^{5}.
\end{equation*}

\subsection{Proof of Lemma~\ref{lemma:coveringnumbergzjk}}
\label{proof:lemma:coveringnumbergzjk}

Throughout the proof, we condition on the event 
\begin{equation}
\label{Eq:proof:lemma:coveringnumbergzjk1}
\cA = \left\{\underset{i\in [n]}{\max} \; \underset{j\in [d]}{\max} \; \max(|X_{ij}|,|Y_{ij}|) \le M_X \cdot \sqrt{\log d}  \right\}.
\end{equation}
Since $\bX$ and $\bY$ conditioned on $Z$ are Gaussian random variables, the event $\cA$ occurs with probability at least $1-1/d$ for sufficiently large constant $M_X>0$.  

Recall that $g_{z,jk} (u,X_{ij},Y_{ik}) = K_h (u-z) X_{ij}Y_{ik}$ and that $q_{z,jk} (u,X_{ij},Y_{ik}) = K_h (u-z) X_{ij}Y_{ik} - \EE\{ K_h (Z-z) X_{j}Y_{k} \} $. We first obtain the covering number of the function classes
\[
\cG_{1,jk} = \{ g_{z,jk} (\cdot) \mid z\in [0,1]\}
\]
and 
\[
\cG_{2,jk} = [ \EE\{g_{z,jk}(Z,X_{j},Y_{k})\}\mid z\in [0,1]].
\]
Then, we apply Lemma~\ref{lemma:coveringnumber2} to obtain the covering number of the function class 
\[
\cG_{jk}= \{q_{z,jk}(\cdot) \mid z\in [0,1],\; j,k\in [d]\}.
\]

\textbf{Covering number for $\cG_{1,jk}$:}
Conditioned on the event $\cA$ in (\ref{Eq:proof:lemma:coveringnumbergzjk1}), we have 
\[
g_{z,jk} (u, X_{ij},Y_{ik})  = K_h (u-z) X_{ij} Y_{ik} \le M^2_X\cdot  \log d \cdot K_h (u-z).
\]
By an application of Lemma~\ref{lemma:coveringnumber3}, the covering number for $\cG_{1,jk}$ is 
\begin{equation}
\label{Eq:proof:lemma:coveringnumbergzjk2}
\underset{Q}{\sup}\; N\{\cG_{1,jk}, L_2(Q),   \epsilon\} \le   \left( \frac{2\cdot M_X^2 \cdot \log d\cdot  C_K \cdot \|K\|_{\mathrm{TV}} }{h\epsilon}\right)^4.
\end{equation}

\textbf{Covering number for $\cG_{2,jk}$:}
We now obtain the covering number for $\cG_{2,jk}$ by showing that the function $\EE\{g_{z,jk} (Z,X_{j},Y_{k})\}$ is Lipschitz.  First, note that
\[
\EE\{g_{z,jk} (Z,X_{j},Y_{k})\}  =  \EE\{K_h (Z-z) \cdot  \bSigma_{jk}(Z)\}  = 
\int K_h (z-Z)\cdot  \varphi_{jk} (Z) dZ = (K_h * \varphi_{jk} )(z),
\]
where $\varphi_{jk} (Z) = f_Z(Z) \cdot \bSigma_{jk} (Z)$ and $K_h*\varphi_{jk}$ is the convolution between $K_h$ and $\varphi_{jk}$.
Similar to (\ref{proof:lemma:coveringnumberkz2})-(\ref{proof:lemma:coveringnumberkz4}), we have 
\begin{equation}
\label{Eq:proof:lemma:coveringnumbergzjk3}
\begin{split}
\underset{z_0\in [0,1]}{\sup} \; \underset{j,k\in [d]}{\max} \; \left|  \frac{\partial }{\partial z}  \EE\{g_{z,jk} (Z,X_{j},Y_{k})\}    \Big|_{z=z_0}   \right|  &=\underset{z_0\in [0,1]}{\sup}   \; \underset{j,k\in [d]}{\max} \;  \left|( \dot{K}_h * \varphi_{jk} )(z_0)    \right|\\
&= \underset{j,k\in [d]}{\max} \; \left\|  (\dot{K}_h * \varphi_{jk})(z) \right\|_{\infty}\\
&\le \left\|  \dot{K}_h  \right\|_1  \cdot \underset{j,k\in [d]}{\max} \; \|\varphi_{jk}\|_{\infty}\\
&\le \frac{1}{h} \cdot \|K\|_{\mathrm{TV}} \cdot \bar{f}_Z \cdot M_{\sigma},
\end{split}
\end{equation}
where the first inequality is obtained by an application of Young's inequality, and the last expression is obtained by (\ref{proof:lemma:coveringnumberkz3}) and Assumptions~\ref{ass:marginal}-\ref{ass:covariance}.

Equation~\ref{Eq:proof:lemma:coveringnumbergzjk3} implies that for any $z_1,z_2 \in [0,1]$, 
\[
\left|   \EE \{g_{z_1,jk} (Z,X_{j},Y_{k})\} -   \EE \{g_{z_2,jk} (Z,X_{j},Y_{k})\}    \right| \le \frac{1}{h}\cdot  \|K\|_{\mathrm{TV}} \cdot \bar{f}_Z \cdot M_\sigma \cdot |z_1-z_2|,
\]
implying that $\EE\{g_{z,jk} (Z,X_{j},Y_{k})\}$ is a Lipschitz continuous function with Lipschitz constant $h^{-1}\cdot \|K\|_{\mathrm{TV}}\cdot \bar{f}_Z \cdot M_\sigma$.
By an application of Lemma~\ref{lemma:coveringnumber1}, we have 
\begin{equation}
\label{Eq:proof:lemma:coveringnumbergzjk4}
\underset{Q}{\sup}\;N\{\cG_{2,jk},L_2(Q),\epsilon\} \le \frac{2}{h \epsilon}\cdot  \|K\|_{\mathrm{TV}} \cdot \bar{f}_Z\cdot M_\sigma.
\end{equation}

\textbf{Covering number of the function class $\cG_{jk}$:}
The function class $\cG_{jk}$ can be written as 
\[
\cG_{jk} = \{f_{1,jk} - f_{2,jk} \mid f_{1,jk} \in \cG_{1,jk}, \;f_{2,jk} \in \cG_{2,jk}, \; j,k\in [d] \}.
\]
By an application of Lemma~\ref{lemma:coveringnumber2} with $C_1 = (2\cdot C_K\cdot \|K\|_{\mathrm{TV}}\cdot M^2_X)^4$, $C_2 = 2\cdot \bar{f}_Z \cdot \|K\|_{\mathrm{TV}}\cdot M_\sigma$, $a_1= h^{-1 }\cdot \log d$,  $a_2=h^{-1}$, $v_1 =4$, and $v_2=1$, along with (\ref{Eq:proof:lemma:coveringnumbergzjk2}) and (\ref{Eq:proof:lemma:coveringnumbergzjk4}), we obtain 
\begin{equation}
\label{Eq:proof:lemma:coveringnumbergzjk5}
\underset{Q}{\sup}\; N\{\cG_{jk},L_2(Q),\epsilon\} \le  \left(\frac{4 \cdot \|K\|_{\mathrm{TV}} \cdot C_K^{4/5} \cdot \bar{f}_Z^{1/5}\cdot M_\sigma^{1/5}  \cdot M^{8/5}_X\cdot  \log^{4/5} d}{h\epsilon}\right)^{5},
\end{equation}
as desired.

\subsection{Proof of Lemma~\ref{lemma:coveringnumberjzjk1}}

\label{proof:lemma:coveringnumberjzjk1}
Similar to the proof of Lemma~\ref{lemma:coveringnumbergzjk},  we condition on the event 
\begin{equation*}
\cA = \left\{\underset{i\in [n]}{\max} \; \underset{j\in [d]}{\max} \; \max(|X_{ij}|,|Y_{ij}|) \le M_X \cdot \sqrt{\log d}  \right\}.
\end{equation*}
The event $\cA$ holds with probability at least $1-1/d$.

Recall that $J_{z,jk}^{(1)} (u,\bX_{i},\bY_{i}) = \sqrt{h} \cdot \left\{\bTheta_j (z)\right\}^T  \cdot \left[ K_h (u-z) \bX_{i}\bY_{i}^T -   \EE\left\{K_h (Z-z) \bX\bY^T  \right\}    \right] \cdot \bTheta_k(z)$ and let 
$\cJ^{(1)}_{jk} = \left\{ J_{z,jk}^{(1)} \;|\; z\in [0,1]  \right\}$.  
To obtain the covering number of the function class $\cJ_{jk}^{(1)}$, we consider bounding the covering number of a larger class of function.  
To this end, we define  $\bPhi_{\omega}^{(1)} (u,\bX_i,\bY_i) = \sqrt{h}\cdot \left[ K_h (u-z) \bX_{i}\bY_{i}^T -   \EE\left\{K_h (Z-z) \bX\bY^T  \right\}\right]     $ to be a $d \times d$ matrix.  
We denote the $(j,k)$th element of $\Phi_{\omega}^{(1)} (u,\bX_i,\bY_i)$ as $\Phi_{\omega,jk}^{(1)} (u,X_{ij},Y_{ik}) = \sqrt{h} \cdot q_{\omega,jk} (u,X_{ij},Y_{ik})$, where $q_{\omega,jk} (u,X_{ij},Y_{ik})= K_h(u-\omega) X_{ij}Y_{ik}- \EE\{K_h(Z-\omega) X_{j}Y_{k}\}$.  
We aim to obtain an $\epsilon$-cover $\cN^{(1')}$ for the following function class
\[
\cJ_{jk}^{(1')} = \left[ \{\bTheta_j (z)\}^T \bPhi_{\omega}^{(1)} (\cdot) \bTheta_k (z)        \; | \; \omega,z\in [0,1] \right].
\]
In other words, we show that for any $(\omega_1,z_1) \in [0,1]^2$, there exists $(\omega_2,z_2) \in \cN^{(1')}$  such that 
\[
\left\|\{\bTheta_j (z)\}^T \bPhi_{\omega}^{(1)} (u, \bX_i,\bY_i) \bTheta_k (z)  -\{\bTheta_j (z')\}^T \bPhi_{\omega'}^{(1)} (u, \bX_i,\bY_i) \bTheta_k (z')                 \right\|_{L_2(Q)} \le \epsilon.
\]

Given any $j,k\in[d]$, $\omega,\omega',z,z'\in [0,1]$, by the triangle inequality, we have 
\begin{equation}
\small
\label{Eq:proof:lemma:coveringnumberjzjk10}
\begin{split}
&\left\|\{\bTheta_j (z_1)\}^T \bPhi_{\omega_1}^{(1)} (u, \bX_i,\bY_i) \bTheta_k (z_1)  -\{\bTheta_j (z_2)\}^T \bPhi_{\omega_2}^{(1)} (u, \bX_i,\bY_i) \bTheta_k (z_2)                 \right\|_{L_2(Q)}\\
&\le \underbrace{\left\| \left\{\bTheta_j(z_1)-\bTheta_j (z_2)\right\}^T \bPhi_{\omega_1}^{(1) }(u, \bX_i,\bY_i)    
\bTheta_k (z_1)     \right\|_{L_2(Q)}}_{I_1} + \underbrace{\left\| \left\{\bTheta_j (z_2)\right\}^T \left\{ \bPhi_{\omega_1}^{(1)}(u,\bX_i,\bY_i)-\bPhi_{\omega_2}^{(1)}(u,\bX_i,\bY_i)  \right\} \bTheta_k (z_1)  \right\|_{L_2(Q)}}_{I_2}\\
& \quad+ \underbrace{\left\| \left\{\bTheta_j(z_2)\right\}^T \bPhi_{\omega_2}^{(1) }(u, \bX_i,\bY_i)    
 \left\{\bTheta_k (z_1)-\bTheta_k (z_2)\right\}    \right\|_{L_2(Q)}}_{I_3}.
\end{split}
\end{equation}
We now obtain the upper bounds for $I_1,I_2$, and $I_3$.

\textbf{Upper bound for $I_1$ and $I_3$:} First, we note that by Holder's inequality, we have 
\[
I_1 \le  \left\|\bTheta_j (z_1) - \bTheta_j (z_2)  \right\|_1 \cdot \underset{j,k\in [d]}{\max}\; \left\|\Phi_{\omega_1,jk}^{(1)} (u,X_{ij},Y_{ik})\right\|_{L_2(Q)} \cdot \|\bTheta_k (z_1)\|_1.
\]
Since $\bTheta (z) \in \cU (s,M,\rho)$, we have 
\begin{equation}
\label{Eq:proof:lemma:coveringnumberjzjk11}
\underset{z\in [0,1]}{\sup} \;\underset{j\in [d]}{\max}\; \|\bTheta_j (z)\|_1 \le M.
\end{equation}
Moreover, for any $z_1,z_2 \in [0,1]$, we have  
\begin{equation}
\label{Eq:proof:lemma:coveringnumberjzjk12}
\begin{split}
\underset{j\in [d]}{\sup}\; \|\bTheta_j (z_1)-\bTheta_j (z_2) \|_1 &\le \sqrt{d} \cdot \|\bTheta (z_1)-\bTheta (z_2)\|_2\\
&\le \sqrt{d} \cdot \|\bTheta (z_1)\|_2 \cdot \| \Ib_{d} - \bSigma (z_1) \bTheta (z_2)       \|_2\\
&\le \sqrt{d} \cdot \|\bTheta (z_1)\|_2 \cdot \|\bTheta (z_2) \|_2 \cdot \|\bSigma (z_1) - \bSigma (z_2)\|_2 \\
&\le \sqrt{d} \cdot \rho^2 \cdot d \cdot \|\bSigma (z_1) - \bSigma (z_2)\|_{\max}\\
&\le d^{3/2} \cdot \rho^2 \cdot M_{\sigma} \cdot |z_1 - z_2|,
\end{split}
\end{equation}
where the second to the last inequality follows from the fact that $\bTheta (z) \in \cU (s,M,\rho)$ and the last inequality follows from Assumption~\ref{ass:covariance}. 
Finally, from (\ref{proof:lemma:empirical process g2}) and the definition of $\Phi_{\omega_1,jk}^{(1)} (\cdot) = \sqrt{h} \cdot q_{\omega_1,jk}(\cdot)$, we have 
\begin{equation}
\label{Eq:proof:lemma:coveringnumberjzjk13}
\underset{j,k\in [d]}{\max}\; \left\|\Phi_{\omega_1,jk}^{(1)} (u,X_{ij},Y_{ik})\right\|_{L_2(Q)} \le \frac{2}{\sqrt{h}} \cdot M_X^2 \cdot \|K\|_{\infty}\cdot \log d.
\end{equation}
Combining (\ref{Eq:proof:lemma:coveringnumberjzjk11})-(\ref{Eq:proof:lemma:coveringnumberjzjk13}), we have 
\begin{equation}
\label{Eq:proof:lemma:coveringnumberjzjk14}
I_1 \le d^{3/2}\cdot \log d \cdot \rho^2 \cdot M_\sigma \cdot M \cdot M_X^2 \cdot \|K\|_{\infty} \cdot \frac{2}{\sqrt{h}}\cdot |z_1-z_2|.
\end{equation}
We note that  $I_3$ can be upper bounded the same way as $I_1$. 

\textbf{Upper bound for $I_2$:} Recall from (\ref{Eq:proof:lemma:coveringnumberjzjk10}) that 
\begin{equation*}
\begin{split}
I_2 &= \left\| \left\{\bTheta_j (z_2)\right\}^T \left\{ \bPhi_{\omega_1}^{(1)}(u,\bX_i,\bY_i)-\bPhi_{\omega_2}^{(1)}(u,\bX_i,\bY_i)  \right\} \bTheta_k (z_1)  \right\|_{L_2(Q)}\\
&\le \|\bTheta_k (z_1)\| \cdot \|\bTheta_j (z_2)\|_1\cdot \underset{j,k\in [d]}{\max}
\; \left\|\sqrt{h} \cdot   \{ q_{\omega_1,jk} (u,X_{ij},Y_{ik})-q_{\omega_2,jk} (u,X_{ij},Y_{ik})\} \right\|_{L_2(Q)},
\end{split}
\end{equation*}
where the inequality holds by Holder's inequality and the definition of $\bPhi_{\omega}^{(1) } (u,\bX_i,\bY_i)$.
Let 
$
\varPhi_{jk}^{(1)} = \left\{\sqrt{h} \cdot q_{\omega,jk} (\cdot) \; |\; \omega\in [0,1]\right\}
$
and recall from Lemma~\ref{lemma:coveringnumbergzjk} that we constructed an $\epsilon$-cover $\cN^{(1'')} \subset [0,1]$ for the function class $\varPhi_{jk}^{(1)}$ with cardinality $\left|\cN^{(1'')}\right| = \left(\frac{4 \cdot \|K\|_{\mathrm{TV}} \cdot C_K^{4/5} \cdot \bar{f}_Z^{1/5}\cdot M_\sigma^{1/5}  \cdot M^{8/5}_X\cdot  \log^{4/5} d}{\sqrt{h} \cdot \epsilon}\right)^{5}$. 
Since the construction of the $\epsilon$-cover in Lemma~\ref{lemma:coveringnumbergzjk} is independent of the indices $j$ and $k$, we have that for any $j,k\in [d]$ and $\omega_1 \in [0,1]$, there exists a $\omega_2 \in \cN^{(1'')}$ such that 
\begin{equation}
\label{Eq:proof:lemma:coveringnumberjzjk15}
\underset{j,k\in [d]}{\max}
\; \left\|\sqrt{h} \cdot   \{q_{\omega_1,jk} (u,X_{ij},Y_{ik})-q_{\omega_2,jk} (u,X_{ij},Y_{ik})\} \right\|_{L_2(Q)}\le \epsilon.
\end{equation}
Thus, by (\ref{Eq:proof:lemma:coveringnumberjzjk11}) and (\ref{Eq:proof:lemma:coveringnumberjzjk15}), we have
\begin{equation}
\label{Eq:proof:lemma:coveringnumberjzjk16}
I_2 \le M^2 \cdot \epsilon.
\end{equation}

\textbf{Covering number of the function class $\cJ_{jk}^{(1)}$:} Since $\cJ_{jk}^{(1)}\subset \cJ_{jk}^{(1')}$, the covering number of $\cJ_{jk}^{(1)}$ is upper bounded by the covering number of $\cJ_{jk}^{(1')}$.
It suffices to construct an $\epsilon$-cover of the function class $\cJ_{jk}^{(1')}$. In the following, we show that $\cN^{(1')} = \cN^{(1'')} \times \left\{ i \cdot \epsilon\cdot \sqrt{h} \; | \; i=1,\ldots,\frac{1}{\epsilon \cdot \sqrt{h}}  \right\}$ is an $\epsilon$-cover of $\cJ_{jk}^{(1')}$.  
For any $(\omega_1,z_1)\in [0,1]^2$, there exists $(\omega_2,z_2) \in \cN^{(1')}$ such that (\ref{Eq:proof:lemma:coveringnumberjzjk15}) holds and that $|z_1-z_2| \le \sqrt{h} \cdot \epsilon$. Thus, combining (\ref{Eq:proof:lemma:coveringnumberjzjk14}) and (\ref{Eq:proof:lemma:coveringnumberjzjk16}), we have 
\begin{equation}
\begin{split}
&\left\|\{\bTheta_j (z_1)\}^T \bPhi_{\omega_1}^{(1)} (u, \bX_i,\bY_i) \bTheta_k (z_1)  -\{\bTheta_j (z_2)\}^T \bPhi_{\omega_2}^{(1)} (u, \bX_i,\bY_i) \bTheta_k (z_2)                 \right\|_{L_2(Q)}\\
&\le 2 \cdot d^{3/2}\cdot \log d \cdot \rho^2 \cdot M_{\sigma}\cdot M \cdot M_X^2 \cdot \|K\|_{\infty} \cdot \frac{2}{\sqrt{h}} \cdot |z_1-z_2| + M^2 \cdot \epsilon\\
&\le 4  \cdot d^{3/2}\cdot \log d \cdot \rho^2 \cdot M_{\sigma}\cdot M \cdot M_X^2 \cdot \|K\|_{\infty} \cdot \frac{2}{\sqrt{h}} \cdot \epsilon + M^2 \cdot \epsilon\\
&\le C \cdot d^{3/2}\cdot  \log d \cdot \epsilon,
\end{split}
\end{equation}
where $C>0$ is a generic constant that does not depend on $d$, $h$, and $n$.  

Thus, we have
\[
N\{\cJ_{jk}^{(1')}, L_2(Q),C \cdot d^{3/2}\cdot \log d \cdot \epsilon\}\le \left| \cN^{(1')}  \right| \le \left(\frac{4 \cdot \|K\|_{\mathrm{TV}} \cdot C_K^{4/5} \cdot \bar{f}_Z^{1/5}\cdot M_\sigma^{1/5}  \cdot M^{8/5}_X\cdot  \log^{4/5} d}{\sqrt{h} \cdot \epsilon}\right)^{5}  \cdot \frac{1}{\sqrt{h} \cdot \epsilon}.
\]
Since $\cJ_{jk}^{(1)} \subset \cJ_{jk}^{(1')}$, the above expression implies that 
\begin{equation}
\begin{split}
N\{\cJ_{jk}^{(1)}, L_2(Q),\epsilon\}\le N\{\cJ_{jk}^{(1')}, L_2(Q),\epsilon\} \le C\cdot \left( \frac{d^{5/4} \cdot \log^{3/2} d }{\sqrt{h} \cdot \epsilon}   \right)^6,
\end{split}
\end{equation}
as desired.

\subsection{Proof of Lemma~\ref{lemma:coveringnumberjzjk2}}
\label{proof:lemma:coveringnumberjzjk2}

First, we note that
\begin{equation*}
\begin{split}
J_{z,jk}^{(2)} (u) &= \sqrt{h} \cdot \left\{\bTheta_j (z)\right\}^T  \cdot \left[ K_h (u-z) -   \EE\left\{K_h (Z-z)  \right\}    \right] \cdot \bSigma (z)\cdot \bTheta_k(z)\\
&= \sqrt{h} \cdot k_z (u) \cdot \bTheta_{jk}(z),
\end{split}
\end{equation*}
where $k_z (u) = K_h (u-z) - \EE\{K_h (Z-z)\}$.  Let $\cJ^{(2)}_{jk} = \left\{ J_{z,jk}^{(2)} \;|\; z\in [0,1]  \right\}$.  
Furthermore, recall that 
$\cK = \left\{ k_z(\cdot)  \; | \;z\in [0,1] \right\}$ and let 
$\varTheta_{jk}  = \left\{   \bTheta_{jk}(z)     \;|\; z\in [0,1]\right\}$.
The function class $\cJ^{(2)}_{jk}$ can be written as
\[
\cJ^{(2)}_{jk} = \{ \sqrt{h} \cdot f_{1} \cdot f_{2,jk} \; | \; f_1\in \cK, \; f_{2,jk} \in \varOmega_{jk}   \}.
\]
It suffices to obtain the covering number for $\cK$ and $\varOmega_{jk}$, and apply Lemma~\ref{lemma:coveringnumber2}.

\textbf{Covering number of the function class $\cK$:} By Lemma~\ref{lemma:coveringnumberkz}, we have 
\begin{equation}
\label{Eq:proof:lemma:coveringnumberjzjk21}
N\{\cK,L_2(Q),\epsilon\} \le \left(\frac{4 \cdot \|K\|_{\mathrm{TV}} \cdot C_K^{4/5} \cdot \bar{f}_Z^{1/5}}{h\epsilon}\right)^{5}.
\end{equation}

\textbf{Covering number of the function class $\varTheta_{jk}$:} We show that $\bTheta_{jk} (z)$ is Lipschitz, and apply Lemma~\ref{lemma:coveringnumber1} to obtain the covering number for $\varTheta_{jk}$.
Similar to (\ref{Eq:proof:lemma:coveringnumberjzjk12}), for any $z_1,z_2\in [0,1]$, we have
\begin{equation*}
\begin{split}
\|\bTheta (z_1)-\bTheta (z_2)\|_{\max}&\le \|\bTheta (z_1)\|_2 \cdot \| \bTheta(z_2) \cdot \{\bSigma (z_1)-\bSigma (z_2)\}\|_2\\
&\le \|\bTheta (z_1)\|_2 \cdot \|\bTheta (z_2)\|_2 \cdot \|\bSigma(z_1)-\bSigma (z_2)\|_{2}\\
&\le \rho^2 \cdot d \cdot \|\bSigma(z_1)-\bSigma (z_2)\|_{\max}\\
&\le \rho^2 \cdot d \cdot M_{\sigma} \cdot |z_1-z_2|,
\end{split}
\end{equation*}
where the last inequality follows from Assumption~\ref{ass:covariance}.
Since $ \bTheta_{jk}(z)$ is $\rho^2 \cdot d \cdot M_{\sigma}$-Lipschitz, by Lemma~\ref{lemma:coveringnumber1}, we have 
\begin{equation}
\label{Eq:proof:lemma:coveringnumberjzjk22}
N\{\varTheta_{jk},L_2(Q),\epsilon\} \le \frac{ M_{\sigma}\cdot  \rho^2 \cdot d}{\epsilon}.
\end{equation}

\textbf{Covering number of the function class $\cJ_{jk}^{(2)}$:} We now apply Lemma~\ref{lemma:coveringnumber2} to obtain the covering number of $\cJ_{jk}^{(2)}$.  
Applying Lemma~\ref{lemma:coveringnumber2} with $a_1 = d$, $v_1=1$, $C_1 = M_{\sigma} \cdot \rho^2$, $a_2 = h^{-1}$, $v_2 = 5$, $C_2 = \left(4 \cdot \|K\|_{\mathrm{TV}}\cdot C_K^{4/5} \cdot \bar{f}_Z^{1/5} \right)^5$, and $U= \frac{2}{h}\cdot \|K\|_\infty$, along with (\ref{Eq:proof:lemma:coveringnumberjzjk21}) and (\ref{Eq:proof:lemma:coveringnumberjzjk22}), we have 
\[
N\{\cJ_{jk}^{(2)}, L_2(Q) , \sqrt{h} \cdot \epsilon \} \le C \cdot \left(\frac{d^{1/6}}{h^{11/6}\cdot \epsilon} \right)^6,
\]
where $C>0$ is a generic constant that does not depend on $n,d$, and $h$. This implies that
\[
N\{\cJ_{jk}^{(2)}, L_2(Q) , \epsilon \} \le C \cdot \left(\frac{d^{1/6}}{h^{4/3}\cdot \epsilon} \right)^6,
\]
as desired.

\section{Technical Lemmas on Empirical Process}
\label{appendix:empirical process}
In this section, we present some existing tools on empirical process. 
  The following lemma states that the supreme of any empirical process is concentrated near its mean.  It follows directly from Theorem 2.3 in \citet{bousquet2002bennett}.
\begin{lemma}
\label{lemma:ep1}
(Theorem A.1 in \citealp{van2008high})
Let $X_1,\ldots,X_n$ be independent random variables and let $\cF$
be a function class such that there exists $\eta$ and $\tau^2$ satisfying 
\[
\underset{f\in \cF}{\sup} \; \|f\|_{\infty} \le \eta \qquad \mathrm{and}\qquad 
\underset{f\in \cF}{\sup}\; \frac{1}{n} \sum_{i\in n} \mathrm{Var} \{f(X_{i})\}\le \tau^2.
\]
Define
\[
Y= \underset{f\in \cF}{\sup} \; \left| \frac{1}{n} \sum_{i\in [n]} \left[ f(X_{i}) -\EE\{f(X_{i})\}\right]  \right|.
\]
Then, for any $t>0$, 
\[
P \left[ Y \ge \EE(Y) + t \sqrt{2\left\{ \tau^2 +2 \eta \EE(Y)\right\}} + 2t^2 \eta/3 \right]\le \exp \left(-n t^2 \right).
\]
\end{lemma}

The above inequality involves evaluating the expectation of the supreme of the empirical process. 
The following lemma follows directly from Theorem 3.12 in \citet{koltchinskii2011oracle}.
It provides an upper bound on the expectation of the supreme of the empirical process  
as a function of its covering number.

\begin{lemma}
\label{lemma:ep2}
(Lemma F.1 in \citealp{lu2015post})
Assume that the functions in $\cF$ defined on $\cX$ are uniformly bounded by a constant $U$ and $F(\cdot)$ is the envelope of $\cF$ such that $|f(x)|\le F(x)$ for all $x\in \cX$ and $f\in \cF$.  Let 
$\sigma^2_P  = \underset{f\in \cF}{\sup} \; \EE(f^2)$.  
Let $X_1,\ldots,X_n$ be $\mathrm{i.i.d.}$ copies of the random variables $X$.  We denote the empirical measure as $\PP_n = \frac{1}{n} \sum_{i\in [n]} \delta_{X_{i}}$.
If for some $A,V>0$ and for all $\epsilon >0$ and $n\ge 1$, the covering entropy satisfies 
\[
N\{\cF,L_2(\PP_n),\epsilon\} \le \left(\frac{A \|F\|_{L_2(\PP_n)}}{\epsilon} \right)^V,
\]
 then for any $\mathrm{i.i.d.}$ sub-gaussian mean zero random variables $\xi_1,\ldots,\xi_n$, there exists a universal constant $C$ such that
 \[
 \EE\left\{  \underset{f\in \cF}{\sup} \; \frac{1}{n} \left| \sum_{i\in [n]} \xi_i f(X_i)\right|   \right\} \le C \left\{ \sqrt{\frac{V}{n}} \sigma_P \sqrt{\log \left( \frac{A\|F\|_{L_2(\PP)}}{\sigma_P}\right)} +\frac{VU}{n}  \log \left( \frac{A\|F\|_{L_2(\PP)}}{\sigma_P}\right)  \right\}.
 \]
 Furthermore, we have 
 \[
 \EE\left\{  \underset{f\in \cF}{\sup} \; \frac{1}{n} \left| \sum_{i\in [n]} [ f(X_i) - \EE\{f(X_i)\} ]  \right| \right\} \le C \left\{ \sqrt{\frac{V}{n}} \sigma_P \sqrt{\log \left( \frac{A\|F\|_{L_2(\PP)}}{\sigma_P}\right)} +\frac{VU}{n}  \log \left( \frac{A\|F\|_{L_2(\PP)}}{\sigma_P}\right)  \right\}.
 \] 
\end{lemma}

\bibliography{reference}
\end{document}